\documentclass[preprint]{elsarticle}
\usepackage{times}
\usepackage{url}
\usepackage{hyperref}

\journal{\ldots}

\usepackage{graphicx}
\usepackage{latexsym}

\usepackage{amsmath}
\usepackage{amsthm}
\usepackage{amssymb}
\usepackage{stmaryrd}

\usepackage{epsfig}
\usepackage{enumerate}
\usepackage[small]{diagrams}
\usepackage{colortbl} 

\usepackage{pgf}
\usepackage{tikz}
\usetikzlibrary{automata}
\usetikzlibrary{trees}

\usepackage{booktabs} 
\usepackage{hhline} 
\usepackage{graphicx} 
\usepackage{pgfplots}

\usepackage{defs}

\newtheorem{proposition}{Proposition}
\newtheorem{definition}{Definition}
\newtheorem{theorem}{Theorem}
\newtheorem{corollary}{Corollary}
\newtheorem{lemma}{Lemma}
\newtheorem{fact}{Fact}
\newtheorem{example}{Example}
\newtheorem{remark}{Remark}

\newcommand{\B}{\mathcal{B}}
\newcommand{\A}{\mathcal{A}}

\newcommand{\verteq}{\rotatebox{90}{=}}

\begin{document}

\begin{frontmatter}

\title{
On the Graded Acceptability of Arguments \\ in Abstract and Instantiated Argumentation\tnoteref{mytitlenote}
}
\tnotetext[mytitlenote]{The article extends earlier work presented at IJCAI'15 \cite{grossi15graded}. We thank the participants of IJCAI'15, the Imperial College Logic Seminar (Nov., 2013) and the Japan National Institute of Informatics Seminar (Oct. 2013), where parts of this work were presented, for many helpful comments and suggestions. We are also indebted to Matthias Thimm for many helpful comments on an earlier draft of this paper. Davide Grossi acknowledges support for this research by EPRSC under grant EP/M015815/1. The bulk of this research was carried out while  Grossi was affiliated to the Department of Computer Science, University of Liverpool.}

\author{Davide Grossi}
\address{University of Groningen}

\author{Sanjay Modgil}
\address{King's College London}





\begin{abstract}
The paper develops a formal theory of the degree of justification of arguments, which relies solely on the structure of an argumentation framework, and which can be successfully interfaced with approaches to instantiated argumentation. The theory is developed in three steps. First, the paper introduces a graded generalization of the two key notions underpinning Dung's semantics: self-defense and conflict-freeness. This leads to a natural generalization of Dung's semantics, whereby standard extensions are weakened or strengthened depending on the level of self-defense and conflict-freeness they meet. The paper investigates the fixpoint theory of these semantics, establishing existence results for them.
Second, the paper shows how graded semantics readily provide an approach to argument rankings, offering a novel contribution to the recently growing research programme on ranking-based semantics.  Third, this novel approach to argument ranking is applied and studied in the context of instantiated argumentation frameworks, and in so doing is shown to account for a simple form of accrual of arguments within the Dung paradigm. Finally, the theory is compared in detail with existing approaches.

\end{abstract}



\end{frontmatter}



\section{Introduction}\label{Sec:Introduction}

 Argumentation is a form of reasoning that makes explicit the reasons for the
conclusions that are drawn, and how conflicts between reasons are resolved. Recent years
have witnessed intensive study of both logic-based and  human orientated  models of argumentation and their
use in formalising agent reasoning, decision making, and inter-agent dialogue \cite{ArgInAI, RahSim}. Much of this work builds on Dung's seminal theory of abstract argumentation
\cite{dung95acceptability}.
A Dung argumentation framework ($AF$) \cite{dung95acceptability} is essentially a directed graph relating arguments by binary, directed forms of conflict, called attacks. The sceptically or credulously justified arguments are those in the intersection, respectively union, of sets---called extensions---of `acceptable' arguments evaluated under various semantics (see \cite{baroni11introduction} for an overview). Extensions are evaluated based on two core principles. Firstly,  a set of arguments should not contain internal attacks, that is, it should be {\em conflict-free}.  Secondly,  it should defend itself in the sense that any argument $a$ in the set is either un-attacked, or if attacked by some argument $b$, there is then an argument in the set that defends $a$ by attacking $b$ (in which case $a$ is said to be defended by, or {\em acceptable with respect to}, the set of arguments).
Arguments and attacks may be seen as primitive or assumed
to be defined by an instantiating set of sentences in natural language or in a formal logical language. The former case thus provides for characterisations of more human-orientated uses of argument in reasoning, while
in the latter case,
the claims of the justified arguments identify the
non-monotonic inferences from the instantiating set of logical formulae, thus providing for dialectical characterisations of non-monotonic logics \cite{dung95acceptability}.

Dung's theory has been extended in a number of directions. In particular, `exogenously' given information about the relative strength of arguments has been used to determine which attacks succeed as defeats, so that the acceptable arguments are evaluated with respect to the arguments related by defeats rather than attacks. In this way one can effectively arbitrate amongst credulously justified conflicting arguments.
Examples include \cite{AmCay,ModgilAIJ,Modgil2013361} that make use of a preference relation over arguments (where in the case of \cite{ModgilAIJ}, preferences are expressed by arguments that attack attacks amongst arguments) and \cite{BC03} that makes use of an ordering over the values promoted by  arguments. Other notable developments include approaches that associate probabilities with arguments \cite{Hunter2017, Li2012}, and weights on attacks so  that extensions do not necessarily comply with the conflict-freeness requirement on Dung extensions;  arguments in an extension may attack each other provided that the summative weight of attacks does not exceed a given `inconsistency budget' \cite{dunne11weighted}.\footnote{See also \cite{Janssen} which account for the relative strength of attacks amongst arguments.}

\paragraph{Context: graduality in argumentation}

It has long been recognized---since at least \cite{BesnardHunter01} and \cite{CLS04}---that one of the drawbacks of Dung's theory of abstract argumentation is the limited level of granularity the theory offers in differentiating the strength or status of arguments (essentially three, but cf. \cite{wu10labelling}). Consider the following informal example:

\begin{example}\label{ExampleGuilty} Suppose an argument $I$ concluding the presumed innocence of a suspect. $I$ is attacked by an argument $G$ which consists of two sub-arguments $G_1$ and $G_2$ that respectively conclude that a suspect had \emph{opportunity} and \emph{motive}, where $G$ defeasibly extends these sub-arguments to conclude that the suspect is guilty. Suppose an argument $I_1$ (in support of an alibi) that attacks $G_1$ on the assumption that the suspect does not have an alibi. Suppose then an additional argument $I_2$ that attacks an assumption in $G_2$.
The level to which $I$ is defended, and so said to be justified, is intuitively increased in the case that we have counter-arguments $I_1$ and $I_2$ that argue against the suspect having motive \emph{and} opportunity,
 as compared with when one only has the counter-argument $I_1$.
\end{example}

The above example highlights one amongst a number of intuitions that are formalised by the above mentioned \cite{BesnardHunter01} and \cite{CLS04}, and more recently \cite{AmgoudNaim,Amgoud:2016,conf/jelia/MattT08}, in which the numbers of attackers and defenders are used to give a more fine grained assignment of status to (and hence ranking of) arguments. While these approaches are defined wth respect to $AF$s, the status of arguments is not determined using the standard `Dungian' concepts of sets of arguments (extensions) and  arguments defended by these sets; rather, measures of the strengths of arguments based on the numbers of attackers and/or defenders are propagated through the $AF$ graph\footnote{The exception being \cite{conf/jelia/MattT08} in which the strength of arguments is evaluated by reference to two person games in which a proponent defends an argument against counter-attacks by an opponent.}. These approaches, which we refer to here as `propagation based approaches'  (see \cite{bonzon16comparative} for a recent comparative overview), have also been developed to account for exogenously given information about the strength of arguments, which provide the initial measures that are then adjusted based on how they are propagated through the graph. Notable examples of the latter include \cite{gabbay-rodrigues:13,Gabbay2015} and \cite{conf/ijcai/LeiteM11}.

\paragraph{Paper Contribution}

The central aim of this paper is to show that a more fine grained assignment of status to arguments that does not rely on exogenous information, and thus only on the numbers of attackers and defenders of arguments, can be formalised as a natural generalisation of the Dungian notions of conflict free sets of arguments and the defense of arguments by sets of arguments.
This aim is achieved in the following way.
{\em First}, our starting point is the classical definition of an admissible set of arguments as one that is conflict free and that defends all its contained arguments \cite{dung95acceptability}. We show that the logical structure of this definition naturally generalises so as to yield more fine grained graded notions of conflict freeness and defense that account for the number of attackers and defenders,  thereby obtaining a graded variant of the concept of admissibility.
{\em Second}, this graded form of admissibility serves as a basis for defining graded variants of all the classic semantics of abstract argumentation studied in \cite{dung95acceptability}: complete, grounded, stable and preferred. Dung's definitions of the standard semantics can all be retrieved as special cases of our graded variants, showing that the form of graduality this paper studies is rooted in a principled way in the classical theory of abstract argumentation. These graded semantics are, intuitively, ways of interpreting the standard Dung semantics in `stricter' or `looser' ways. For instance, the grounded semantics can be interpreted more `�strictly' by requiring that all attackers be counter-attacked by at least two arguments, instead of just one as in the classic case. So each Dung semantics now comes  equipped with a family of strengthenings and weakenings, which we call {\em graded semantics}.
{\em Third}, these strengthenings and weakenings define a natural (partial) ordering dictated by set-inclusion: a stricter semantics will define sets of arguments which are subsets of the sets defined by a weaker one. This natural ordering then induces an ordering on the arguments themselves, thereby defining a ranking over the arguments in the framework. Such a ranking enables arbitration amongst arguments without recourse to exogenous preference information.
 As this ranking is induced from a generalization of Dung standard semantics, our approach is at the outset methodologically distinct from propagation-based approaches, where argument rankings are defined directly by reference to the argument graph.  {\em Fourth}, we study the application of graded semantics and their induced rankings to two key instantiations of Dung $AF$s: classical logic instantiations that, under the standard Dung semantics, yield a dialectical characterisation of non-monotonic inference in Preferred Subtheories \cite{bre89}, and instantiations that accommodate more human orientated uses of argumentation through  the use of schemes and critical questions \cite{Wal96}. In so doing, we seek to substantiate some of the intuitions captured by our generalisation of Dung semantics,
and show how the graded semantics capture a simple form of counting based accrual of arguments, which has traditionally been regarded as being incompatible with Dung's theory \cite{PrakAccrual}.



\paragraph{Outline of the paper}
The paper is structured in three parts.
\fbox{Part 1} concerns the development of the abstract theory of graded argumentation. It starts with Section \ref{Sec:Background}, where we review Dung's theory, giving prominence to its fixpoint-theoretic underpinnings, which have remained relatively under-investigated in the literature, and that we thus consider to be of some independent interest.
Section \ref{Sec:Graded Acceptability} then generalises Dung's notions of conflict freeness and  defense,  yielding grading variants of these notions
in terms of the number of arguments attacking and defending any given  $a$ whose acceptability with respect to a given set of arguments is at issue. This yields a ranking among types of conflict freeness and defense.
Section \ref{Sec:GradedSemantics} then generalises Dung's standard semantics,
so that extensions are graded with respect to the attacks and counter-attacks on their contained arguments. These semantics---which we call {\em graded}---are shown to generalise Dung's theory and are studied providing constructive existence results.
\fbox{Part 2} first shows, in Section \ref{Sec:Ranking}, how the new graded semantics yield a natural way of ranking arguments according to how strongly they are justified under different graded semantics, thereby enabling endogenous arbitration among credulously justified arguments. Then, Section \ref{Sec:Applications} illustrates application of these type of rankings to \emph{ASPIC+} instantiations of $AF$s \cite{Modgil2013361} that formalise stereotypical patterns of argumentation encoded in schemes and critical questions \cite{Wal96}, thus accounting for more human-orientated uses of argument,  and  \emph{ASPIC+} instantiations of $AF$s that provide dialectical characterisations of non-monotonic inference in Preferred Subtheories \cite{bre89}. Both types of instantiation are then shown to capture a simple form of counting based accrual, whereby multiple arguments in support of the same conclusion mutually strengthen each other.
In \fbox{Part 3}, Section \ref{Sec:RelatedWork} develops a thorough comparison of our approach to the existing approaches to graduality and rankings in argumentation, leveraging the systematization recently introduced in \cite{bonzon16comparative}. This allows us to place more precisely graded argumentation in the growing landscape of ranking-based semantics.
We conclude in Section \ref{Sec:Future} outlining some avenues for future research in graded argumentation.


\section{Preliminaries: Abstract Argumentation} \label{Sec:Background}

This preliminary section reviews key concepts and results from Dung's abstract argumentation theory \cite{dung95acceptability}. The presentation we provide gives prominence to the fixpoint theory underpinning Dung's theoretical framework. After \cite{dung95acceptability} the fixpoint theory of abstract argumentation has remained relatively under-investigated in the literature. It is, however, the most natural angle from which to pursue the objectives of this paper. We are unaware of any comprehensive exposition to date of the fixpoint theory of abstract argumentation, and so hope this preliminary section is of  independent interest.

Since this paper will provide a generalization of Dung's original theory, all results presented in this section can actually be obtained as direct corollaries of the results we will establish later in Section \ref{Sec:GradedSemantics}. This, we argue, should be a desirable feature for any theory of graduality in argumentation which bases itself on Dung's original proposal.
For completeness of the exposition, direct proofs of the results dealt with in this section can be found in Section \ref{appendix:proofs}.

\subsection{Basic Definitions}

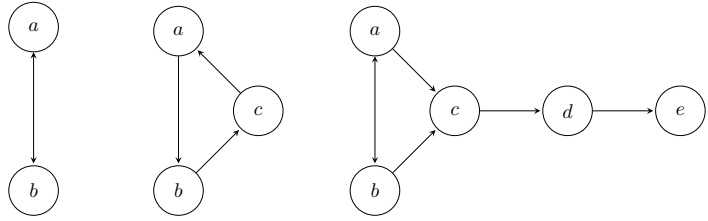
\begin{figure}
\begin{center}
\scalebox{0.75}{
\begin{tikzpicture}[node distance=2.9cm,shorten >=1pt,>=stealth,auto]

\node[state]	(a)                      {$a$};
\node[state]	(b)  [below of= a] {$b$};

\path[<->] (a) edge	node {} (b);

\end{tikzpicture}

\hspace{1.5cm}
\begin{tikzpicture}[node distance=2cm,shorten >=1pt,>=stealth,auto]

\node[state]	(c)     {$c$};
\node[state]	(a)  [above left of= c]  {$a$};
\node[state]	(b)  [below left of= c] {$b$};

\path[->] (a) edge	node {} (b);
\path[->] (c) edge	node {} (a);
\path[->] (b) edge	node {} (c);

\end{tikzpicture}

\hspace{1cm}
\begin{tikzpicture}[node distance=2cm,shorten >=1pt,>=stealth,auto]

\node[state]	(c)     {$c$};
\node[state]	(a)  [above left of= c]  {$a$};
\node[state]	(b)  [below left of= c] {$b$};

\node[state]	(d)  [right of= c] {$d$};
\node[state]	(e)  [right of= d] {$e$};

\path[<->] (a) edge	node {} (b);
\path[->] (a) edge	node {} (c);
\path[->] (b) edge	node {} (c);
\path[->] (c) edge	node {} (d);
\path[->] (d) edge	node {} (e);

\end{tikzpicture}
}
\end{center}
\caption{Argumentation frameworks}
\label{figure:graphs}
\end{figure}

\begin{definition}[Frameworks]\label{definition:attack_graph}
An {\em argumentation framework} (AF) $\Delta$ is a tuple $\tuple{\A, \ar}$ where $\A \neq \emptyset$, and $\ar$ $\subseteq$ $\A^2$ is a binary attack relation on $\A$. Notation
$x \ar y$ denotes that $x$ attacks $y$, and $X \ar x$ denotes that $\exists y \in X$ s.t. $y \ar x$. Similarly, $x \ar X$ denotes that $\exists y \in X$ s.t. $x \ar y$. For a given $\Delta$ we write $\overline{x}$ to denote $\set{y \in \Delta \mid y \rightarrow x}$ (the direct attackers of $x$), and $\overline{\overline{x}}$ to denote $\set{z \in \Delta \mid z \rightarrow y, y \in \overline{x}}$ (the direct defenders of $x$). Also, as is customary, $+$ denotes the transitive closure of a given relation, so that $b \ar^+ a$ stands for `there exists a path of attacks from $b$ to $a$'.
Finally, an AF such that for each $x$, $\overline{x}$ is finite, is called {\em finitary}, whereas an AF with a finite number of arguments $\A$ is called {\em finite}, else {\em infinite}.

\end{definition}
We will sometimes refer to $AF$s as {\em attack graphs}.\footnote{Although note that all definitions in this paper equally apply to `defeat' graphs  which assume a binary defeat relation on arguments, obtained through use of preferences in deciding which attacks succeed as defeats.} Figure \ref{figure:graphs} depicts three $AF$s. An argument $a \in \A$ is said to be acceptable w.r.t. $X \subseteq \A$, if any argument attacking $a$ is attacked by some argument in $X$, in which case $X$ is said to defend $a$. An $AF$'s characteristic (also called `defense') function, applied to some $X \subseteq \A$, returns the arguments defended by $X$ \cite{dung95acceptability} (henceforth, `$\pw$' denotes powerset): 

\begin{definition}[Defense Function]\label{definition:defense}

The defense function $\cf{\Delta}: \pw{\A} \map \pw{\A}$ for $\Delta = \tuple{\A, \ar}$ is defined as follows. For any $X \subseteq \A$:\\[-15pt]
\begin{eqnarray*}
\cf{\Delta}(X) & = & \set{x \in \A \mid \forall y \in \A: \IF y \ar x \THEN X \ar y}\\[-18pt]
\end{eqnarray*}
Where no confusion arises we may drop the subscript $\Delta$ in $\cf{\Delta}$.
\end{definition}
An argument $a \in \A$ is not attacked by a set $X \subseteq \A$ if no argument in $X$ attacks $a$. One can define a function which, applied to some $X \subseteq \A$ in an $AF$, returns the arguments that are not attacked by $X$. This function was introduced by Pollock in \cite{pollock87defeasible} for his theory of defeasible reasoning, and we refer to it here as the `neutrality function'.

\begin{definition}[Neutrality Function]\label{definition:neutrality}
The neutrality function $\cff{\Delta} : \pw{\A} \map \pw{\A}$ for $\Delta = \tuple{\A, \ar}$ is defined as follows. For any $X \subseteq \A$:
\[
\cff{\Delta}(X)  =  \set{x \in \A \mid \NOT X \ar x}.
\]
Again, where no confusion arises we may drop the subscript $\Delta$ in $\cff{\Delta}$.
\end{definition}

One final bit of terminology. In what follows we will often use the notion of function iteration for $\cf{}$ and $\cff{}$ which we define in the standard inductive way, for $F \in \set{\cf{}, \cff{}}$: $F^0(X) = X$; $F^{k+1}(X) = F(F^k(X))$.

\begin{example}[Defense and neutrality in Figure \ref{figure:graphs}]
The functions applied to the symmetric graph of Figure \ref{figure:graphs} (left) yield the following equations:
\[
\begin{array}{ccccccc}
\cf{}(\emptyset) & = & \emptyset & & \cff{}(\emptyset) & = & \set{a, b} \\
\cf{}(\set{a}) & = & \set{a} & & \cff{}(\set{a}) & = & \set{a} \\
\cf{}(\set{b}) & = & \set{b} & & \cff{}(\set{b}) & = & \set{b} \\
\cf{}(\set{a,b}) & = & \set{a,b} & & \cff{}(\set{a,b}) & = & \emptyset
\end{array}
\]
Notice that the output of $\cff{}$ on $\emptyset$ corresponds to the whole set of arguments, as no arguments can be attacked by $\emptyset$. Notice also that while $\set{a}$ and $\set{b}$ are included in $\cff{}(\set{a})$ and $\cff{}(\set{b})$, $\set{a,b}$ is not.
\end{example}

One can define the extensions of an $AF$ $\Delta$ under Dung's semantics,
in terms of the fixpoints ($X= \cf{\Delta}(X)$ and $X= \cff{\Delta}(X)$) or post-fixpoints ($X \subseteq \cf{\Delta}(X)$ and $X \subseteq \cff{\Delta}(X)$) of the defense and neutrality functions, as recapitulated in Table \ref{table:basics}.
The justified arguments are then defined under various semantics:
\begin{definition}[Justification under Semantics] \label{def:DungSemantics} Let $\Delta$ = $\tuple{\A,\ar}$.
Then for semantics $S \in \{$grounded, stable, preferred$\}$\footnote{Typically, the justified arguments are not defined w.r.t. the complete semantics, which subsume each of grounded, stable and preferred.}, $a \in \A$ is {\em credulously}, respectively {\em sceptically, justified}
under $S$, if $a$ is in at least one, respectively all, $S$ extensions of $\Delta$. 
\end{definition}

\begin{table}[t]
\begin{tabular}{lcl}
\hline
$X$ is conflict-free in $\Frame$ & iff & $X \subseteq \cff{\Frame}(X)$ \\
$X$ is self-defended in $\Frame$ & iff & $X \subseteq \cf{\Frame}(X)$ \\
$X$ is admissible in $\Frame$ & iff & $X \subseteq \cff{\Frame}(X)$ and $X \subseteq \cf{\Frame}(X)$ \\
$X$ is a complete extension in $\Frame$ & iff & $X \subseteq \cff{\Frame}(X)$ and $X = \cf{\Frame}(X)$ \\
$X$ is a stable extension of $\Frame$& iff & $X = \cff{\Frame}(X)$ \\
$X$ is the grounded set in $\Frame$ & iff &  $X$ is the smallest complete ext. of $\Frame$ ($X = \lfp.\cf{\Frame}$)\\
$X$ is a preferred extension of $\Frame$& iff & $X$ is a largest complete extension of $\Frame$ \\
\hline
\end{tabular}
\caption{Classical notions of abstract argumentation theory from \cite{dung95acceptability}.}
\label{table:basics}
\end{table}

Finally, we recapitulate some well-known properties, first established in \cite{dung95acceptability}, of the defense and neutrality functions that will be referred to later:
\begin{fact} \label{fact:simple}
Let $\Delta = \tuple{\A, \ar}$ be an $AF$ and $X, Y \subseteq \A$. The following holds:
\begin{eqnarray*}
X \subseteq Y & \IMPLIES & \cf{}(X) \subseteq \cf{}(Y) \\
X \subseteq Y & \IMPLIES & \cff{}(Y) \subseteq \cff{}(X) \\
\cf{}(X) & = & \cff{}(\cff{}(X))
\end{eqnarray*}
\end{fact}
That is, function $\cf{}$ is monotonic, function $\cff{}$ is antitonic, and the composition of $\cff{}$ with itself, which we will also denote $\cff{} \circ \cff{}$, is function $\cf{}$. For example, in Figure \ref{figure:graphs} (right), we have that $\cff{}(\set{a})$ = $\set{a,d,e}$, and $\cff{}(\set{a,d,e}) = \set{a,d} = \cf{}(\set{a})$.


\begin{fact}[$\omega$-continuity\footnote{Cf. \cite[Lemma 28]{dung95acceptability}.}] \label{fact:continuous}
If $\Frame$ is finitary, then $\cf{\Frame}$ is (upward-)continuous for any $X \subseteq \A$, i.e., for any upward directed set $D \in \pw{\pw{\A}}$ of finite subsets of $\A$:\footnote{We recall that an upward directed set $D \in \pw{\pw{\A}}$ is a set of sets such that any two elements $X$ and $Y$ in $D$ have an upper bound in $D$, that is, there also exists a superset $Z \supseteq X \cup Y$ in $D$. A downward directed set is defined dually in the obvious way.}
\begin{align}
\cf{\Frame}\left(\bigcup_{X \in D} X \right) & = \bigcup_{X \in D} \cf{\Frame}(X). \label{eq:upward}
\end{align}
Similarly, $\cf{\Frame}$ is (downward-)continuous for any $X \subseteq \A$, i.e., for any downward directed set $D \in \pw{\pw{\A}}$:
\begin{align}
\cf{\Frame}\left(\bigcap_{X \in D} X\right) & = \bigcap_{X \in D} \cf{\Frame}(X). \label{eq:downward}
\end{align}
\end{fact}
The above fact establishes important properties of the behaviour of the defense function with respect to sequences of sets of arguments and their limits. As we will also  later see in the graded generalization of Dung's theory, these properties are key in the construction of extensions through the iteration of the defense function.

\putaway{
\subsection{Rudiments of fixpoint theory of Dung's extensions}

In light of Facts \ref{fact:simple} and \ref{fact:continuous} we can rely on general results from order theory to establish the existence of the least fixpoint ($\lfp$) of the characteristic function, that is, the existence of the grounded extension. The monotonicity of $\cf{\Frame}$ guarantees the existence of the least fixpoint of $\cf{\Frame}$ as the intersection of all pre-fixpoints of $\cf{\Frame}$:
\begin{align}
\lfp.\cf{\Frame} & =  \bigcap\set{X \subseteq \A \mid \cf{\Frame}(X) \subseteq X}. & \mbox{     (Knaster-Tarski Theorem)} \label{formula:knaster}
\end{align}

Given a set of arguments $X$, the $n$-fold iteration of $\cf{\Frame}$ is denoted $\cf{\Frame}^n$ for $0 \leq n < \omega$ and its (countably) infinite iteration is denoted $\cf{\Frame}^\omega$. For a given $X$, an infinite iteration generates an infinite sequence, or stream, $\cf{\Frame}^0(X), \cf{\Frame}^1(X), \cf{\Frame}^2(X), \ldots$. A stream is said to stabilize if and only if there exists $0 \leq n < \omega$ such that $\cf{\Frame}^{n}(X) = \cf{\Frame}^{n +1}(X)$. Such set $\cf{\Frame}^{n}(X)$ is then called the limit of the stream. The monotonicity and continuity of the characteristic function, in finitary frameworks, guarantee together that such least fixpoint can be computed `from below' through a stream $\cf{\Frame}^0(X), \cf{\Frame}^1(X), \cf{\Frame}^2(X), \ldots$:
\begin{align}
\lfp.\cf{\Frame} & =  \bigcup_{0 \leq n < \omega}\cf{\Frame}^n(\emptyset) \label{formula:computation}
\end{align}
The reader is referred to \cite{davey90introduction} for a detailed presentation of these results.
We will come back later in some more detail to equation \eqref{formula:computation}, which is the stepping stone of some of the results the paper presents.
} 

\subsection{The Fixpoint Theory of Acceptability and Conflict-freeness}

We show how any admissible set of arguments can be saturated to a complete extension through a process of fixpoint approximation.
This establishes a general result concerning the computation of complete extensions in (finitary) attack graphs which, to the best of our knowledge, has never been reported in the literature. It is, however, a generalization of well-known existing results such as \cite[Lemma 46, Theorem 47]{dung95acceptability} (cf. also \cite{lifschitz96foundations}).

\subsubsection{Construction of Fixpoints from Admissible Sets}

Fix a framework $\Frame$ and take a set $X$ such that $X \subseteq \cf{}(X)$ and $X \subseteq \cff{}(X)$ (i.e., an admissible set). By iterating $\cf{}$, consider the stream of sets
$
\cf{}^0(X), \cf{}^1(X), \ldots,
$
and
$
\cf{}^0(\cff{}(X)), \cf{}^1(\cff{}(X)), \ldots.
$
Since $X$ is admissible, $\cf{}$ is monotonic and $\cff{}$ antitonic (Fact \ref{fact:simple}), the first stream (see the lower stream in Figure \ref{figure:decomposition1})  is non-decreasing and the second stream (see the upper stream in Figure \ref{figure:decomposition1}) is non-increasing, with respect to set inclusion. In finite attack graphs, these streams must therefore stabilize reaching a limit at state $|\A| + 1$. In infinite but finitary attack graphs, we will see that the limit can be reached at $\omega$. We will see (Lemma \ref{theorem:approxim8}) that the limits of these streams correspond to the {\em smallest} fixpoint of $\cf{}$ {\em containing} the admissible set $X$ and, respectively, the {\em largest} fixpoint of $\cf{}$ which is {\em contained} in $\cff{}(X)$. We denote the first one by $\lfp_X.\cf{}$ and the second one by $\gfp_X.\cf{}$. Intuitively, the two sets denote the smallest superset of $X$ which is equal to the set of arguments it defends\footnote{Theorem \ref{lemma:smallest} will show this set is also conflict-free and it is therefore the smallest complete extension containing $X$.} and, respectively, the largest set which is not attacked by $X$ and which is equal to the set of arguments it defends.\footnote{Note that such a set is not necessarily conflict-free. E.g., consider $\Frame =  \langle \set{a,b,c}$, $\set{(b,c),(c,b)} \rangle$, that is $b \ar c$ and $c \ar b$. Then $\cff{}(\set{a}) = \set{a,b,c}$ and $\cf{}(\set{a,b,c})  = \set{a,b,c}$. But clearly it is not the case that $\set{a,b,c} \subseteq \cff{}{\set{a,b,c}}$, that is, $\set{a,b,c}$ is not conflict-free.} The construction is illustrated in Figure \ref{figure:decomposition1} below.

\begin{example}
Consider the cycle of length three in Figure \ref{figure:graphs} (center), and take the admissible set $\emptyset$. By applying the above construction we obtain immediately $\cf{}(\emptyset) = \emptyset$ as limit of the lower stream, and $\cff{}(\emptyset) = \set{a,b,c} = \cf{}(\set{a,b,c})$ as limit of the upper stream.  $\emptyset$ is the smallest self-defended set  containing $\emptyset$, $\set{a, b, c}$ is the largest self-defended set contained in $\cff{}(\emptyset)$.
\end{example}

We prove now the correctness of the above construction by showing that the limits of the above streams correspond indeed to the desired fixpoints. First of all the following important lemma shows how conflict-freeness is preserved by the above process of iteration of the defense function:
\begin{lemma} \label{lemma:preserve_cf}
Let $\Frame$ be a finitary attack graph and $X \subseteq \A$ be admissible. Then for any $n$ s.t. $0 \leq n < \omega$,
$$
X \subseteq \cf{\Frame}^n(X) \subseteq \cff{\Frame}(\cf{\Frame}^n(X)) \subseteq \cff{\Frame}(X).
$$
\end{lemma}
That is, each $\cf{\Frame}^n(X)$ in the stream of iteration of the defense function from an admissible set $X$ is a conflict-free set. The lemma can be seen as a reformulation of \cite[Lemma 10]{dung95acceptability}, known as Dung's {\em fundamental lemma}.\footnote{It also generalises \cite[Lemma 46]{dung95acceptability} to the case of $X$ admissible, instead of $X =\emptyset$.}

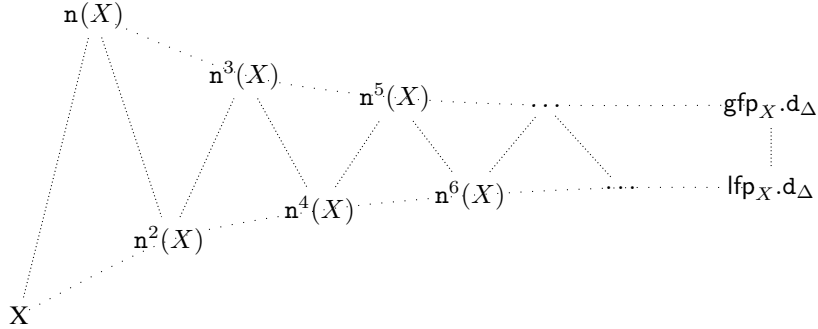
\begin{figure}[t]
\centering
\begin{tikzpicture}

\draw

(0,0) node (0) {X}
(1,4) node (1) {$\cff{}(X)$}

(2,1) node (2) {$\cff{}^2(X)$}
(3,3.2) node (3) {$\cff{}^3(X)$}

(4,1.4) node (4) {$\cff{}^4(X)$}
(5,2.9) node (5) {$\cff{}^5(X)$}

(6,1.6) node (6) {$\cff{}^6(X)$}
(7,2.8) node (7) {$\ldots$}
(10,2.8) node (9) {$\gfp_X.\cf{\Frame}$}

(8,1.7) node (8) {$\ldots$}
(10,1.7) node (10) {$\lfp_X.\cf{\Frame}$}
;

\draw[densely dotted] (0) -- (1);
\draw[densely dotted] (1) -- (2);
\draw[densely dotted] (2) -- (3);
\draw[densely dotted] (3) -- (4);
\draw[densely dotted] (4) -- (5);
\draw[densely dotted] (5) -- (6);
\draw[densely dotted] (6) -- (7);
\draw[densely dotted] (7) -- (8);
\draw[densely dotted] (10) -- (9);

\draw[thin,loosely dotted] plot[smooth] coordinates {(0,0) (2,1) (4,1.4) (6,1.6) (8,1.7) (10,1.7)};
\draw[thin,loosely dotted] plot[smooth] coordinates {(1,4) (3,3.2) (5,2.9) (7,2.8) (10,2.8)};

\end{tikzpicture}
\caption{Streams generated by indefinite iteration of $\cff{}$ applied to an admissible set $X$ (in symbols, $\cff{}^\omega(X)$). Position with respect to the horizontal axis indicates the number of iterations (growing from left to right), while position with respect to the vertical axis indicates set theoretic inclusion. The lower stream $X, \cff{}^2(X), \cff{}^4(X), \ldots$ stabilizes to $\lfp_X.\cf{\Frame}$. The upper stream $\cff{}(X), \cff{}^3(X), \cff{}^5(X), \ldots$ stabilizes to $\gfp_X.\cf{\Frame}$.}
\label{figure:decomposition1}
\end{figure}

We can show that the above streams obtained through the process of iteration of the defense function construct the desired fixpoints:
\begin{lemma} \label{theorem:approxim8}
Let $\Frame$ be a finitary attack graph and $X \subseteq \A$ be admissible:
\begin{align}
\lfp_X.\cf{\Frame} & = \bigcup_{0 \leq n < \omega} \cf{\Frame}^n(X) \label{eq:below} \\
\gfp_X.\cf{\Frame} & = \bigcap_{0 \leq n < \omega} \cf{\Frame}^n(\cff{\Frame}(X)) \label{eq:above}
\end{align}
\end{lemma}
Notice that since $\cff{}^2 = \cf{}$ (Fact \ref{fact:simple}), a stream generated by the indefinite iteration of the defense function can actually be viewed as a stream generated by the indefinite iteration of the neutrality function. So equations \eqref{eq:below} and \eqref{eq:above} of Theorem \ref{theorem:approxim8} can be rewritten as follows:
\begin{align}
\lfp_X.\cf{\Frame} & =  \bigcup_{0 \leq n < \omega} (\cff{\Frame}^2)^n(X) \\
\gfp_X.\cf{\Frame} & =  \bigcap_{0 \leq n < \omega} (\cff{\Frame}^2)^n(\cff{\Frame}(X))
\end{align}
In this light, Lemmas \ref{lemma:smallest} and \ref{theorem:approxim8} capture several of the key features of the stream generated by the indefinite iteration of the neutrality function on an admissible set $X$.
First, the stream can be split into two parts, the part consisting of even and, respectively, odd iterations of $\cff{\Frame}$.
Second, the stream of even iterations converges to a limit which is the smallest complete set including $X$, and the stream of odd iteration converges to a limit which is the largest self-defended set contained in $\cff{\Frame}(X)$ (that is, not attacked by $X$).\footnote{Again the finitariness assumption in the theorem could be lifted by making use of transfinite induction. Cf. Remark \ref{remark:ordinal} below.} Notice that such a set is just free of conflict with respect to $X$, but it is not necessarily conflict-free, and hence it is not necessarily a complete extension.
Third, the two streams can actually be viewed as streams of the defense function $\cf{\Frame}$ applied to $X$ and, respectively, to $\cff{\Frame}(X)$.
Fourth, the two parts grow towards each other as the stream of even iterations is increasing, while the one of odd iterations is decreasing. See Figure \ref{figure:decomposition1} for an illustration.

\begin{remark}
The proof of Lemma \ref{theorem:approxim8} relies in an essential manner on the finitariness assumption on the underlying framework. The assumption simplifies the proof but, it should be stressed, could be lifted.
For infinite graphs which are not finitary, the lemma could be proved by resorting to  transfinite induction:
\begin{eqnarray*}
\cf{\Frame}^0(\emptyset) & = & X \\
\cf{\Frame}^{\alpha + 1}(\emptyset) & = & \cf{\Frame}(\cf{\Frame}^{\alpha}(X) \\
\cf{\Frame}^{\lambda} & = & \bigcup_{\alpha < \lambda} \cf{\Frame}^{\alpha}(X) \ \ \mbox{        (for $\lambda$ arbitrary limit ordinal)}.
\end{eqnarray*}
By the monotonicity of $\cf{\Frame}$ it can then be shown that there exists an ordinal $\alpha$ of cardinality at most $|A|$ such that: $\lfp_X.\cf{\Frame} = \cf{\Frame}^{\alpha}(X)$. A proof of this statement in the general setting of complete partial orders can be found in \cite[Ch. 3]{venema08lectures}. Transfinite induction relies on the Axiom of Choice (or equivalent formulations such as Zorn's Lemma or the Well-Ordering Principle), which is known to be required for the existence results of the standard Dung semantics (cf. \cite{baumann15infinite}).
 \label{remark:ordinal}
 \end{remark}


\subsubsection{Construction of Complete Extensions}

With the above results in place, one can then show how complete extensions can be constructed through a process of fixpoint approximation.
\begin{theorem} \label{lemma:smallest}
Let $\Frame$ be a finitary $AF$ and $X \subseteq \A$ be admissible. Then the limit
$\bigcup_{0 \leq n < \omega} \cf{\Frame}^n(X)$
is the smallest complete extension of $\Frame$ that includes $X$.
\end{theorem}
The theorem establishes that, in finitary frameworks, any complete set can be computed via a process of iteration at $\omega$ of the defense function, starting with some admissible set. It is a novel simple generalization of the earlier result in \cite{dung95acceptability} for the case of $X$ = $\emptyset$.
This process starts by including the arguments that have no attackers or that belong to an initial admissible set, then including those arguments that are defended by the first set of arguments included, and so on.\footnote{Cf. Remark \ref{remark:ordinal}.} At an intuitive level, the theorem states that the indefinite iteration of the defense from an admissible $X$ can be considered as a formalization of the process whereby an agent constructs a rational argumentative position---a complete extension---starting from $X$.


\subsubsection{Construction of Other Dung Extensions} \label{Sec:pref}

Theorem \ref{lemma:smallest} also yields  a constructive proof of existence (in finitary graphs) for the grounded extension.
By setting $X = \emptyset$ (the trivially admissible set), the theorem returns the known result for the construction of the grounded extension $\lfp.\cf{\Frame} =  \bigcup_{0 \leq n < \omega}\cf{\Frame}^n(\emptyset) $ \cite[Th. 47]{dung95acceptability}. In this section we show how the theorem relates to the other classical Dung extensions.

\smallskip

Specific conditions can be identified which guarantee that the indefinite iteration of the defense function constructs preferred and stable extensions from a given admissible set $X$.
It is fairly easy to see that if the chosen admissible set $X$ of $\Frame$ is `big enough' in the precise sense that it contains enough arguments to be able, from some argument in $X$, to reach any argument in the graph via the attack relation, i.e., if $\A \subseteq \set{a \mid \exists b \in X: b \ar^+ a}$, then the stream of iterations of $\cf{\Frame}$ from $X$ converges to a complete extension containing $X$ (by Theorem \ref{lemma:smallest}), but this extension is now maximal as all arguments in $\Frame$ can be reached from $X$.

The condition under which Theorem \ref{lemma:smallest} constructs stable extensions is particularly interesting. If the streams of even and odd iterations of the neutrality function (recall Figure \ref{figure:decomposition1}) converge to the same limit, then the process of fixpoint approximation defines a stable extension:
\begin{fact}
Let $\Frame$ be a finitary $AF$ and $X \subseteq \A$ be admissible. If
$\bigcup_{0 \leq n < \omega} \cf{\Frame}^n(X) = \bigcap_{0 \leq n < \omega} \cf{\Frame}^n(\cff{\Frame}(X))$, then both sets coincide with
the unique stable extension of $\Frame$ that includes $X$.
\end{fact}
\begin{proof}
By Theorem \ref{lemma:smallest}, $\bigcup_{0 \leq n < \omega} \cf{\Frame}^n(X)$ is the smallest complete extension containing $X$. However, as $\bigcup_{0 \leq n < \omega} \cf{\Frame}^n(X) = \bigcap_{0 \leq n < \omega} \cf{\Frame}^n(\cff{\Frame}(X))$ by assumption, the set $\bigcup_{0 \leq n < \omega} \cf{\Frame}^n(X)$ is therefore also a fixpoint of the neutrality function, and therefore a stable extension.
\end{proof}
This observation is, to the best of our knowledge, novel and provides a characterization of the existence of a stable extension that includes a given admissible set.


\begin{example}[Construction of complete extensions in Figure \ref{figure:graphs}]
Consider the rightmost $AF$. Starting with the admissible set $\set{a}$, the non-decreasing stream
$$
\set{a}, \set{a,d}, \set{a,d}, \ldots
$$
converges after one step to the smallest complete extension $\lfp_{\set{a}}.\cf{} = \set{a,d}$ containing $\set{a}$. The non-increasing stream
$
\set{a,d,e}, \set{a,d}, \set{a,d}, \ldots
$
converges to $\gfp_{\set{a}}.\cf{} = \set{a,d}$, i.e., to the same set  which is also the largest fixpoint of $\cf{}$ included in $\cff{}(\set{a}) = \set{a,d,e}$. As $\set{a,d}$ is also conflict-free, it is a preferred extension. Notice also that if we were to start with $\emptyset$, the resulting streams would be $\emptyset, \emptyset, \ldots$ and $\set{a,b,c,d,e},$ $ \set{a,b,c,d,e}, \ldots$. The first one constructs the grounded extension $\emptyset$ of the $AF$, and the second the largest fixpoint of $\cf{}$ in the $AF$, that is, $\set{a,b,c,d,e}$.
\end{example}


\begin{figure}[t]
\begin{center}
\includegraphics[scale=0.25]{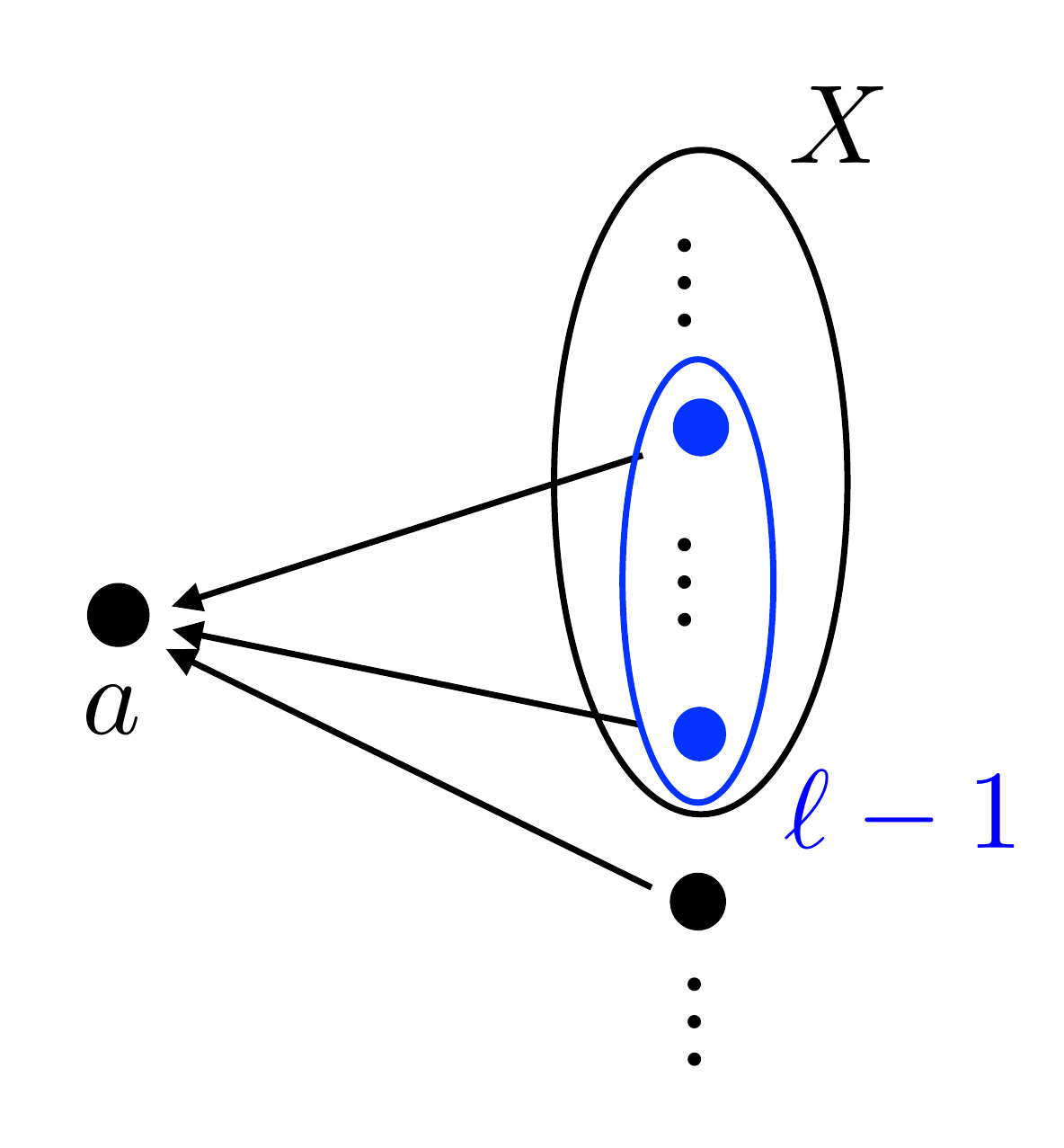}
\hspace{1cm}
\includegraphics[scale=0.25]{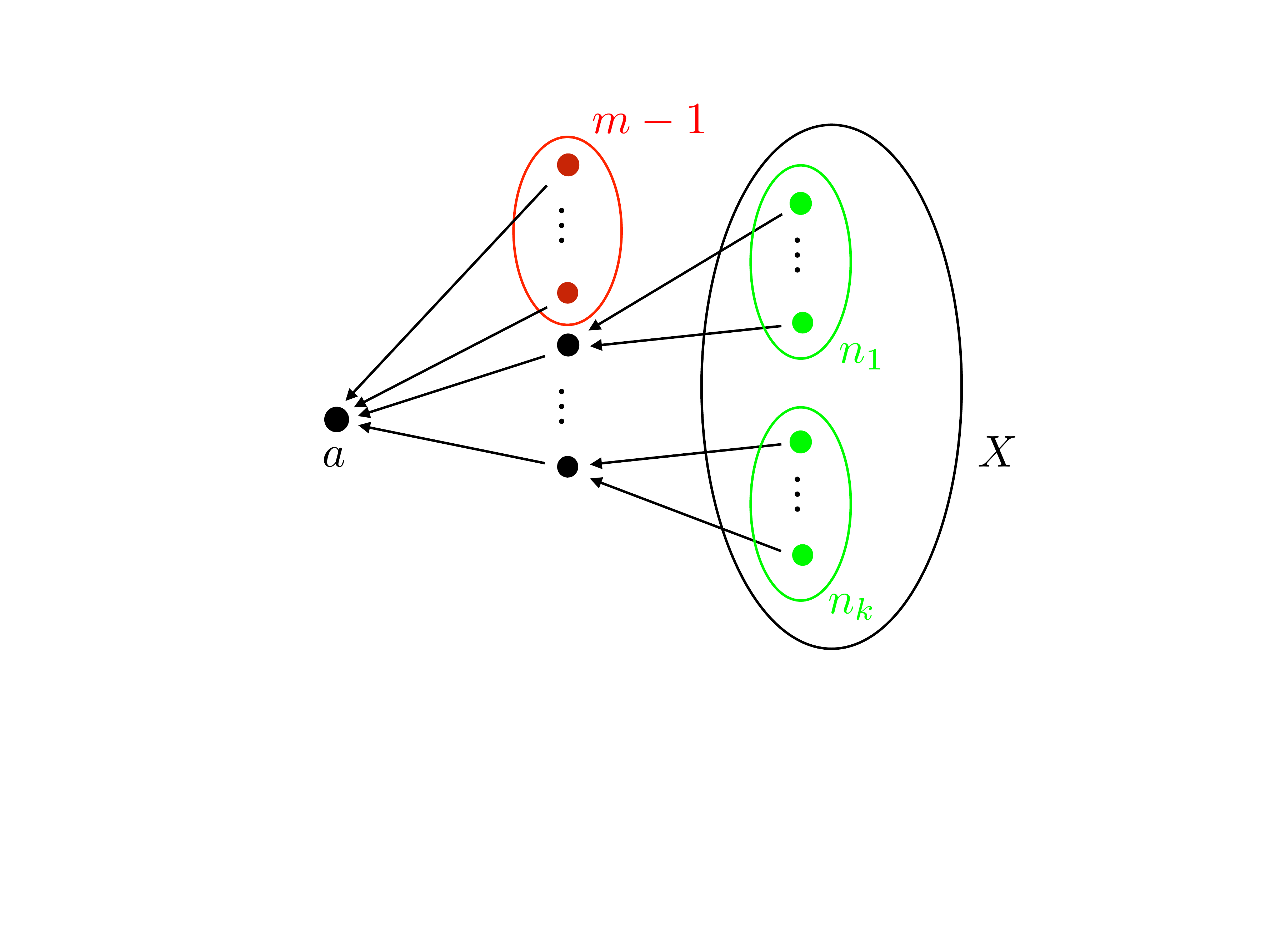}
\end{center}
\caption{Variants of neutrality (left) and defense (right). The lower integer $\ell$ is, the more neutral is the set of arguments $X$ with respect to $a$ (left). The lower integer $m$ is and the higher integer $n = \min(n_1, \ldots, n_k)$ is, the better is $a$ defended by $X$ (right). Standard neutrality and defense correspond to $\ell = m = n = 1$.}
\label{figure:motivation2}
\end{figure}

\section{Graded Acceptability}\label{Sec:Graded Acceptability}

We now turn to the main contribution of this paper: a graded generalisation of Dung's acceptability semantics.
We first introduce the intuitions behind our generalisation, and then define and study graded variants of Dung's defense and neutrality functions, which capture the proposed intuitions. These functions will then be used later (Section \ref{Sec:GradedSemantics}) to define and study a family of graded variants of Dung's semantics, and the rankings they enable (Section \ref{Sec:Ranking}).


\subsection{Introducing Graded Acceptability: Intuitions}\label{Sec:GradedAcceptabilityIntuitions}

The central tenet of argumentation theory is that any individual argument cannot, in and of itself, constitute definitive grounds for believing that a claim is true. Rather, the status of an epistemic claim as true justified belief is not established  by an individual argument, but   through the dialectical consideration of counter-arguments and defenders of these counter-arguments \cite{SMReasoner}. Similarly in practical reasoning, the status of (a claim representing) a decision option supported by an individual practical argument is not considered to be the option that simply maximises a given objective, but rather the best option contingent on having dis-preferred alternative options and refuted challenges made to the
assumptions made in support of the argument. Pragmatically however, we operate under the assumption that the claim is true/the decision option is the best available, to the extent that as of yet we know of no good reason to suppose otherwise. Argumentatively, a claim can be considered established only in as much as it is the claim of a justified argument included in a network of interrelated arguments and counter-arguments. Dung's abstract argumentation theory captures these principles by assuming {\em sets of} arguments, rather than individual arguments, as the units of analysis, and studying formal criteria (semantics) for sets of arguments to be acceptable. Apart from trivial cases (unattacked arguments), arguments are acceptable only as members of a set of acceptable arguments. The graded theory of acceptability that we aim at, captures a notion of graduality while at  the same time retaining the notion of a set of arguments as the central unit of analysis.

\subsubsection{Graded neutrality}


According to the standard definition of neutrality (Definition \ref{definition:neutrality}) a set $X$ is neutral with respect to $a$ if $a$ is not attacked by {\em any} argument in $X$.
A less demanding criterion of neutrality of $X$ with respect to $a$ would require that there exists {\em at most one} attacker of $a$ in $X$, a yet less demanding one (or at least not `as demanding as') would require that there exist {\em at most two} attackers of $a$ in $X$, and so on. Intuitively, these weakened neutrality criteria capture the idea that one (two, three, \ldots) attackers are not enough to rule out the co-acceptability of $a$ and its attacking arguments.\footnote{Using terminology from logic, this may be viewed as an argumentative form of paraconsistency.  In more `human orientated' argumentation formalisms (e.g., \cite{dunne11weighted} and developments thereof), this may be viewed as accounting for an attacking argument not establishing definitive grounds for its claim (as discussed at the beginning of Section \ref{Sec:GradedAcceptabilityIntuitions}), and hence not definitively ruling out the claim of the attacked argument.}
One then obtains a natural way to generalise the neutrality function (Definition \ref{definition:sensitive-neutrality} below) by making explicit a numerical level $\ell$ of neutrality of a set of arguments with respect to a given argument, as depicted in Figure \ref{figure:motivation2} (left).\footnote{Weighted Argument Systems \cite{dunne11weighted} propose a somewhat similar idea, whereby an inconsistency budget sets a threshold on the number of attacks that can be tolerated within a given set. However, notice that our notion of a threshold set, yielded by graded neutrality, is local in the sense that it pertains to the incoming attacks \emph{on each individual argument}. We will later compare this and other related approaches in more detail (Section \ref{Sec:RelatedWork}).} So we say that $a$ is $\ell$-neutral with respect to $X$ whenever there are at most $\ell-1$ attackers of $a$ in $X$.

\subsubsection{Graded defense}

According to the standard notion of defense, an argument $a$ is defended by a set of arguments $X$ whenever {\em every} attacker of $a$ is attacked by {\em some} argument in $X$. The quantification pattern (`for all', `some') involved in this definition offers again a natural handle to generalise the notion of defense. If {\em all but at most one} attackers of $a$ are attacked by at least one argument in $X$, the quality of such defense (and hence the extent to which $a$ is acceptable w.r.t. $X$) can reasonably be considered `lower' than in the case in which \emph{all} arguments are counter-attacked by at least one argument in $X$. But the former quality of defense is still `higher' than (or at least not `as low as') the case in which {\em all but at most two} attackers are counter-attacked by at least one argument  in $X$, and so on.
Similarly, if all attackers of $a$ are counterattacked by {\em at least two} arguments in $X$, then the quality of this defense can reasonably be considered `higher' than in standard acceptability, but `lower' than (or at least not `as high as') the case in which all attackers of $a$ are counterattacked by {\em at least three} arguments in $X$. Combining these intuitions---depicted in Figure \ref{figure:motivation2} (right)---one obtains a way to generalise the defense function (Definition \ref{definition:sensitive} below) by making explicit, through numeric grades ($m$ and $n$) of the above type, how well a set $X$ defends an argument $a$. So we say that $X$ $mn$-defends $a$ whenever there are at most $m-1$ attackers of $a$, which are not counterattacked by at least $n$ arguments in $X$.

This notion of graded defense is related in a natural way to the above notion of graded neutrality: 
the set of arguments that are $mn$-defended by $X$, is the set of arguments which is not attacked by at least $m$ arguments, that are not in turn attacked by at least $n$ arguments in $X$ (i.e., arguments that are $m$-neutral with respect to the set of arguments that are $n$-neutral with respect to $X$). In other words, the notion of tolerance towards attack (graded neutrality) can be iterated to obtain a notion of graded defense. Fact \ref{fact:properties_dn} will establish this claim formally. We illustrate the above intuitions with a few examples.

\begin{example}\label{Ex:intuitions}
In Figures \ref{Motivating1}i) -- \ref{Motivating1}iv), the encircled set $Xi$ defends $ai$ ($i = 1 \ldots 4$) under Dung's Definition \ref{definition:defense}. However we can differentiate these cases based on the number of attackers and defenders of $ai$.
For instance, $X2$ more strongly defends $a2$ than $X1$ defends $a1$, as $a2$ is defended by two arguments whereas $a1$ is defended by one argument (i.e., the standard of defense that allows at most 0 attackers to not be defended by 2 arguments is met by $X2$'s defense of $a2$ but not by $X1$'s defense of $a1$). We will later, in Example \ref{Ex:DEJ}, reference the defense of $a3$ by $X3$ and of $a4$ by $X4$ to illustrate that neither can be said to be a more strong defense than the other.
While neither $X5$ or $X6$ defend $a5$, respectively $a6$, under Dung's Definition \ref{definition:defense}, observe that
$X5$'s defense of  $a5$ is  stronger than $X6$'s defense of $a6$. The former
 meets a standard of defense that allows at most one attacker ($d5$) not to be defended by at least one defender (which goes hand in hand with accommodating the co-acceptability of $a5$ with at most one undefended attacker; i.e., $d5$ is $2$-neutral with respect to $X5$). This standard is not met by $X6$'s defense of $a6$, since $a6$ is attacked by two undefended attacks (from $d6$ and $e6$). In the latter case, a weaker standard of defense is met, which again goes hand in hand with accommodating the co-acceptability of $a6$ with its two undefended attackers. These notions then naturally generalise so that one can discriminate standards of defense based only on the number of attackers. Allowing at most one attacker not to be defended by two arguments, is a standard of defense met by $X1$s defense of $a1$, but not $X3$'s defense of $a3$ (the  former defense thus being stronger than the latter).
\end{example}

\begin{figure}[h!!]
\centering
\includegraphics[scale=0.5]{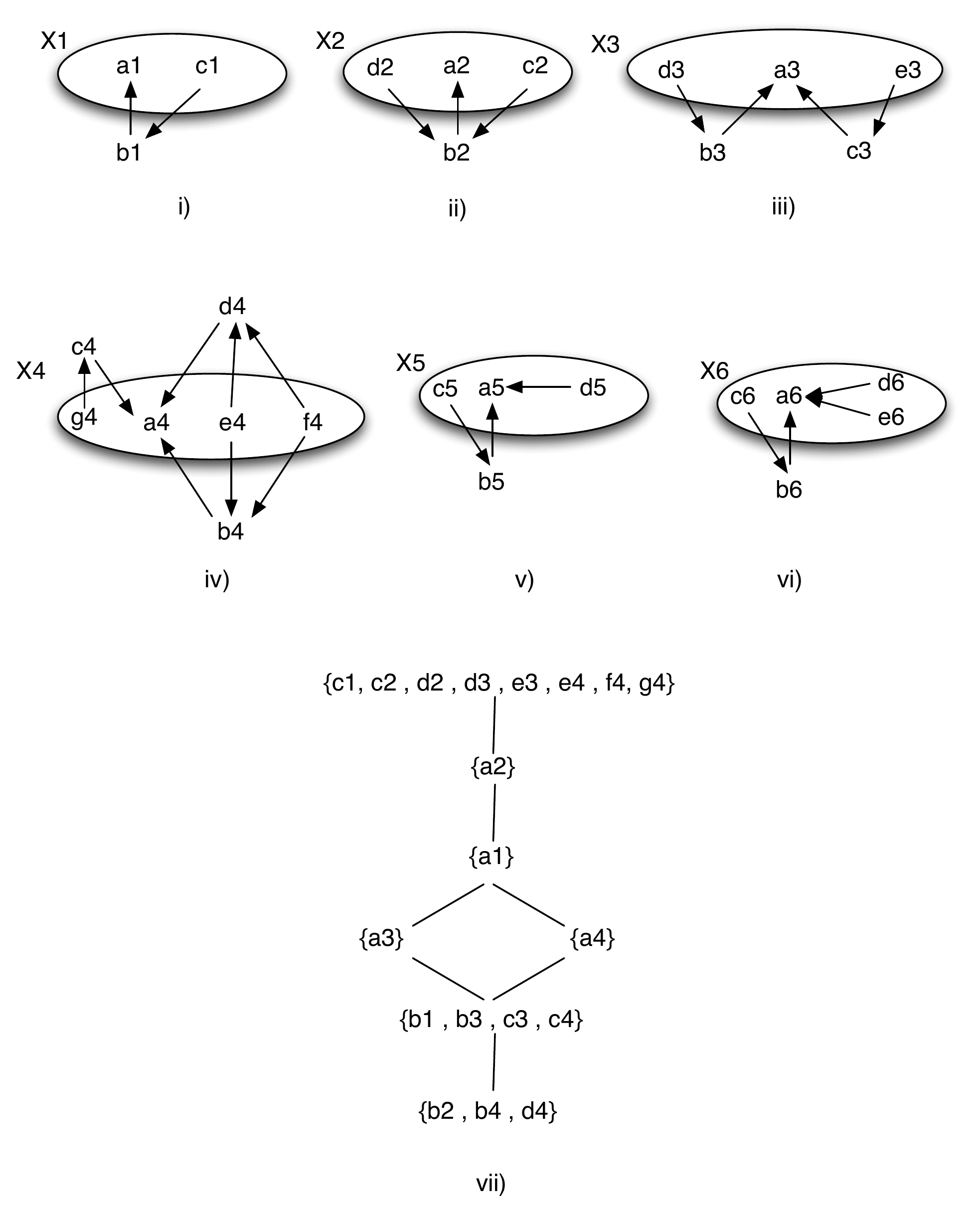}
\caption{i)--vi) $X_{i\in\{1\ldots6\}}$ defending $a_1,\ldots a_6$. vii) is the Hasse diagram of a preorder on the set of all arguments from i)--v) relevant for Example \ref{Ex:DEJ}.}
\label{Motivating1}
\end{figure}


\subsection{Graded Defense and Neutrality Functions}\label{Sec:DefiningDefense}
Let us move now to the formal definitions of graded defense and neutrality. Take an argument $a$ and a set of arguments $X$. Let $k$ be the number of $a$'s attackers ($b_1, \ldots, b_k$) and, for each $b_i$ ($1 \leq i \leq k$) let $n_i$ be the (non zero) number  of attackers of $b_i$ (i.e., defenders of $a$) in $X$. Finally, let $n$ be the minimum among the $n_i$s, i.e., $n = \min(\set{n_i}_{1 < i \leq k})$. We can now count the number $m$ ($\leq k$) of attackers of $a$, which are counter-attacked by at least $n$ arguments in $X$. Integers  $m$ and $n$ therefore encode
 information about how strongly $a$ is defended by $X$, in the sense that they express a maximum number (i.e., $m - 1$) of attackers of $a$ which are not counterattacked by a minimum given number (i.e., $n$) of arguments in $X$. We can now generalise Definition \ref{definition:defense} as follows:

\begin{definition}[Graded defense] \label{definition:sensitive}
Let $\Delta = \tuple{\A, \ar}$ be an $AF$ and let $m$ and $n$ be two positive integers ($m,n > 0$).
The graded defense function $\cf{\substack{m \\ n}} : \pw{\A} \map \pw{\A}$ for $\Delta$ is defined as follows. For any $X \subseteq \A$:
\begin{align*}
\cfmn{m}{n}(X) & = \set{x \in \A \mid \nexistsn{y}{m}:  [~y \ar x \AND \nexistsn{z}{n}:  [~z \ar y \AND z \in X~]~]}
\end{align*}
where $\exists_{\geq n} x$, for integers $n$ (`there exist at least $n$ arguments $x$') are the standard first-order logic counting quantifiers.\footnote{Cf. \cite{dalen80logic}. The definition can be reformulated without counting quantifiers as follows:
$$
\cf{\substack{m \\ n}}(X) = \set{ x \in \A  \ST |\set{y \in \overline{x}  \ST |\set{z \in \overline{y} \cap X  }| < n }| < m }
$$
where we write $\overline{x}$ to denote $\set{y \in \A \mid y \rightarrow x}$.
}
In the rare cases in which we need to make $\Delta$ explicit we write $\cfmn{m}{n}^\Delta$.
\end{definition}
So, $\cfmn{m}{n}(X)$ is the set of arguments (in the given framework) which have at most $m-1$ attackers that are not counter-attacked by at least $n$ arguments in $X$.

\begin{example}
In Figure \ref{Motivating1}, $a1 \in \cf{\substack{1 \\ 1}}(X1)$ and $a2 \in \cf{\substack{1 \\ 1}}(X2)$ since in both cases the following holds: at most $0$
arguments attacking $a1$, respectively $a2$, are not attacked by at least one argument in $X1$, respectively $X2$. However if we increment $n$ by $1$ we have that: $a2 \in \cf{\substack{1 \\ 2}}(X2)$  but $a1 \notin \cf{\substack{1 \\ 2}}(X1)$ since \emph{it is} the case that at least one argument attacking $a1$ is not attacked by at least two arguments in $X1$. Intuitively, this standard of defense allows for up to $0$ attackers to not be counter-attacked by two defenders, a standard met by $X2$'s defense of $a2$, but not by $X1$'s defense of $a1$.
Continuing with Figure \ref{Motivating1},
$a5 \in \cf{\substack{3 \\ 1}}(X5)$ and $a6 \in \cf{\substack{3 \\ 1}}(X6)$, since in both cases the standard of defense that allows for no more than 2 unattacked arguments is met. However,  $a5 \in \cf{\substack{2 \\ 1}}(X5)$ and $a6 \notin \cf{\substack{2 \\ 1}}(X6)$ since this standard of defense accommodates up to a maximum of 1 unattacked attackers, and in the latter case there is more than one unattacked attacker of $a6$. Finally,
  $a1 \notin \cf{\substack{1 \\ 2}}(X1)$ and $a3 \notin \cf{\substack{1 \\ 2}}(X3)$ since the $\cf{\substack{1 \\ 2}}$ standard of defense requires that all attackers of $a1$ ($a3$) are attacked by at least two arguments. However,   $a1 \in \cf{\substack{2 \\ 2}}(X1)$ and $a3 \notin \cf{\substack{2 \\ 2}}(X3)$, since the standard of defense allowing at most one attacker not to be defended by two counter-attackers is met by $a1$ but not by $a3$.
Notice that in this last case the two arguments $a1$ and $a3$ are discriminated based on the number of their attackers.
\end{example}


By the same logic, Definition \ref{definition:neutrality} can be generalised as follows:
\begin{definition}[Graded neutrality function] \label{definition:sensitive-neutrality}
Let $\Delta = \tuple{\A, \ar}$ be an $AF$ and let $\ell$ be any positive integer.
The graded neutrality function $\cff{l} : \pw{\A} \map \pw{\A}$ for $\Delta$ is defined as follows. For any $X \subseteq \A$:
\begin{align*}
\cff{\ell}(X)  & =  \set{x \in \A \mid \nexistsn{y}{\ell}: y \ar x \AND y \in X}.
\end{align*}
\end{definition}
So, given a set of arguments $X$, $\cff{\ell}(X)$ denotes the set of arguments which have at most $\ell - 1$ attackers in $X$.\footnote{Equivalently, graded neutrality can be defined as follows, without the use of counting quantifiers:
$$
\cff{\ell}(X)   =  \set{x \in \A \ST |\overline{x} \cap X|<\ell}.
$$
}

\begin{example}\label{IllustratingNeutrality}
In Figure \ref{Motivating1}, $n_2(X_5) = \{a5,b5,c5,d5\}$. Notice that $n_2(\{a5,b5,c5,d5\}) = \{b5,c5,d5\}$.
Also,
$n_2(X_6) = \{b6,c6,d6,e6\}$.

\end{example}

\subsection{Properties of Graded Defense and Neutrality}

The following two facts show that the graded defense and neutrality functions
 are generalisations of the standard functions defined in Definitions \ref{definition:defense} and \ref{definition:neutrality}, and that such generalisations remain well-behaved in the sense that they retain many of the key features of their standard variants.

\begin{fact} \label{fact:properties_dn}
For any $AF$ $\Delta = \tuple{\A, \ar}$, $m$, $n$ and $\ell$ positive integers:
\begin{eqnarray}
\cf{\substack{1 \\ 1}}(X) & = & \cf{}(X) \label{eq:defense} \\
\cff{1}(X) & = & \cff{}(X)  \label{eq:neutrality} \\
X \subseteq Y & \IMPLIES & \cff{\ell}(Y) \subseteq \cff{\ell}(X) \label{formula:antitonicm} \\
X \subseteq Y & \IMPLIES &  \cf{\substack{m \\ n}}(X) \subseteq  \cf{\substack{m \\ n}}(Y) \label{formula:monotonicmn} \\
\cff{m}(\cff{n}(X)) & = & \cf{\substack{m \\ n}}(X) \label{eq:twofold}
\end{eqnarray}
\end{fact}
\begin{proof}

Equation \eqref{eq:defense} follows from the fact that Definition \ref{definition:defense} can be retrieved from Definition \ref{definition:sensitive} by setting $n = m = 1$. Similarly \eqref{eq:neutrality} follows from the fact that Definition \ref{definition:neutrality} can be retrieved from Definition \ref{definition:sensitive-neutrality} by setting $\ell = 1$.
Equation \eqref{eq:twofold} follows from Definitions \ref{definition:sensitive} and \ref{definition:sensitive-neutrality} by the following series of equations:
\begin{eqnarray*}
\cff{m}(\cff{n}(X)) & = & \cff{m}(\set{y \in A \mid \nexistsn{z}{n}: [~z \ar y \AND z \in X] }) \\
                              & = & \set{x \in A \mid \nexistsn{y}{m}: [~y \ar x \AND \nexistsn{z}{n}: [~z \ar y \AND z \in X]]} \\
                              & = & \cfmn{m}{n}(X)
\end{eqnarray*}
Formulae \eqref{formula:antitonicm} and \eqref{formula:monotonicmn} are direct consequences of Definitions \ref{definition:sensitive} and \ref{definition:sensitive-neutrality}.
\end{proof}
%

Equation \eqref{eq:defense} reformulates $\cf{}(X)$ as the set of arguments for which it is not the case that there are one or more attackers, which are not counter-attacked by one or more arguments in $X$; that is, no attacker is not attacked by some argument in $X$. So does Equation \eqref{eq:neutrality}  for $\cff{}(X)$. The remaining formulae generalise Fact \ref{fact:simple} to the graded setting. In particular, graded defense is monotonic \eqref{formula:monotonicmn}, graded neutrality is antitonic \eqref{formula:antitonicm}, and
 equation \eqref{eq:twofold} shows that, as in the standard case, the defense function is the two-fold iteration of the neutrality function (as in the standard case we may use the notation $\cff{n} \circ \cff{m}$ to denote this composition).

Importantly, the continuity of the defense function is also preserved in the graded setting:
\begin{fact}[$\omega$-continuity of graded defense] \label{fact:graded_continuous}
If $\Frame$ is finitary, then function $\cfmn{m}{n}$ is \mbox{(upward-)} continuous for any $X \subseteq \A$, $m$ and $n$ positive integers. I.e., for any upward directed set $D$ of finite subsets of $\A$:
\begin{align}
\cfmn{m}{n}(\bigcup_{X \in \mathcal{D}} X) & = \bigcup_{X \in \mathcal{D}} \cfmn{m}{n}(X) \label{eq:scott}
\end{align}
Similarly, $\cfmn{m}{n}$ is (downward-)continuous for any $X \subseteq \A$, $m$ and $n$ positive integers. I.e., for any downward directed set $D$ of finite subsets of $\A$:
\begin{align}
\cfmn{m}{n}(\bigcap_{X \in \mathcal{D}} X) & = \bigcap_{X \in \mathcal{D}} \cfmn{m}{n}(X) \label{eq:scottt}
\end{align}
\end{fact}
\begin{proof}[Sketch of proof]
The argument used to prove Fact \ref{fact:continuous} carries through in exactly the same manner, exploiting the monotonicity of $\cfmn{m}{n}$ \eqref{formula:monotonicmn} and the finitariness assumption over $\Frame$.
\end{proof}

Finally, we establish some properties showing how the values for the defence and neutrality functions are affected by varying the parameters $m$ and $n$.
\begin{fact} \label{fact:accrual_relations}
For any $AF$ $\Delta = \tuple{\A, \ar}$, $X \subseteq \A$, and $\ell$, $m$ and $n$ positive integers:
\begin{eqnarray}
\cff{\ell}(X) & \subseteq &\cff{\ell+1}(X) \label{eq:increase1} \\
\cfmn{m}{n}(X) & \subseteq & \cfmn{m+1}{n}(X) \label{eq:increase2}  \\
\cfmn{m}{n}(X) & \supseteq & \cfmn{m}{n+1}(X) \label{eq:increase3}
\end{eqnarray}
\end{fact}
\begin{proof}
Recall the definition of the neutrality function (Definition \ref{definition:sensitive-neutrality}). To establish \eqref{eq:increase1} it suffices to notice that the property expresses the contrapositive of the following statement: if there exist at least $\ell+1$ attackers in $X$ then there exist at least $\ell$ attackers in $X$.
Property \eqref{eq:increase2} then follows  directly by \eqref{eq:increase1} above and \eqref{eq:twofold} (Fact \ref{fact:properties_dn}), through the following series of relations:
\begin{align*}
\cfmn{m}{n}(X) & = \cff{m}(\cff{n}(X)) \\
                         & \subseteq \cff{m+1}(\cff{n}(X)) = \cfmn{m+1}{n}(X).
\end{align*}
A similar argument applies to establish \eqref{eq:increase3}, which follows by \eqref{eq:increase1} above, \eqref{eq:twofold}, and the antitonicity of $\cff{}$ (Fact \ref{fact:properties_dn}):
\begin{align*}
\cfmn{m}{n}(X) & = \cff{m}(\cff{n}(X)) \\
                         & \supseteq \cff{m}(\cff{n+1}(X)) = \cfmn{m}{n+1}(X).
\end{align*}
This completes the proof.
\end{proof}
Intuitively, \eqref{eq:increase1} states that the set of arguments attacked by at most $\ell$ arguments in $X$ is included in the set of arguments attacked by at most $\ell+1$ arguments in $X$. This establishes an ordering, in terms of logical strength, among the values of different neutrality functions: the lower is $\ell$ the stricter is the value of $\cff{l}$ applied to a same set of arguments $X$. Properties \eqref{eq:increase2} and \eqref{eq:increase3} then follow by combining this simple fact with the fact that $\cfmn{m}{n}$ is the composition of $\cff{m}$ with $\cff{n}$ \eqref{eq:twofold}.
%
%
%



\subsection{Comparing Graded Defense and Neutrality Functions} \label{Sec:RankingGradedDF} 

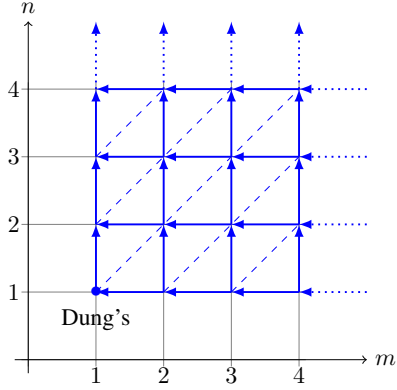
\begin{figure}[t]

\begin{center}
\scalebox{.9}{
\begin{tikzpicture}[domain=0:4]
\draw[very thin,color=gray] (-0.1,-0.1) grid (4.1,4.1);

\draw[->] (-0.2,0) -- (5,0) node[right] {$m$};
\draw[->] (0,-0.2) -- (0,5) node[above] {$n$};

\draw  (1,0)  node [below] {$1$};
\draw  (2,0)  node [below] {$2$};
\draw  (3,0)  node [below] {$3$};
\draw  (4,0)  node [below] {$4$};
\draw (1,1) node [blue] {\textbullet};
\draw (1,0.6) node {Dung's};

\draw  (0,1)  node [left] {$1$};
\draw  (0,2)  node [left] {$2$};
\draw  (0,3)  node [left] {$3$};
\draw  (0,4)  node [left] {$4$};

\draw[blue,thick, -latex] (1,1) -- (1,2);
\draw[blue,thick, -latex] (1,2) -- (1,3);
\draw[blue,thick, -latex] (1,3) -- (1,4);
\draw[blue,thick,dotted, -latex] (1,4) -- (1,5);

\draw[blue,thick, -latex] (2,1) -- (2,2);
\draw[blue,thick, -latex] (2,2) -- (2,3);
\draw[blue,thick, -latex] (2,3) -- (2,4);
\draw[blue,thick,dotted, -latex] (2,4) -- (2,5);

\draw[blue,thick, -latex] (3,1) -- (3,2);
\draw[blue,thick, -latex] (3,2) -- (3,3);
\draw[blue,thick, -latex] (3,3) -- (3,4);
\draw[blue,thick,dotted, -latex] (3,4) -- (3,5);

\draw[blue,thick, -latex] (4,1) -- (4,2);
\draw[blue,thick, -latex] (4,2) -- (4,3);
\draw[blue,thick, -latex] (4,3) -- (4,4);
\draw[blue,thick,dotted, -latex] (4,4) -- (4,5);

\draw[blue,thick, -latex] (2,1) -- (1,1);
\draw[blue,thick, -latex] (3,1) -- (2,1);
\draw[blue,thick, -latex] (4,1) -- (3,1);
\draw[blue,thick,dotted, -latex] (5,1) -- (4,1);

\draw[blue,thick, -latex] (2,2) -- (1,2);
\draw[blue,thick, -latex] (3,2) -- (2,2);
\draw[blue,thick, -latex] (4,2) -- (3,2);
\draw[blue,thick,dotted, -latex] (5,2) -- (4,2);

\draw[blue,thick, -latex] (2,3) -- (1,3);
\draw[blue,thick, -latex] (3,3) -- (2,3);
\draw[blue,thick, -latex] (4,3) -- (3,3);
\draw[blue, thick,dotted, -latex] (5,3) -- (4,3);

\draw[blue, thick, -latex] (2,4) -- (1,4);
\draw[blue, thick, -latex] (3,4) -- (2,4);
\draw[blue, thick, -latex] (4,4) -- (3,4);
\draw[blue, thick,dotted, -latex] (5,4) -- (4,4);

\draw[blue, dashed, -] (1,1) -- (4,4);
\draw[blue, dashed, -] (1,2) -- (3,4);
\draw[blue, dashed, -] (1,3) -- (2,4);
\draw[blue, dashed, -] (2,1) -- (4,3);
\draw[blue, dashed, -] (3,1) -- (4,2);

\end{tikzpicture}
}
\end{center}
\caption{Depiction of the partial order  $\rhd$ over the set of all graded defense functions (with $0 < m,n \in \mathbb{N}$). The horizontal and vertical axes consist of the values of $m$ and, respectively, $n$. Arrows go from `weaker' to `stronger' defense functions. Dashed lines (diagonals) denote incomparability. The point $m = n = 1$ denotes the position that Dung's characteristic function occupies in the ordering.}
\label{figure:order}
\end{figure}

Fact \ref{fact:accrual_relations} provides ground for a natural way in which different graded defense and neutrality functions can be ordered as their parameters $m$ and $n$ vary. The choice of these parameters determines the logical strength  of different `types' or `standards' of conflict-freeness, which is based on neutrality, and acceptability, which is based on defense.

In  light of Fact \ref{fact:accrual_relations}, comparing different neutrality functions is straightforward. Any relaxation  on  the requirement that no  argument in a set be attacked by other arguments in that set  leads to weaker forms of conflict-freeness. For any $X \subseteq \A$, $\cff{\ell}(X) \subseteq \cff{k}(X)$ for $k$ and $\ell$ positive integers whenever $\ell \leq k$. So neutrality functions can simply be ordered linearly like natural numbers, with lower numbers denoting `stronger' forms of neutrality and hence conflict-freeness.

The ordering of defense functions is more interesting, as these functions are parameterized by two integers:

\begin{definition}\label{DefGradedDefFunction} $\cfmn{m}{n} \rhd \cfmn{s}{t}$ (to be read ``is at least as strong as'') iff for any $X \subseteq \A$, $\cfmn{m}{n}(X) \subseteq \cfmn{s}{t}(X)$, with $m,n,s,t$ positive integers.
\end{definition}

Relation $\rhd$ orders the set of all graded defense functions in a well-behaved manner:
\begin{fact}\label{Fact:Ordering}
Let $\Delta = \tuple{\A, \ar}$ be an $AF$, and let $\rhd$ be defined as above. Then:
\begin{enumerate}[(i)]
\item $\cfmn{m}{n} \rhd \cfmn{s}{t}$ iff $m \leq s$ and $t \leq n$;
\item Relation $\rhd$ is a  {\em partial order}, i.e.,  reflexive, antisymmetric and transitive.
\end{enumerate}
\end{fact}
\begin{proof}
\fbox{{\em (i)}} is a direct consequence of Fact \ref{fact:accrual_relations}.
\fbox{{\em (ii)}} follows directly from how relation $\rhd$ is defined and the properties of set inclusion.
\end{proof}
The relation is depicted in its generality in Figure \ref{figure:order}.
Expressions $\cfmn{m}{n} \rhd \cfmn{s}{t}$ may be read as follows: `being $mn$-defended is {\em weakly preferable over} being $st$-defended' or `the $mn$-defense function is {\em at least as strong as} the $st$-defense function'. Intuitively, the partial order $\rhd$ uses logical strength as a way to order graded defense functions. This equates with the intuition that if an argument meets a demanding standard of defense it also meets a less demanding one.

\begin{figure}[t]
\centering
\includegraphics[width=2in]{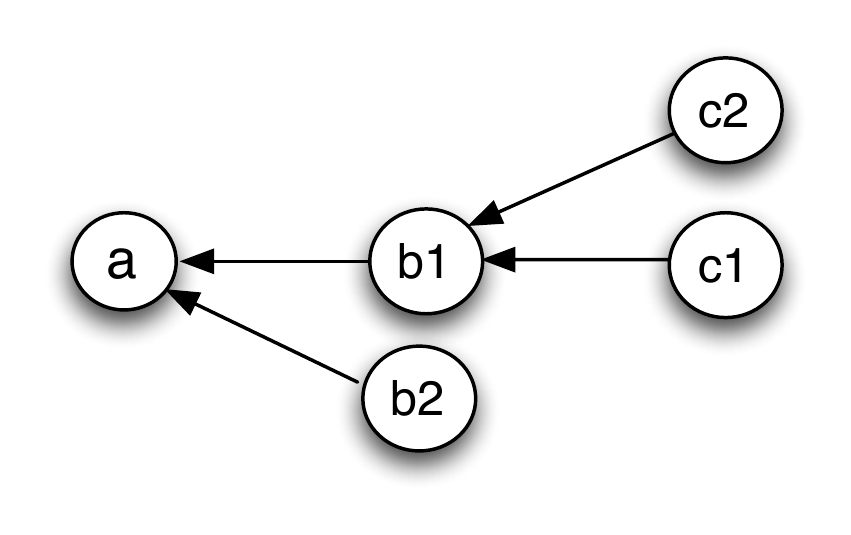}

\caption{Framework of Example \ref{example:defense_neutrality}} \label{MotivatingDEJ}
\end{figure}

\begin{example} \label{example:defense_neutrality}
Referring to the framework in  Figure \ref{MotivatingDEJ}, we illustrate Formula \eqref{eq:increase2} and Fact \ref{Fact:Ordering}: $\cfmn{1}{1}(\{c1,c2\})$ = $\{c1,c2,b2\}$ $\subseteq$ $\cfmn{2}{1}(\{c1,c2\})$ = $\{c1,c2,a,b2\}$ $\subseteq$ $\cfmn{3}{1}(\{c1,c2\})$ = $\{c1,c2,a,b1,b2\}$.
We also illustrate Formula \eqref{eq:increase3} and Fact \ref{Fact:Ordering} with reference to  Figure \ref{Motivating1}ii): $\cfmn{1}{1}(\{d2,c2,a2\})$ = $\{d2,c2,a2\}$ $\supseteq$ $\cfmn{1}{2}(\{d2,c2,a2\})$ = $\{d2,c2,a2\}$ $\supseteq$ $\cfmn{1}{3}(\{d2,c2,a2\})$ = $\{d2,c2\}$.
\end{example}

\subsubsection{On the Partiality of $\rhd$}\label{Sec:ExtendingPartialOrder}

As the relation $\rhd$ is a partial order, some defense functions may be incomparable (see Figure \ref{figure:order}) and this, we claim, is intuitive. By way of example, consider Dung's defense $\cfmn{1}{1}$. This standard of defense is strengthened by $\cfmn{1}{2}$ (higher $n$ parameter) and weakened by $\cfmn{2}{1}$ (higher $m$ parameter). But under the definition of $\rhd$, $\cfmn{2}{2}$ defines a standard of defense which is incomparable with respect to $\cfmn{1}{1}$: it demands more defenders per attacker, but tolerates more attackers that are not counter-attacked to the desired level. In general, incomparability arises every time the parameters of the functions do not meet the condition $m \leq s$ and $t \leq n$ of Fact \ref{Fact:Ordering}.


It should be clear, however, that the partial order $\rhd$ over graded defense functions could be further refined to a total order by resolving incomparability. This can be done in two ways: by either giving priority to parameter $m$ or to parameter $n$. For example, if a set of arguments is $mn$-defended and another one is $st$-defended, where $m < s$ and $n < t$ (i.e., they are incomparable w.r.t. $\rhd$) then the first one can be stipulated to be more strongly defended because it is less tolerant with respect to the failure of defense. Therefore, for $m < s$ and $n < t$, belonging to $\cfmn{m}{n}(X)$ is `better' than belonging to $\cfmn{s}{t}(X)$. One could then redefine $\rhd$ as follows: $\cfmn{m}{n} \rhd \cfmn{s}{t}$ iff either $m < s$, or $m = s$ and $n \geq t$. This yields a lexicographic order over graded defense functions giving priority to the $m$ parameter over the $n$ parameter. We do not investigate such refinements further in this paper.


\section{Graded Semantics for Abstract Argumentation} \label{Sec:GradedSemantics}

By means of the graded defense and neutrality functions, Dung's notions of acceptability and conflict-freeness can be generalised to graded variants in a natural way. A set of arguments $X$ is said to be conflict-free at grade $\ell$ (or, $\ell$-conflict-free) whenever none of its arguments is attacked by at least $\ell$ arguments in $X$.
A set of arguments $X$ is said to be acceptable at grade $mn$ (or, $mn$-acceptable)
whenever all of its arguments are such that at most $m-1$ of their attackers are not counter-attacked by at least $n$ arguments in $X$.
A graded notion of admissibility follows ($\ell$-conflict-freeness plus $mn$-acceptability) and we thereby obtain graded variants of all the main admissibility-based semantics, which are simply Dung's standard semantics based on graded admissibility instead  of  standard admissibility. The first part of this section formally defines graded semantics. The rest of the section then develops a core theory of graded semantics. In the tradition of abstract argumentation, our results focus on the central questions of the existence and construction of graded extensions, and provides positive results under certain constraints on the parameters $n$, $m$ and $l$.

\subsection{Graded Generalisation of Dung's Semantics}

\begin{table}[t]
\hspace*{-1cm}
\begin{tabular}{lcl}
\hline
$X$ is $\ell$-conflict-free in $\Frame$ & iff & $X \subseteq \cff{\ell}(X)$ \\
$X$ is $mn$-self-defended & iff & $X \subseteq \cfmn{m}{n}(X)$ \\
$X$ is $\ell mn$-admissible in $\Delta$ & iff & $X \subseteq \cff{\ell}(X)$ and $X \subseteq \cfmn{m}{n}(X)$ \\
$X$ is an $\ell mn$-complete extension of $\Delta$ & iff & $X \subseteq \cff{\ell}(X)$ and $X = \cfmn{m}{n}(X)$ \\
$X$ is an $\ell mn$-stable extension of $\Delta$  & iff & $X = n_n(X) = n_m(X) \subseteq n_l(X)$ \\
$X$ is the $\ell mn$-grounded extension of $\Delta$  & iff & $X$ is the smallest $\ell mn$-complete ext. of $\Delta$ \\
$X$ is an $\ell mn$-preferred extension of $\Delta$ & iff & $X$ is a largest $\ell mn$-complete ext. of $\Delta$ \\
\hline
\end{tabular}
\caption{Graded generalizations of standard argumentation theory notions from \cite{dung95acceptability}.}
\label{table:graded}
\end{table}

We are now in the position to generalise Definition \ref{def:DungSemantics} as follows:
\begin{definition}[Graded Extensions]\label{table:accrual-sensitive}
Let $\Delta = \tuple{\A, \ar}$ be an $AF$, $X \subseteq \A$, and $\ell$, $m$ and $n$ be positive integers. Graded extensions are defined as in Table \ref{table:graded}.
We may write $\adm_{lmn}(\Delta)$, $\prf_{lmn}(\Delta)$ and $\stb_{lmn}(\Delta)$ to denote, respectively, the set of $lmn$-admissible, $lmn$-preferred. and $lmn$-stable extensions of $\Delta$, and $\grn_{lmn}(\Delta)$ to denote the $lmn$-grounded extension of $\Delta$. Finally, for an extension type $S \in \set{\mathit{grounded}, \mathit{preferred}, \mathit{stable}}$, we say that $a \in \A$ is {\em credulously} justified w.r.t. $\ell mn$-S if $a \in \bigcup S_{\ell mn}(\Delta)$; and {\em sceptically} justified w.r.t. $\ell mn$-S if $a \in \bigcap S_{\ell mn}(\Delta)$
Henceforth we assume the sceptical definition when referring to an argument simply as being {\em justified}.
\end{definition}

The definition deserves some comment. Note first of all that when $l = m = n =1$, we recover the standard definition of conflict-freeness, admissibility and extensions (Definition \ref{def:DungSemantics}), which we henceforth refer to as `Dung conflict-freeness' and `Dung admissibility' and `Dung extensions'. The key notion is graded admissibility, which is obtained by parameterizing the conflict-freeness requirement by $\ell$ --- i.e., $X \subseteq \cff{\ell}(X)$ ---, and parameterizing the self-defense requirement by $m$ and $n$ --- i.e., $X \subseteq \cfmn{m}{n}(X)$. The remaining graded semantics are defined by extending graded admissability in exactly the same way in which Dung admissibility is extended to define the standard Dung semantics. So, a graded complete extension, with parameters $\ell, m$ and $n$, is a fixpoint of $\cfmn{m}{n}$, which is also $\ell$-conflict-free, the graded grounded extension is the smallest $\ell mn$-complete extension, and the graded preferred extensions are the largest $\ell mn$-complete extensions. Finally, a graded stable extension, with parameters $\ell, m$ and $n$, is a fixpoint of $n_n(X)$ and $n_m(X)$ (and therefore of $\cfmn{m}{n}$), which is also $\ell$-conflict-free. Constructive existence results for these semantics are provided in the next section.

Each graded extension type should then be interpreted as a class of weakenings and strengthenings of its standard Dung counterpart. For example: Dung complete extensions are strengthened by $11n$-complete extensions, with $n > 1$, which require a higher number of defenders for each attacked argument (that is, the requirements for acceptability are strengthened); and are weakened by $\ell 11$-complete extensions, with $\ell > 1$, which tolerate a higher level of internal conflict (that is, weakening the conflict-freeness requirement), or by $1 m1$-complete extensions, with $m > 1$, which tolerate a higher level of undefended arguments (that is,  weakening the acceptability requirement). So for each Dung extension type, we now have an ordered family of extensions incorporating a form of graduality.


\subsection{Fixpoint Construction for Graded Exensions}

We proceed as in the standard case (cf. Section \ref{Sec:Background}). The basic idea is as follows: given a graded admissible set, we show that, and under what assumptions on the parameters $\ell$, $m$ and $n$, this can be expanded into a graded complete set through a process of fixpoint approximation.

Fix a framework $\Frame$ and take a set $X$ such that $X \subseteq \cfmn{m}{n}(X)$ and $X \subseteq \cff{\ell}(X)$ (that is, an $\ell mn$-admissible set). By iterating $\cfmn{m}{n}$, consider the stream of sets
\begin{align}
\cfmn{m}{n}^0(X), \cfmn{m}{n}^1(X), \ldots  \label{stream1}
\end{align}
and the stream of sets
\begin{align}
\cff{n}(\cfmn{m}{n}^0(X)), \cff{n}(\cfmn{m}{n}^1(X)), \ldots \label{stream1.5}
\end{align}
which, by Fact \ref{fact:properties_dn}, is equivalent to the stream
\begin{align}
\cfmn{n}{m}^0(\cff{n}(X)), \cfmn{n}{m}^1(\cff{n}(X)), \ldots. \label{stream2}
\end{align}
By the above assumptions on $X$, and since $\cfmn{m}{n}$
is monotonic and $\cff{n}$ is antitonic (cf. Fact \ref{fact:properties_dn}), the stream in \eqref{stream1} is non-decreasing and the stream in \eqref{stream2} is non-increasing, with respect to set inclusion. In finite attack graphs, these streams must therefore stabilize reaching a limit at iteration $|\A| + 1$. In infinite but finitary attack graphs, we will see that the limit can be reached at $\omega$. We will also see (Lemma \ref{theorem:approxim81}) that the limits of these streams correspond to the {\em smallest} fixpoint of $\cfmn{m}{n}$ {\em containing} $X$ and, respectively, the {\em largest} fixpoint of $\cfmn{n}{m}$ which is {\em contained} in $\cff{n}(X)$.\footnote{Notice the reversal in the parameters $m$ and $n$ due to the fact that $\cff{n}(\cfmn{m}{n}(X)) = \cfmn{n}{m}(\cff{n}(X))$, a direct consequence of Fact \ref{fact:properties_dn}.} We denote the first one by $\lfp_X.\cfmn{m}{n}$ and the second one by $\gfp_X.\cfmn{n}{m}$. Intuitively, the two sets denote the smallest superset of $X$ which is equal to the set of arguments it $mn$-defends
and, respectively, the largest set whose arguments are not attacked by at least $n$ arguments in $X$ and which is equal to the set of arguments it $nm$-defends.
The above construction is illustrated in Figure \ref{figure:decomposition} below. We prove now its correctness showing that the limits of the streams in \eqref{stream1} and \eqref{stream2} correspond indeed to the desired fixpoints.

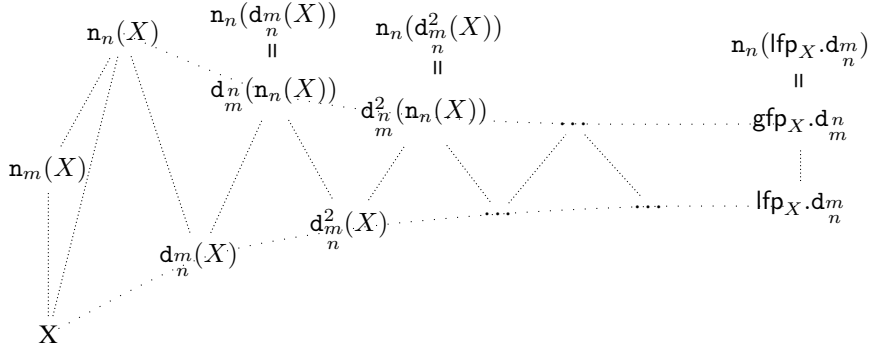
\begin{figure}[t]
\centering
\begin{tikzpicture}
\draw
(0,0) node (0) {X}
(0,2.2) node (12) {$\cff{m}(X)$}
(1,4) node (1) {$\cff{n}(X)$}

(2,1) node (2) {$\cfmn{m}{n}(X)$}
(3,3.2) node (3) {$\cfmn{n}{m}(\cff{n}(X))$}
(3,4) node (3a)
{$
\begin{array}{c}
\cff{n}(\cfmn{m}{n}(X)) \\
\verteq
\end{array}
$}

(4,1.4) node (4) {$\cfmn{m}{n}^2(X)$}
(5,2.9) node (5) {$\cfmn{n}{m}^2(\cff{n}(X))$}
(5.2,3.8) node (5a)
{$
\begin{array}{c}
\cff{n}(\cfmn{m}{n}^2(X)) \\
\verteq
\end{array}
$}

(6,1.6) node (6) {\ldots}
(7,2.8) node (7) {\ldots}
(10,2.8) node (10) {$\gfp_X.\cfmn{n}{m}$}
(10,3.6) node (10a){$
\begin{array}{c}
\cff{n}(\lfp_X.\cfmn{m}{n}) \\
\verteq
\end{array}
$}

(8,1.7) node (8) {\ldots}
(10,1.7) node (11) {$\lfp_X.\cfmn{m}{n}$};

\draw[densely dotted] (0) -- (1);
\draw[densely dotted] (0) -- (12);
\draw[densely dotted] (12) -- (1);
\draw[densely dotted] (1) -- (2);
\draw[densely dotted] (2) -- (3);
\draw[densely dotted] (3) -- (4);
\draw[densely dotted] (4) -- (5);
\draw[densely dotted] (5) -- (6);
\draw[densely dotted] (6) -- (7);
\draw[densely dotted] (7) -- (8);
\draw[densely dotted] (11) -- (10);

\draw[thin,loosely dotted] plot[smooth] coordinates {(0,0) (2,1) (4,1.4) (6,1.6) (8,1.7) (10,1.7)};
\draw[thin,loosely dotted] plot[smooth] coordinates {(1,4) (3,3.2) (5,2.9) (7,2.8) (10,2.8)};
\end{tikzpicture}
\caption{
The two streams of \eqref{stream1} and \eqref{stream2}
under the assumption that $X$ is $mmn$-admissible and the two integers $n$ and $m$ are such that $n \geq m$.
Position with respect to  the horizontal axis indicates the number of iterations, and positions with respect to the vertical axis indicates set theoretic inclusion. Cf. Figure \ref{figure:decomposition1} depicts the same behaviour for the standard, non-graded, case.}
\label{figure:decomposition}
\end{figure}
\smallskip

First of all the following important lemma shows under what conditions on $m$ and $n$ graded conflict-freeness can be preserved by the above process of iteration of the defense function.\footnote{In the following two lemmas we handle only parameters $m$ and $n$ directly. This, we will see, is sufficient to establish results concerning also parameter $\ell$ later (Theorem \ref{fact:smallest_complete}).}
\begin{lemma} \label{lemma:preserve_cf1}
Let $\Frame$ be a finitary attack graph, $X \subseteq \A$ be such that $X \subseteq \cfmn{m}{n}(X)$ and $X \subseteq \cff{m}(X)$, and $m,n$ be two positive integers such that $n \geq m$. Then for any $k$ s.t. $0 \leq k < \omega$,
$$
X \subseteq \cfmn{m}{n}^k(X) \subseteq \cff{m}(\cfmn{m}{n}^k(X)) \subseteq \cff{m}(X).
$$
\end{lemma}
\begin{proof}
\fbox{$X \subseteq \cfmn{m}{n}^k(X)$} holds by the assumption that $X$ is admissible and by the monotonicity of $\cfmn{m}{n}$ (Fact \ref{fact:properties_dn}).
\fbox{$\cfmn{m}{n}^k(X) \subseteq \cff{m}(\cfmn{m}{n}^k(X))$}
is proven by induction over $k$.
The \fbox{base case} holds by assumption as $\cfmn{m}{n}^0(X) = X$ is $mmn$-admissible and therefore $m$-conflict-free.
For the \fbox{induction step} assume (IH) that $\cfmn{m}{n}^k(X) \subseteq \cff{m}(\cfmn{m}{n}^k(X))$ (that is, the $k^{\mathit{th}}$ step is $m$-conflict-free).
We show that $\cfmn{m}{n}^{k+1}(X) \subseteq \cff{m}(\cfmn{m}{n}^{k+1}(X))$. Suppose towards a contradiction that is not the case. Then there exists $x$ and $y_1, \ldots, y_m$ in set $\cfmn{m}{n}^{k+1}(X) = \cfmn{m}{n}(\cfmn{m}{n}^{k}(X))$ such that $x \al y_i$ for $1 \leq i \leq m$.
Since $n \geq m$ by assumption, by the definition of $\cfmn{m}{n}$ there exist $z_1, \ldots, z_n \in \cfmn{m}{n}(X)$ such that $y_\ell \al z_i$ for $1 \leq i \leq n$. That is, at least one among all $y_i$'s (w.l.o.g. assumed to be $y_m$) is attacked by at least $n$ arguments in $\cfmn{m}{n}(X)$. But from this and the assumption that $n \geq m$, there exist $w_1, \ldots, w_n  \in \cfmn{m}{n}(X)$ such that $z_n \al w_i$ for $1 \leq i \leq n$. That is, at least one among all $z_i$'s (w.l.o.g. assumed to be $z_n$) is also attacked by at least $n$ arguments in $\cfmn{m}{n}(X)$. Now since $n \geq m$ we conclude that $\cfmn{m}{n}(X)$ is not $m$-conflict-free, against IH.
\fbox{$\cff{m}(\cfmn{m}{n}^k(X)) \subseteq \cff{m}(X)$} follows from the first claim by the antitonicity of $\cff{m}$ (Fact \ref{fact:properties_dn}).
\end{proof}
That is, each stage $\cfmn{m}{n}^k(X)$ in the stream of iteration of the graded defense function from set $X$ is an $m$-conflict-free set.

\begin{lemma} \label{theorem:approxim81}
Let $\Frame$ be a finitary attack graph and $X \subseteq \A$ be such that $X \subseteq \cfmn{m}{n}(X)$ and $X \subseteq \cff{m}(X)$, and $m,n$ be two positive integers such that $n \geq m$:
\begin{align}
\lfp_X.\cfmn{m}{n} & = \bigcup_{0 \leq k < \omega} \cfmn{m}{n}^k(X) \label{eq:below1} \\
\gfp_X.\cfmn{n}{m}& = \bigcap_{0 \leq k < \omega} \cfmn{n}{m}^k(\cff{n}(X)) \label{eq:above1}
\end{align}
\end{lemma}

\begin{proof}

\fbox{\eqref{eq:below1}}
\fbox{First}, we prove that $\bigcup_{0 \leq n < \omega} \cfmn{m}{n}^k(X)$ is a fixpoint by the following series of equations:
\begin{align*}
\cfmn{m}{n} \left( \bigcup_{0 \leq k < \omega} \cfmn{m}{n}^k(X) \right) & =  \bigcup_{0 \leq k < \omega}\cfmn{m}{n}(\cfmn{m}{n}^k(X)) \\
                                                                                                                                                & =  \bigcup_{0 \leq k < \omega} \cfmn{m}{n}^k(X)
\end{align*}
where the first equation holds by the $\omega$-continuity of $\cfmn{m}{n}$ (Fact \ref{fact:graded_continuous}) and the second by the fact that, since $X$ is $\ell mn$-admissible by assumption and $\cfmn{m}{n}$ is monotonic (Fact \ref{fact:properties_dn}) we have that: $X = \cfmn{m}{n}^0(X) \subseteq \cfmn{m}{n}^k(X) \subseteq \cfmn{m}{n}^{k +1}(X)$ for any $k$ s.t. $0 \leq k < \omega$.
\fbox{Second}, we prove that $\bigcup_{0 \leq k < \omega} \cfmn{m}{n}^k(X)$ is indeed the least fixpoint containing $X$. Suppose, towards a contradiction that there exists $Y$ s.t.: $X \subset Y = \cfmn{m}{n}^k(Y) \subset \bigcup_{0 \leq k < \omega} \cfmn{m}{n}^k(X)$. It follows that $X \subset Y = \cfmn{m}{n}^k(Y) \subset \cfmn{m}{n}^k(X)$ for some $k$ s.t. $0 \leq k < \omega$. But, by the monotonicity of $\cfmn{m}{n}$ (Fact \ref{fact:properties_dn}), we have that $\cfmn{m}{n}^k(X) \subseteq \cfmn{m}{n}^k(Y)$. Contradiction.

\noindent \fbox{\eqref{eq:above1}} The proof is similar to the previous case.
\fbox{First}, we prove that  $\bigcap_{0 \leq k < \omega} \cfmn{n}{m}^k(\cff{n}(X))$ is a fixpoint, through the series of equations
\begin{align*}
\cfmn{n}{m} \left( \bigcap_{0 \leq k < \omega} \cfmn{n}{m}^k(\cff{n}(X)) \right) & =  \bigcap_{0 \leq k < \omega}\cfmn{n}{m}(\cfmn{n}{m}^k(\cff{n}(X))) \\
                                                                                                                                                & =  \bigcap_{0 \leq n < \omega} \cfmn{n}{m}^k(\cff{n}(X))
\end{align*}
which hold by the $\omega$-continuity of $\cfmn{n}{m}$ (Fact \ref{fact:graded_continuous}) and by the fact that, since $X$ is $m mn$-admissible by assumption, and $\cfmn{n}{m}$ is monotonic (Fact \ref{fact:properties_dn}), $\cff{n}(X) = \cfmn{n}{m}^0(\cff{n}(X))$ $\supseteq$ $\cfmn{n}{m}^k(\cff{n}(X))$ for any $k$ s.t. $0 \leq k < \omega$.
The latter property holds because $X$ is assumed to be such that $X \subseteq \cfmn{m}{n}(X)$. By the antitonicity of $\cff{n}$ and the interdefinability of $\cfmn{m}{n}$ as $\cff{n} \circ \cff{m}$ (Fact \ref{fact:properties_dn})  we therefore have that $\cff{n}(\cfmn{m}{n}(X)) = \cfmn{n}{m}(\cff{n}(X)) \subseteq \cff{n}(X)$ from which it follows that $\cfmn{n}{m}(\cff{n}(X)) \subseteq \cff{n}(X)$. A simple induction on $k$ then establishes the claim.
\fbox{Second}, it remains to be proven that $\bigcap_{0 \leq k < \omega}  \cfmn{n}{m}^k(\cff{n}(X))$ is indeed the largest fixpoint of $ \cfmn{n}{m}$ contained in $\cff{n}(X)$.
Like in the previous case we proceed towards a contradiction. Suppose there exists $Y$ s.t.: $\bigcap_{0 \leq k < \omega} \cfmn{n}{m}^k(\cff{n}(X)) = \bigcap_{0 \leq k < \omega} \cff{n}(\cfmn{m}{n}^k(X)) \subset Y = \cfmn{n}{m}(Y) \subseteq \cff{n}(X)$. There must therefore exist an integer $k$ such that, as a consequence of Lemma \ref{lemma:preserve_cf1}, $X \subseteq \cfmn{m}{n}^k(X) \subseteq  \cff{n}(\cfmn{m}{n}^k(X)) \subset Y$. By the antitonicity of $\cff{n}$ and
again by the interdefinability of $\cf{}$ and $\cff{}$ (Fact \ref{fact:properties_dn}), and since $Y$ is taken to be a fixpoint of $\cfmn{n}{m}$, it follows that $\cff{n} (Y) = \cff{n}(\cfmn{n}{m}(Y)) = \cfmn{m}{n}(\cff{n}(Y))$. Then, from the fact that $ \cff{n}(\cfmn{m}{n}^k(X)) \subset Y$, and the fact that $\cff{m}$ is antitonic (Fact \ref{fact:properties_dn}) we have that $\cff{n} (Y) = \cff{m}(\cff{n}(\cfmn{m}{n}^k(X))) = \cfmn{m}{n}(\cfmn{n}{m}^k(X))$ which is a subset of  $\bigcup_{0 \leq k < \omega} \cfmn{m}{n}^k(X)$.
So we have that $\cff{n} (Y)$ is also a fixpoint of $\cfmn{m}{n}$, it contains $X$ and it is included in $\bigcup_{0 \leq k < \omega} \cfmn{m}{n}^k(X)$, which, by the previous claim, we know to be the smallest fixpoint of $\cf{\Frame}$ containing $X$. Contradiction.
\end{proof}
The content of Lemmas \ref{lemma:preserve_cf1} and \ref{theorem:approxim81} underpinning the construction of the two fixpoints is depicted in Figure \ref{figure:decomposition}.
Notice that like in the standard case of $\cfmn{1}{1}$ and $\cff{1}$ the two streams can be generated by the indefinite iteration of the application of $\cff{n}$ followed by $\cff{m}$.
The fact that $n \geq m$ is assumed, guarantees that each set of arguments $\cfmn{m}{n}^k(X)$ in the lower stream remains included in the set of arguments towards which it is $m$-neutral; that is, it remains $m$-conflict fee (and a fortiori $n$-conflict-free).

\begin{remark}
Like for the proof of Lemma \ref{theorem:approxim8}, finitariness plays an essential role in the proof of Lemma \ref{theorem:approxim81}. However, also in this case the assumption can be lifted and the proof could proceed using transfinite induction (cf. Remark \ref{remark:ordinal}).
\end{remark}


\subsubsection{Constructing Graded Complete and Grounded Extensions}

We now move on to show how the above results provide constructive proofs of existence of graded semantics.
We will focus on finitary graphs, but it should be clear that the non-finitary case can be handled by ordinal induction.\footnote{Cf. Remark \ref{remark:ordinal}.}

\begin{theorem} \label{fact:smallest_complete}
Let $\Delta$ be a finitary $AF$, $X \subseteq \A$ be such that $X \subseteq \cff{\ell}(X)$ and $X \subseteq \cfmn{m}{n}$, with $\ell, m$ and $n$ positive integers such that $n \geq m$ and $\ell \geq m$.
Then $\bigcup_{0 \leq k < \omega} \cfmn{m}{n}^k(X)$ is the smallest $\ell mn$-complete extension containing $X$.
\end{theorem}
\begin{proof}
By Lemma \ref{theorem:approxim81}, we know that $\bigcup_{0 \leq k < \omega} \cfmn{m}{n}^k(X) = \lfp_X.\cfmn{m}{n}$, that is, the smallest fixpoint of $\cfmn{m}{n}$ that contains $X$. By Lemma \ref{lemma:preserve_cf1} we know that this set is $m$-conflict-free, that is, $\bigcup_{0 \leq k < \omega} \cfmn{m}{n}^k(X) \subseteq \cff{m}(\bigcup_{0 \leq k < \omega} \cfmn{m}{n}^k(X))$. Since $\ell \geq m$ by assumption, by Fact \ref{fact:accrual_relations} and the transitivity of $\subseteq$ we obtain that $\bigcup_{0 \leq k < \omega} \cfmn{m}{n}^k(X) \subseteq \cff{\ell}(\bigcup_{0 \leq k < \omega} \cfmn{m}{n}^k(X))$. It therefore follows that $\bigcup_{0 \leq k < \omega} \cfmn{m}{n}^k(X)$ is a fixpoint of $\cfmn{m}{n}$, it is $\ell$-conflict-free and it is the smallest such set containing $X$. As claimed, it is therefore the smallest $\ell mn$-complete extension (Definition \ref{table:accrual-sensitive}) containing $X$.
\end{proof}

\begin{corollary} \label{fact:grounded}
Let $\Delta$ be a finitary $AF$,  $n \geq m$ and $\ell \geq m$. Then $\bigcup_{0 \leq k < \omega} \cfmn{m}{n}^k(\emptyset)$ is the $\ell mn$-grounded extension of $\Delta$.
\end{corollary}

So given an $\ell mn$ admissible set ($n \geq m$ and $\ell \geq m$), the smallest $\ell mn$-complete extension containing $X$ is the least fixpoint of $\cfmn{m}{n}$ which contains $X$ ($\lfp_X.\cfmn{m}{n}$), and the $\ell mn$-grounded extension is simply the smallest fixpoint of $\cfmn{m}{n}$ ($\lfp.\cfmn{m}{n}$).


\subsubsection{Constructing Graded Preferred Extensions}

Theorem \ref{fact:smallest_complete} implies that if we choose a `large enough' $\ell mn$-admissible set $X$, in the sense that such a set can reach through a chain of attacks any other argument,
then the indefinite iteration of $\cfmn{m}{n}$ will yield an $\ell mn$-preferred extension (when $n \geq m$ and $\ell \geq m$).


\begin{fact} \label{fact:prf}
Let $\Delta$ be a finitary $AF$,  $X \subseteq \A$ be such that $X \subseteq \cff{\ell}(X)$ and $X \subseteq \cfmn{m}{n}(X)$, and $ n \geq m$ and $\ell  \geq m$. Assume furthermore that $\A \subseteq \set{a \mid \exists b \in X: b \ar^+ a}$. 
Then $\bigcup_{0 \leq k < \omega} \cfmn{m}{n}^k(X)$ is the $\ell mn$-preferred extension of $\Delta$ that contains $X$.
\end{fact}
\begin{proof}
By Theorem \ref{fact:smallest_complete} we know that $\bigcup_{0 \leq k < \omega} \cfmn{m}{n}^k(X)$ is the smallest $\ell mn$-complete extension containing $X$. It needs to be shown that there exists no $\ell mn$-admissible set $Y$ such that $X \subset Y$. Suppose, towards a contradiction, that this is the case. Then $X \subseteq \bigcup_{0 \leq k < \omega} \cfmn{m}{n}^k(X) \subset Y \subseteq \cfmn{m}{n}(Y) \subseteq \set{a \mid \exists b \in X: a \al^+ b}$.  It follows that there exists $x \in Y \subseteq \cfmn{m}{n}(Y)$ such that $x \not\in \cfmn{m}{n}^k(X)$ for some positive integer $k$. By assumption, there exists a finite path of attacks from $X$ to $x$. So let $X_1$ be the smallest set of arguments that $mn$-defends $x$, $X_2$ the smallest set of arguments that $mn$-defends the arguments in $X_1$ and so on. It follows there exists a $j$ such that $X_j \subseteq X$, otherwise $x$ would not be $mn$-defended. It therefore follows that $x \in \bigcup_{0 \leq k < \omega} \cfmn{m}{n}^k(X)$. Contradiction.
\end{proof}


\subsubsection{Constructing Graded Stable Extensions}

We finally arrive at a characterization of graded stable extensions as limits of streams generated by the graded neutrality function:
\begin{fact} \label{theorem:approximate_stable}
Let $\Delta$ be a finitary $AF$,  $X \subseteq \A$ be such that $X \subseteq \cff{\ell}(X)$ and $X \subseteq \cfmn{m}{n}(X)$, and $n \geq m$ and $\ell \geq m$. Then, $\bigcup_{0 \leq k < \omega} \cfmn{m}{n}^k(X)$ is the smallest $\ell mn$-stable extension containing $X$ if and only if $\bigcup_{0 \leq k < \omega} \cfmn{m}{n}^k(X) = \bigcap_{0 \leq k < \omega} \cff{n}(\cfmn{m}{n}^k(X))$.
\end{fact}
\begin{proof}
\rightleft  Assume $\bigcup_{0 \leq k < \omega} \cfmn{m}{n}^k(X) = \bigcap_{0 \leq k < \omega} \cff{n}(\cfmn{m}{n}^k(X))$. It therefore follows that $\bigcup_{0 \leq k < \omega} \cfmn{m}{n}^k(X) = \cff{n}(\bigcup_{0 \leq k < \omega} \cfmn{m}{n}^k(X))$. We conclude that  $\bigcup_{0 \leq k < \omega} \cfmn{m}{n}^k(X)$ is a fixpoint of $\cff{n}$, that is by Definition \ref{table:accrual-sensitive}, an $\ell mn$-stable extension.
 \leftright Straightforward.
\end{proof}
So, as in the standard case, graded stable extensions are the results of the convergence of the upper and lower streams of iteration of the graded defense function (cf. Figure \ref{figure:decomposition}). As in the standard case, such convergence is not guaranteed in general.


\subsection{On the constraints $n \geq m$ and $\ell \geq m$}


Theorem \ref{fact:smallest_complete} assumed that the three parameters of graded admissibility $\ell, m$ and $n$ are in the relation $n \geq m$ and $\ell \geq m$. Notice that Dung's standard semantics trivially meets this constraint with $\ell = m = n = 1$.
This assumption plays a crucial role in the proofs of the above results and one can in fact show that graded semantics for a choice of parameters failing the constraint may not exist in some frameworks.\footnote{This is a situation fully analogous to the case of stable extensions in the standard Dung framework, where they are not always guaranteed to exist.}

\begin{example}
Let $\Delta$ be the $AF$ consisting of the arguments and attacks in Figure \ref{Motivating1}iii). Then $\Delta$ has no $221$-grounded extension. To establish this let us try to construct such an extension with the construction of Theorem \ref{fact:smallest_complete}:
\[
\cfmn{2}{1}(\emptyset) = \set{b_3, c_3, d_3, e_3}, \cfmn{2}{1}(\set{b_3, c_3, d_3, e_3}) = \set{a_3, b_3, c_3, d_3, e_3}
\]
So the whole set of arguments $\A$ in that framework constitutes the smallest fixpoint of $\cfmn{2}{1}$. Such a set is, however, not $2$-conflict-free, that is $\A \not\subseteq \cff{2}(\A)$. In fact no $221$-complete extensions exist in this framework, as these would have to include the set of unattacked arguments $\set{d_3, e_3}$ which, it is easy to see, $21$-defends all arguments.\footnote{See Example \ref{example:inexistence2} later for another such example.}
\end{example}
Intuitively, the constraint imposes two properties on graded admissibility: first, that the level $m$ of failure of defense which we are willing to tolerate, should not exceed the number $n$ of counter-attackers we want for such a defense to hold;
second, that such a level $m$ of failure of defense, should not exceed the level $\ell$ of conflict-freeness that we are willing to tolerate. So to guarantee existence of a graded semantics tolerating arguments for which $m$ counter-attacks fail, one has to set the number $n$ of required counter-attackers higher or equal to $m$; and similarly, one has to set the level of tolerance $\ell$ to internal attacks at least as high as $m$.\footnote{Consider a $232$-admissible extension $X$ (where in violation of the constraints, $l = 2,m = 3,n=2$), and some $x \in X$ such that $x$ is attacked by two undefended attackers, in keeping with the $m$ parameter indicating that $x$ can be considered acceptable up to a cumulative threshold of two fully acceptable (i.e., two undefended) attackers of $x$. This level of tolerance with respect to the acceptability of $x$ is then at odds with the more restrictive toleration of the co-acceptability of $x$ with a maximum of one attacker $y \in X$ on $x$ (as indicated by $l=2$).
Moreover, the fact that
a third attacker $z$ of $x$ need only be counter-attacked by two arguments (as indicated by  $n=2$), means that one can still (given $m = 3$) consider $z$ to be an acceptable argument by the above reasoning applied to $x$, which in turn implies that the tolerated cumulative threshold on attackers of $x$ has been exceeded. Precisely when this sort of mismatch between levels of tolerance of acceptability and co-acceptability occurs, graded semantics may fail to exist.
}


\subsection{Basic Properties of Graded Semantics}

A direct consequence of Definition \ref{table:accrual-sensitive} is that graded extensions are in the same logical relations as standard Dung's extensions:
\begin{fact}
For any AF $\Delta = \tuple{\A, \rightarrow}$, and integers $\ell, m$ and $n$, graded semantics are related according to the following diagram:
\[
\begin{diagram}
\set{X \subseteq \A \mid X \subseteq \cfmn{m}{n}(X)} \ \ \ \  &                &                             &               &  \ \ \ \ \set{X \subseteq \A \mid X \subseteq \cff{\ell}(X} \\
                                   &  \luLine &                             &  \ruLine &         \\
\set{X \subseteq \A \mid \cfmn{m}{n}(X) \subseteq X}  \ \ \ \   &                &  \adm_{\ell mn}(\Delta) &                &          \\
                                   &  \luLine &  \uLine              &                &          \\
                                   &                &  \cmp_{\ell mn}(\Delta) &                &          \\
                                   &  \ruLine &  \uLine              &                &          \\
\set{\grn_{\ell mn}(\Delta)}    &                &  \prf_{\ell mn}(\Delta)    &                &   \ \ \ \ \set{X \subseteq \A \mid \cff{n}(X) \subseteq X}  \\
                                   &                &  \uLine              &  \ruLine &          \\
                                   &                &  \stb_{\ell mn}(\Delta)    &                &          \\
\end{diagram}
\]
where two nodes are connected when the lower node in the pair is a subset ($\subseteq$) of the upper node in the pair.
\end{fact}
So all $\ell mn$ extensions are $mn$-admissible and $\ell$-conflict-free; $\ell mn$-grounded, -stable and -preferred are all $\ell mn$-complete extensions; and $\ell mn$-stable extensions are $\ell mn$-preferred. We will further illustrate Definition \ref{table:accrual-sensitive}, and the associated constructions, with a series of examples later in Section \ref{section:examples}.

\begin{fact}
For any $AF$ $\Delta$, and integers $\ell, m$ and $n$ such that $n \geq m$ and $\ell \geq m$:
\begin{align}
\set{x \mid \overline{x} = \emptyset} & \subseteq \grn_{\ell mn}(\Delta) \label{eq:grn} \\
\bigcup S_{\ell mn}(\Delta) & \subseteq \bigcup S_{\ell' m'n'}(\Delta) \label{eq:lmn0} \\
\bigcap S_{\ell mn}(\Delta) & \subseteq \bigcap S_{\ell' m'n'}(\Delta) \label{eq:lmn}
\end{align}
for $S \in \set{\mathit{grounded}, \mathit{preferred}, \mathit{stable}}$ and any $\ell', m', n'$ such that $n' \geq m'$ and $\ell' \geq m'$ and $\ell' \geq \ell$, $m' \geq m$ and $n' \leq n$.
\end{fact}
\begin{proof}
\fbox{\eqref{eq:grn}} It is easy to see that for any $\ell, m$ and $n$ if $\overline{x} = \emptyset$ then $x \in \cfmn{m}{n}(\emptyset)$ and therefore, by Definition \ref{table:accrual-sensitive} $x \in \grn_{\ell mn}(\Delta)$.
\fbox{\eqref{eq:lmn0} \& \eqref{eq:lmn}} Both claims follow from Definition \ref{table:accrual-sensitive}  by Fact \ref{fact:accrual_relations}. The constraint on $\ell', m'$ and $n'$ is necessary to guarantee existence under such parameters.
\end{proof}


\subsection{Some Examples} \label{section:examples}

\begin{example} \label{example:inexistence2}
Referring to Figure \ref{MotivatingSem}'s $AF$ below:
\[
\cfmn{1}{2}(\emptyset) = \set{d,f,g}; \cfmn{1}{2}(\set{d,f,g}) = \set{d,f,g}.
\]
So $\set{d,f,g}$ is the smallest fixpoint of $\cfmn{1}{2}$ and it is $1$-conflict-free (as well as $\ell$-conflict-free for every $\ell \geq 1$). It is therefore also the $112$-grounded extension (as well as $\ell12$-grounded extension for every $\ell \geq 1$) of the given framework. Consider now:
\[
\begin{array}{c}
\cfmn{2}{1}(\emptyset) \\
\verteq  \\
\set{c,d,f,g}
\end{array}
\
\begin{array}{c}
\cfmn{2}{1}(\set{c,d,f,g}) \\
\verteq  \\
\set{a,b,c, d, f, g}
\end{array}
\
\begin{array}{c}
\cfmn{2}{1}(\set{a, b, c,d,f,g}) \\
\verteq  \\
\set{a,b,c, d, f, g}
\end{array}
\]
So $\set{a,b,c, d, f, g}$ is the smallest fixpoint of $\cfmn{2}{1}$. However, it is not $2$-conflict-free (argument $a$ in the set is attacked by two other arguments in the set) and therefore it is not a $221$-grounded extension. To make it a graded grounded extension one has to tolerate more internal attacks, setting for instance $\ell = 3$. The set is indeed a $321$-grounded extension.
\end{example}

\begin{example}
Consider now the central graph (3-cycle) in Figure \ref{figure:graphs}. We know that such a graph has no grounded extension and the only complete extension is $\emptyset$. But one can slightly relax the defence and conflict-free requirements to obtain a non-empty graded complete extension. Set $\ell = 2$, $m = 2$ and $n = 1$. We can construct the $221$-grounded extension of this framework as follows:
\[
\cfmn{2}{1}(\emptyset) = \set{a,b,c}, \cfmn{2}{1}(\set{a,b,c}) = \set{a,b,c}
\]
So, $\set{a,b,c}$ is the smallest fixpoint of the graded defense function $\cfmn{2}{1}$ and such a set is clearly $2$-conflict-free (that is, $\set{a,b,c} \subseteq \cffm{2}(\set{a,b,c})$). It is therefore also a $221$-preferred and -stable extension.
\end{example}



\section{Ranking Arguments by Graded Semantics} \label{Sec:Ranking}

The theory developed in the previous sections offers a novel perspective on how arguments can be compared from an (abstract) argumentation theoretic point of view. In this section we describe two natural ways in which the theory of graded acceptability can be applied to define orderings on arguments: a `contextual' way, whereby, given a fixed set of arguments, one can rank arguments by how well they are iteratively defended by the given set; and an `absolute' way, whereby arguments are compared  based on their acceptability under given graded semantics. The definitions introduced are illustrated by means of several examples in this section and later in Section \ref{Sec:Applications}.


\subsection{Contextual Approach: Ranking by Quality of Defense}

The same recursive principles underpinning the standard Dung semantics can be used to characterise how strongly the set of arguments defended by a given set, defends another set, and how the latter set defends yet another set, and so forth. That is to say, given a set $X$---the context---one can iteratively apply $\cfmn{m}{n}$ to $X$. Iterated graded defense can thus rank arguments with respect to how well a given set $X$ defends them:

\begin{definition} \label{definition:degree}
Let $\Delta = \tuple{\A, \ar}$ be a finitary $AF$  and $X \subseteq \A$. For $a,b \in \A$, we define that $a$ is `{\em at least as justified as}' $b$ w.r.t $X$ as follows:
\begin{align*}
a \succeq^X b  & \IFF  \forall m,n > 0
\IF b \in \bigcup_{0 \leq k < \omega}\cfmn{m}{n}^k(X) \THEN a \in  \bigcup_{0 \leq k < \omega} \cfmn{m}{n}^k(X)
\end{align*}
The strict part $\succ_X$ of the above relation is defined in the obvious way. 
\end{definition}
As usual, $a \approx^X b$ denotes that $a \succeq^X b$ and $b \succeq^X a$.
When we want to make the underlying $AF$ explicit we use the heavier notation $a \succeq_\Delta^X b$.

The key intuition behind this definition is the following. Take two arguments $a$ and $b$, and some fixed set $X$. Is it the case that every time $b$ is defended through the iteration of some graded defense function, $a$ also is? If that is the case, it means that (w.r.t. $X$) every standard of defense met by $b$ is also met by $a$, but $a$ may satisfy yet stronger ones (recall Fact \ref{Fact:Ordering}). It is easy to see that $\succeq^X$ is a partial order, for any $X \subseteq \A$.

\begin{example}\label{Ex:DEJ}
Let us rank, by iterated defense w.r.t $\emptyset$, the arguments in Figure \ref{Motivating1}i-iv). Applying Definition \ref{definition:degree} we obtain the partial order shown in the Hasse diagram in Figure \ref{Motivating1}vii).  Note that we assume one single $AF$ $\Delta$ consisting only of the arguments and attacks shown in Figures \ref{Motivating1}i-iv). Under the standard Dung semantics, all arguments in $\bigcup_{i=1}^4 Xi$ are in the iterated application of $\cf{\Delta}$ to $\emptyset$ (i.e., in the grounded extension of $\Delta$). However, we can now differentiate amongst these arguments. As expected the best arguments are those with no attackers. Second-best is $a2$ whose attackers are all counter-attacked by two un-attacked arguments.  Third-best is $a1$, since $a2 \in \bigcup_{0 \leq k < \omega} \cfmn{1}{2}^k(\emptyset)$ but $a1 \notin \bigcup_{0 \leq k < \omega} \cfmn{1}{2}^k(\emptyset)$ (and so $a1  \nsucceq^{\emptyset} a2$), but $\forall m,n$: if $a1 \in \bigcup_{0 \leq k < \omega} \cfmn{m}{n}^k(\emptyset)$ then $a2 \in \bigcup_{0 \leq k < \omega} \cfmn{m}{n}^k(\emptyset)$ (and so $a2  \succeq^{\emptyset} a1$). We then have that $a3$ and $a4$ are incomparable (recall Section \ref{Sec:ExtendingPartialOrder}). Formally, $a3 \in
\bigcup_{0 \leq k < \omega} \cfmn{3}{3}^k(\emptyset)$ and $a3 \notin
\bigcup_{0 \leq k < \omega} \cfmn{2}{2}^k(\emptyset)$, but $a4 \notin
\bigcup_{0 \leq k < \omega} \cfmn{3}{3}^k(\emptyset)$ and $a4 \in
\bigcup_{0 \leq k < \omega} \cfmn{2}{2}^k(\emptyset)$.
Critically, we can also differentiate amongst the rejected arguments (those not in the Dung grounded extension). Thus $b1, b3, c3, c4$, (each of which are attacked by one argument), are ranked above $b2, b4, d4$ (each of which are attacked by two arguments).
\end{example}


\subsection{Absolute Approach: Ranking by Quality of Justification}

More generally, for a given semantics, we can rank the justification status of an argument with respect to a given framework, exactly as we did in Definition \ref{definition:degree} for iterated  graded defense:


\begin{definition}[Ranking arguments by graded semantics]\label{definition:degree-semantics}
Let $\Delta = \tuple{\A, \ar}$ be an $AF$. For $a,b \in \A$, and for $S \in \{$$\ell mn$-grounded, $\ell mn$-stable, $\ell mn$-preferred$\}$:\footnote{Recall we are working with the sceptical notion of justifiability throughout the paper.}
\begin{align*}
a \succeq^S b  & \IFF   \forall \ell, m,n > 0,  \IF b \mbox{    is justified    w.r.t.    } S
                                     \THEN a \mbox{    is justified    w.r.t.    } S.
\end{align*}
The strict part $\succ_S$ of the above relation is defined in the obvious way.
\end{definition}
As usual, $a \approx^S b$ denotes that $a \succeq^S b$ and $b \succeq^S a$.
Again, it is easy to see that $\succeq^S$ is a partial order. When we want to make the underlying $AF$ explicit we use the heavier notation $a \succeq_\Delta^S b$.

\smallskip

We  illustrate how the ranking of arguments by graded semantics can be applied to arbitrate between arguments that are credulously but not sceptically justified under Dung's standard semantics.

\begin{figure}[t]
\centering
\includegraphics[width=3in]{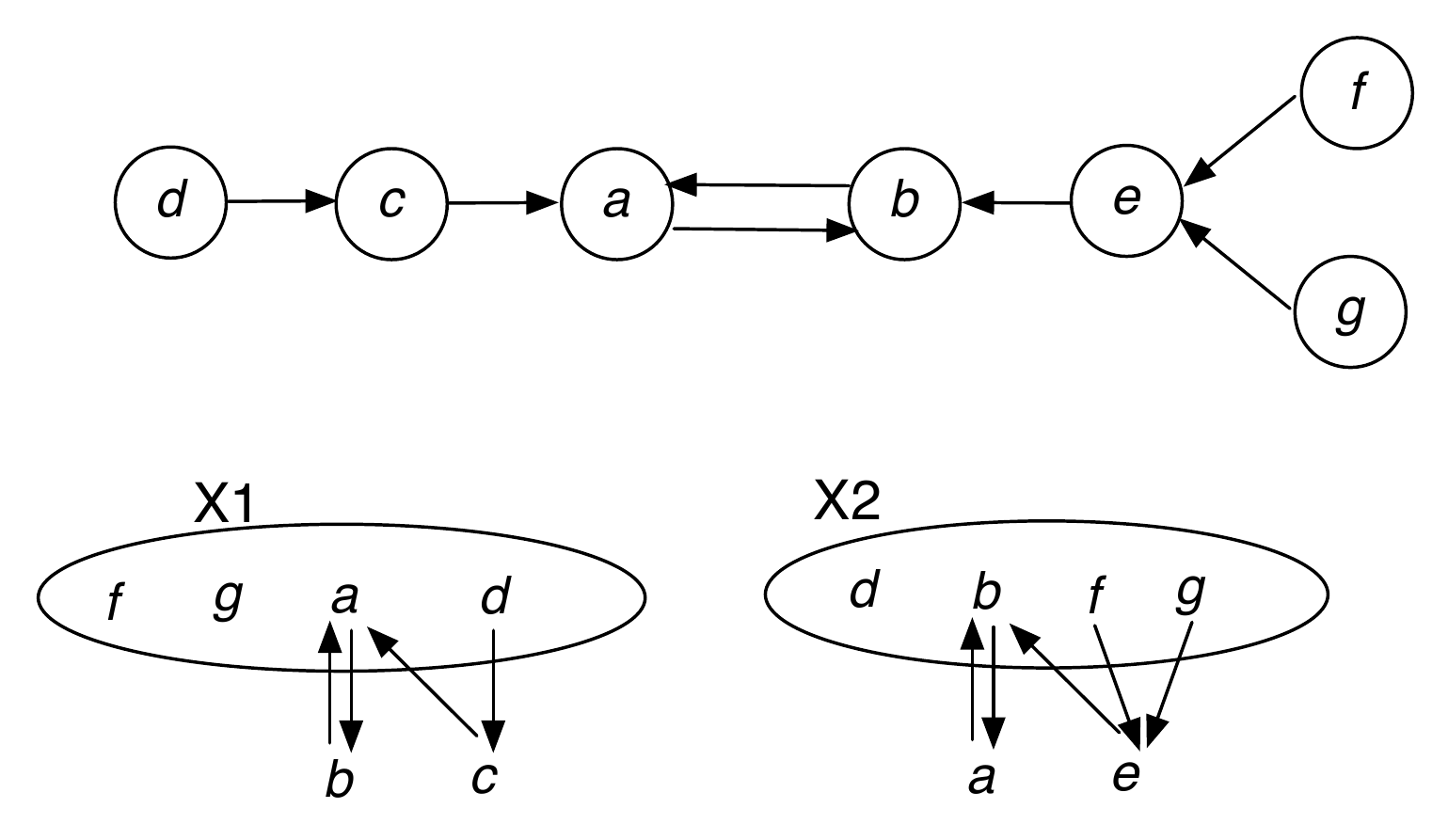}

\caption{An $AF$ with two preferred extensions \label{MotivatingSem}}

\end{figure}

Consider the $AF$ $\Delta$ in Figure \ref{MotivatingSem} that has two preferred extensions---$X1$ and $X2$---under the standard Dung semantics. This equates with $X1$ and $X2$ both being $111$-preferred extensions in the graded terminology. Hence $a$ and $b$ are credulously justified, while only $f$, $g$ and $d$ are sceptically justified. Typically, one would then rely on exogenously given preferences \cite{AmCay,Modgil2013361}
to arbitrate between such arguments. However, we can make use of the endogenously derived ranking of arguments yielded by our graded semantics to arbitrate amongst $a$ and $b$. Intuitively, $b$ is more strongly defended in $X2$ than $a$ is defended in $X1$, and the graded semantics can make use of this information in the standard framework to arbitrate in  favour of
$b$ over $a$.
Specifically,  only $X2$ is a $222$-admissible (and hence a subset of a $222$-preferred\footnote{The  $222$-preferred extension is $X2 \cup \{a\}$.}) extension since only one attacker ($a$) of $b$ is defended by strictly less than two arguments, whereas $X1$ is not a $222$-admissible  ($X1$ is not a subset of a $222$-preferred) extension since both attackers of $a$ ($b$, $c$) are defended by strictly less than two arguments. 
Indeed,  $\forall l, m,n \geq 0$:
$ \IF a$   is justified    w.r.t.   $lmn$-preferred $
                                     \THEN b$    is justified    w.r.t.     $lmn$-preferred,
                                     whereas \emph{it is not the case} (as illustrated above) that $\forall l, m,n \geq 0$:
$ \IF b$   is justified    w.r.t.   $lmn$-preferred $
                                     \THEN a$    is justified    w.r.t.   $lmn$-preferred. Hence $b \succ_\Delta^{lmn\text{-preferred}} a$, so arbitrating in favour of $b$ over $a$.

It is worth noting that the contextual and absolute approaches to argument ranking are related as follows.

\begin{fact}
Let $\Delta$ be a finitary $AF$,  $n \geq m$ and $\ell \geq m$. Then:
\begin{align*}
a \succeq^\emptyset b  & \IFF   a \succeq^{lmn\text{-grounded}} b
\end{align*}
\end{fact}
\begin{proof}
The claim is a direct consequence of Corollary \ref{fact:grounded} and Definitions \ref{definition:degree} and \ref{definition:degree-semantics}.
\end{proof}

Further properties of rankings based on graded semantics will be discussed in detail later (Section \ref{Sec:RelatedWork}), in the context of other existing approaches to argument rankings in argumentation.


\section{Instantiating Graded Semantics}\label{Sec:Applications}

We have thus far focussed on graded defense, acceptability and semantics for abstract argumentation frameworks. We now illustrate these notions by reference to
instantiated frameworks. In particular,  \textit{ASPIC+}  \cite{Modgil2013361, hp10aspicJAC}  provides a general framework for specifying logical instantiations of  Dung frameworks. It has been shown to  formalise human orientated accounts of  argumentation based reasoning that make use of Schemes and Critical Questions \cite{Wal96}, and
provide a dialectical characterisation of non-monotonic inference in Brewka's Preferred Subtheories \cite{bre89} and Prioritised Default Logic \cite{PDL} (in \cite{Modgil2013361} and, respectively, \cite{Young2016}). In what follows we show how our graded semantics can be applied to both these instantiations.


 \subsection{Graded Semantics for Instantiations based on Schemes and Critical Questions}\label{SecGradedSchCQ}

We use a well-established instantiated argumentation setting to illustrate the usefulness of the theory of graded acceptability for evaluating argument strength.

\subsubsection{Schemes and Critical Questions}

 Schemes and Critical Questions (\emph{SchCQ}) have been developed by the informal logic community, most notably by Walton \cite{Wal96}, and capture stereotypical patterns of argument as deployed in epistemic and practical reasoning.
For example, the \emph{argument from expert opinion} scheme states that if $E$ is an expert in domain $D$, and $E$ states that $S$ is true (false), and $S$ is within domain $D$, then (presumably) $S$ is true (false).

Echoing our motivation for graded defense and acceptability in Section \ref{Sec:GradedAcceptabilityIntuitions}, Walton emphasises that
it is not feasible for reasoning agents to establish beyond doubt the validity of a scheme's \emph{presumptions} if one is to effectively engage in epistemic or practical reasoning. The presumptive nature of the grounds used to establish a claim means that any given argument instantiating a scheme, does not in and of itself provide grounds for acceptance of the claim as having the status of true justified belief, or being the decision option that indisputably maximises a given objective. Rather, one
establishes  confidence in the claim sufficient to reason or act on the basis of the claim  \cite{Walton2013}. Moreover, each scheme is associated with critical questions that  render explicit the presumptive nature of an argument's grounds. For example,
`Is $E$ an expert in domain $D$ ?' and `Is $E$ reliable as as  source ?'.  A natural way to formalise reasoning with argument schemes is to regard them as defeasible inference rules and to regard critical questions as pointers to counter-arguments that may themselves instantiate schemes \cite{m+p14, Verheij2003}. \textit{ASPIC+} arguments are built from such defeasible
rules defined over a first order language, and premises that are first order formulae. Arguments can be \emph{rebut} attacked on the conclusions of an argument's defeasible rules, \emph{undermine} attacked on the argument's premises, or \emph{undercut} attacked by challenging the applicability of a defeasible rule (via a naming mechanism for rules such that the attacking argument concludes the negation of the name of the rule in the attacked argument). For simplicity of presentation we will illustrate using rule based formulations of schemes defined over a propositional language.

Contravening the notational conventions used thus far, but in line with the standard notation in the instantiations literature, we will in this section use upper case Roman letters $A,B,C,\ldots$ to denote arguments, lower case roman letters $a, b, c, \ldots$ for propositions, Greek letters $\alpha, \beta, \gamma, \ldots$ for variables in schemes, and notation $[a], [b], [c], \ldots$ to denote subarguments consisting of a single proposition. We will use the following (abbreviated) schemes and (selected) critical questions, and instantiations of these schemes represented as propositional defeasible inference rules (whose names are shown as subscripts on the defeasible inference symbol $\Rightarrow$):
\smallskip

\paragraph{Argument from position to know} ($APK$):
\begin{itemize}
\item Source $\alpha$ is in a position to know about proposition $\gamma$;
\item Source $\alpha$ asserts that $\gamma$ is true (false);
\item Therefore presumably $\gamma$ is true (false).
\end{itemize}
\noindent $APK$ critical questions:  (APK1) Is $\alpha$ in a position to know about  proposition $\gamma$?; (APK2) Is $\alpha$ an honest/trustworthy/reliable source?; (APK3) Did $\alpha$ assert that $\gamma$ is true (false)?

\begin{figure}[t]
\centering
\includegraphics[width=3.6in]{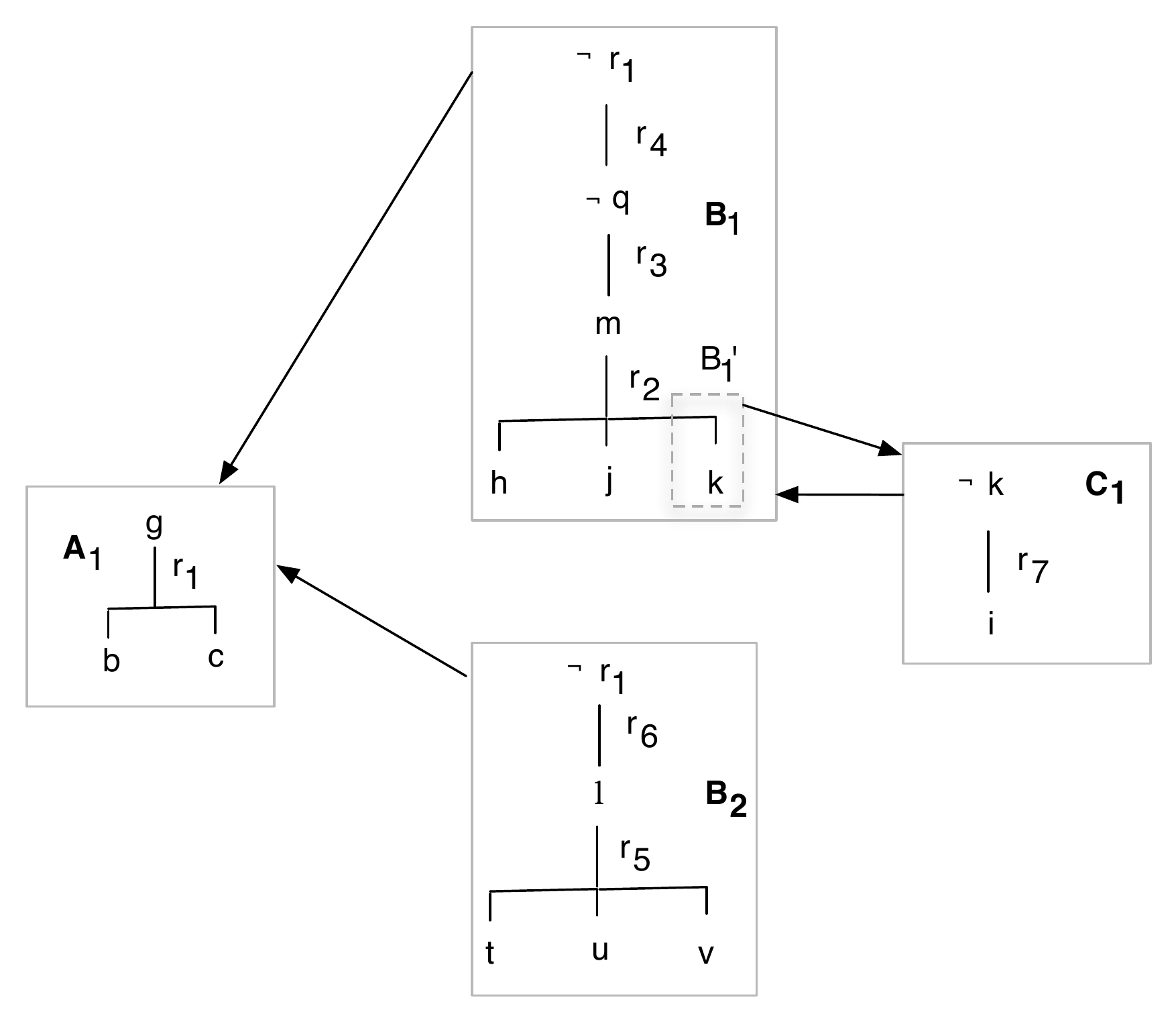}

\caption{\textit{ASPIC+} arguments are upside down trees whose roots are the arguments' claims and whose leaves are the arguments' premises. Notice that $C_1$ asymmetrically undermine attacks $B_1$ on its sub-argument $B_1'$. $C_1$ and $B_1'$ also symmetrically attack ($C_1$ undermine attacks $B_1'$ and $B_1'$ rebut attacks $C_1$). 
\label{Arguments}}

\end{figure}

\paragraph{Argument from expert opinion} ($AEO$):
\begin{itemize}
\item $\alpha$ is an expert in domain $\beta$;
\item The domain $\beta$ contains proposition $\gamma$;
\item $\alpha$ asserts that $\gamma$ is true (false);
\item Therefore presumably $\gamma$ is true (false).
\end{itemize}
\noindent $AEO$ critical questions:  (AEO1) how credible an expert  is $\alpha$?;
(AEO2) Is $\alpha$ an expert in domain $\beta$?;
(AEO3) Is $\alpha$ reliable as a source?;
(AEO4) Did $\alpha$ assert that  $\gamma$ is true (false)?;
(AEO5) Is $\gamma$ consistent with what other experts assert?


\subsubsection{A Detailed Example}

\begin{example}
Assume the premises $b$ = `Blair is in a position to know about whether removing Assad will achieve democracy in Syria'\footnote{Ex UK prime minister Tony Blair was appointed middle east envoy for peace in 2007.} and  $c$ = `Blair asserts that removing Assad will achieve democracy in Syria is true', and the rule $b, c \Rightarrow_{r_1} g$ where $g$ denotes `removing Assad will achieve democracy in Syria'. We then have the argument $A_1$ in Figure \ref{Arguments}.

Assume furthermore the premises $h$ = `Chilcot is an expert in the domain of Blair's conduct in the Iraq war', $j$ = 'the domain Blair's conduct in the Iraq war contains proposition Blair lied about weapons of mass destruction (WMD)', $k$ = `Chilcot asserts that Blair lied about WMD', and the rules $h,j, k \Rightarrow_{r_2} m$ (where $m$ = `Blair lied about WMD'), $m \Rightarrow_{r_3} \neg q$ (where $q$ = `Blair is an honest source'), and $\neg q \Rightarrow_{r_4} \neg r_1$. We then have the argument $B_1$ in Figure \ref{Arguments}, where
 $B_1$ addresses critical question APK2 and  undercuts $A_1$ on the inference $r_1$.
\end{example}
Intuitively, $B_1$'s claim does not have the status of an incontrovertible belief, since its grounds do not have such status.\footnote{Commentators  varyingly interpreted the Chilcot report's investigation into the Iraq war (\url{www.iraqinquiry.org.uk/the-report}) as asserting that Blair did/did not lie.} Hence, one may retain some residual confidence in $A_1$'s claim. Graded semantics allow for this level of granularity. So, in the $AF$ containing $A_1$ and $B_1$, $A_1$ and $B_1$ are in the $222$-grounded extension.\footnote{All sub-arguments of $A_1$ and $B_1$, i.e., $[b]$, $[c]$, $[h]$, $[j]$, $[k]$, and the sub-arguments of $B_1$ that conclude $m$ and $\neg q$, are also in the $222$-grounded extension (none of these additional arguments attack or are attacked by any other arguments).}

Clearly however, both our contextual (Definition \ref{definition:degree}, the natural context being $\emptyset$) and absolute (Definition \ref{definition:degree-semantics}) approaches for ranking arguments would rank $B_1$ above $A_1$.

\begin{example}\label{ExTwoVOneAttack}
Suppose an additional instance of the $AEO$ scheme constructed from the premises $t$ = `The Arab League is an expert in the domain of Blair's knowledge of middle east affairs' ; $u$ = 'the domain Blair's knowledge of middle east affairs contains the proposition Blair is an unreliable source' ; $v$ = `The Arab League assert that
Blair is an unreliable source', and the rules $t, u, v \Rightarrow_{r_5} l$ (where $l$ = `Blair is an unreliable source'), $l \Rightarrow_{r_6} \neg r_1$.
 We then have argument $B_2$ in Figure \ref{Arguments},
 where $B_2$ addresses critical question APK2 (by claiming unreliability rather than dishonesty) and undercuts $A_1$ on the inference $r_1$. We now have two unattacked arguments attacking $A_1$. $A_1$ is now further weakened, and neither $\{A_1,B_1\}$ nor  $\{A_1,B_2\}$ are included in $222$-preferred extensions (the threshold of allowing one unattacked attacker on $A_1$ and thus maintaining co-acceptability of $A_1$ with one such attacker, has been exceeded). Abstractly, this equates with the ranking $a1  \succ^S a1^*$   (according to Definition \ref{definition:degree-semantics}) in the AF shown in Figure \ref{SchCQAFs}i.
\end{example}
In the remainder of this section we will focus on $\ell mn$-preferred extensions and assume $S$ to be of such type.

\begin{figure}[t]
\centering
\includegraphics[width=4.2in]{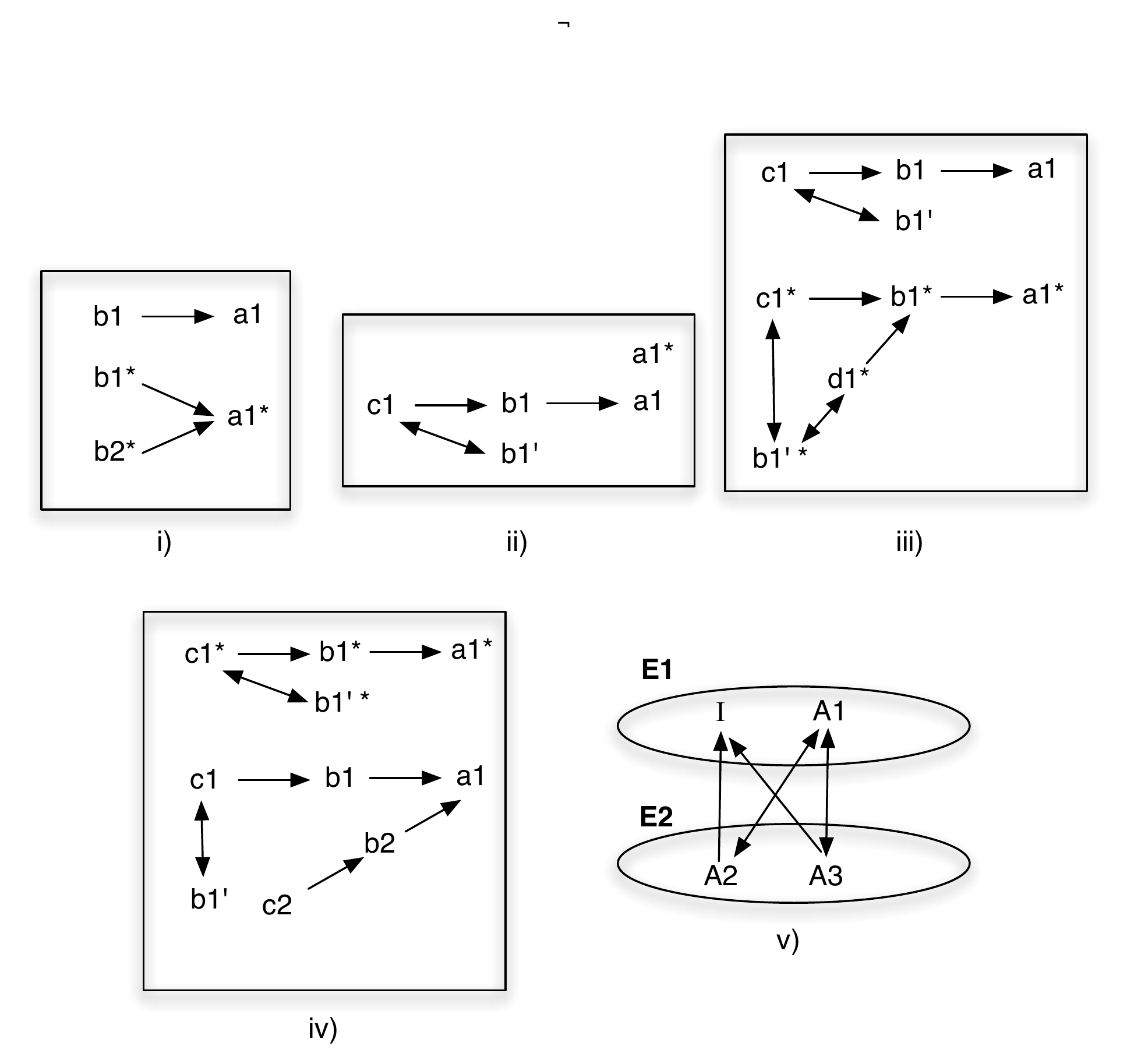}

\caption{$AF$s  in which $a1^*  \succ^{\emptyset} a1$, illustrating the analysis of arguments in Example \ref{ExTwoVOneAttack} (i), Example \ref{ExReinstatement} (ii),  Example \ref{ExReinstatement2Same} (iii),  and Example \ref{ExampleTwoVOneReinstatment} (iv). In v) we show the two Dung extensions for   Example \ref{ExAccrual}.  \label{SchCQAFs}}

\end{figure}

\begin{example}\label{ExReinstatement}
Suppose we now have only $A_1$, $B_1$ (and their sub-arguments) and the argument $C_1$ that cites an interpretation of the report $i$ = `Chilcot did not use the word lie'\footnote{Chilcot did not use the word ``lie''. His report specified that it ``is not questioning'' Mr Blair's fixed belief, but the former Prime Minister deliberately blurred the distinction between what he believed and what he actually knew. } to conclude that Chilcot did not assert that Blair lied about WMD ($\neg k$),  so that $C_1$ undermine attacks $B_1$ (on its sub-argument $B_1' = [k]$), and
$C_1$ undermine attacks $B_1'$ and  $B_1'$ rebut attacks $C_1$  (Figure \ref{Arguments}).
In an $AF$ consisting only of $A_1$ and its sub-arguments $A_1' = [b]$ and $A_1'' = [c]$, then for all $l,m,n$ set $\{A_1, A_1', A_1'' \}$  would be the single $lmn$-complete extension. However, in the $AF$ that includes $C_1 \rightarrow B_1 \rightarrow A_1$, $C_1 \leftrightarrow B1'$ (the additional subarguments of $C_1, B_1$ and $A_1$ are not listed here as none of these are involved in attacks)
$A_1$ is not in every  $lmn$-complete extension (for instance, $\{A_1,C_1\}$ is not $112$-admissible).
Abstractly, this equates with 
 $a1^*  \succ^S a1$   in the $AF$ shown in Figure \ref{SchCQAFs}ii (absolute ranking), and $a1^*  \succ^\emptyset a1$ (contextual ranking).
%

\end{example}

We note with interest that the higher ranking for unattacked versus reinstated arguments (illustrated by the above example) is supported by experimental findings reporting that human subjects appear to have higher confidence in claims of arguments that are unattacked, than when those arguments are subsequently attacked and then defended \cite{rahwan:behavioral}. This suggests that our theory of graded acceptability incorporates 
features of human argumentation.\footnote{We do see the theory of graded acceptability also as a contribution to the long term goal of providing formal frameworks of argumentation accommodating both computational and human argumentation \cite{AddValue}.}

\begin{example}\label{ExReinstatement2Same}
Suppose the $AF$ $\Delta$ that includes the arguments $A1, B_1$ and $C_1$ in Figure \ref{Arguments}, and the $AF$ $\Delta'$  that includes the additional argument $D_1$ that cites another interpretation of the report that concludes $\neg k$. Then $D_1$ also undermines attack $B_1$ on its sub-argument $B_1'$  and we also now have $D_1 \leftrightarrow B1'$. 
The defense of $A_1$ against $B_1$ by $\set{C_1, D_1}$ is now stronger than that offered by $\set{C_1}$ alone. Abstractly, this equates with  $a1^*  \succ^\emptyset a1$ in the $AF$ shown in Figure \ref{SchCQAFs}iii (contextual ranking). Furthermore,  $\{D_1,C_1,A_1\}$ is a $112$-preferred extension of $\Delta'$ but $\{C_1,A_1\}$ is not a  $112$-preferred extension of $\Delta$. Abstractly, this equates with  $a1^*  \succ^S a1$ in the $AF$ shown in Figure \ref{SchCQAFs}iii (absolute ranking).
\end{example}

\begin{example}\label{ExampleTwoVOneReinstatment}
Suppose the $AF$ $\Delta$ that includes the arguments  $A1, B_1$ and $C_1$ in Figure \ref{Arguments}, and  the $AF$ $\Delta'$
that now includes  the additional argument $B_2$ in Figure \ref{Arguments}, as well as the argument  $C_2$ claiming that the Arab League are not credible experts, so that $C_2$ undercut attacks $B_2$  (on  $r5$). That is, we have that $\Delta'$ includes:
\begin{quote}  $C_1 \rightarrow B_1 \rightarrow A_1$, $C_1 \leftrightarrow B1'$, $C_2 \rightarrow B_2 \rightarrow A_1$.
\end{quote} Intuitively, $A_1$ is more strongly justified in $\Delta$ than in $\Delta'$ since in the latter case we have two arguments ($B_1$ and $B_2$) that continue to exert a weakening effect on $A_1$ as opposed to the one  ($B_1$) in $\Delta$.
$A_1$ is in a $222$-preferred extension of $\Delta$ but $A_1$ is not in a  $222$-preferred extension of $\Delta'$. This equates with   $a1^*  \succ^S_{\Delta} a1$   in the $AF$ shown in Figure \ref{SchCQAFs}iv.

\end{example}


\subsubsection{Graded Acceptability and Accrual}

We now briefly illustrate, by means of a simple example, the relationship between graded acceptability and the notion of accrual. In the following section we then show how graded acceptability captures a simple form of accrual when applying graded semantics to a dialectical characterisation of non-monotonic inference.

\begin{example}\label{ExAccrual} Suppose an argument in support  of invading Syria, instantiating the scheme for practical reasoning \cite{atk05}:
\begin{quote} Assad is suppressing
Syrians, and invading Syria will remove Assad from power, and \textbf{removing Assad will
achieve  democracy in Syria}, so promoting the value of peace.
\end{quote}
Now, as well as pointing to counter-arguments, critical questions can also be posed as challenges to presumptions, shifting the burden of proof so  as to provide an argument in support of the presumption. So the critical question `Do the consequences of the action achieve the goal ?' can be posed to the presumption emphasised in bold in the above argument (with the consequences being `the removal of Assad' and the goal being `achieve  democracy in Syria'). Suppose that in response to this challenge, the above argument is then extended to an argument $I$ that now includes as a
 sub-argument, the argument $A_1$ in Figure \ref{Arguments} whose conclusion is the bold text in the above. Now suppose two additional arguments $A_2$ and $A_3$, each of which are instances of the AEO scheme. $A2$ cites the  Institute of Middle Eastern Studies at King's College London who are experts in the domain of middle east politics and who assert that removing Assad will
\emph{not} achieve  democracy in Syria (i.e., $A2$ concludes $\neg g$). $A_3$ cites the  United Nations working group on middle east affairs who are  also experts that assert that removing Assad will
\emph{not} achieve  democracy in Syria. 
Both $A_2$ and $A_3$ symmetrically rebut attack $A_1$. Moreover, both $A2$ and $A3$ asymmetrically rebut attack $I$ on its sub-argument $A1$. We thus have two Dung admissible extensions $E1$ and $E2$ in Figure \ref{SchCQAFs}v.

However, $A_2$ and $A_3$ accrue in support of each other's claim ($\neg g$) and so strengthen each other at the expense of $A_1$. Hence, although $E1$ and $E2$ are both subsets of Dung preferred extensions, we have that $A2, A3 \succ^{S}_{\Delta} A1, I$ (intuitively each attack on $I$ and $A_1$ is defended by one counter-attacker, whereas each attack on $A_2$ and $A_3$  is defended by two counter-attackers).


\end{example}

Observe  also  that $C_1$ and $D_1$ accrue in support of $\neg k$ to strengthen $A_1$ in Example \ref{ExReinstatement2Same}.
The incompatibility of accrual with Dung's standard theory has been discussed and argued for in \cite{PrakAccrual}. However, as illustrated in the above examples, our graded semantics partially challenges this view by showing how a simple counting-based form of accrual can be coherently accommodated within Dung's theory.


 \subsection{Graded Semantics and Characterisations of Non-monotonic Inference}\label{Sec:GradedPrefSubtheories}


A number of works \cite{Modgil2013361, Amgoud2010,Thang01082014} provide 
 argumentation-based characterisations of non-monotonic inference defined by Brewka's Preferred Subtheories \cite{bre89}. The latter starts with a
a \emph{totally} ordered ($\leq$) set $\B$ of classical wff \footnote{Where as usual, $\alpha \approx  \beta$ iff $\beta \leq \alpha$ and $\alpha \leq \beta$, and
$\beta < \alpha$ iff $\beta \leq \alpha$ and $\alpha \nleq  \beta$.} partitioned into equivalence classes
 $(\B_1,\ldots,\B_n)$ (for $i = 1 \dots n$, $\alpha, \beta \in \B_i$ iff $\alpha \approx  \beta$) and such that:
 \begin{quote}$\forall \alpha \in \B_i, \forall \beta \in \B_j$, $i <  j$ iff $\beta < \alpha$.
 \end{quote}
 A `preferred subtheory' (`\emph{ps}' for short) is obtained by taking a maximal under set inclusion consistent subset of $\B_1$, maximally extended with a subset of $\B_2$, and so on. In this way, multiple \emph{ps}s may be constructed; for example $(\B_1 = \{\neg a \vee  \neg b\},\B_2 = \{a , b \})$ yields
 the  \emph{ps}s $\{\neg a \vee  \neg b, a\}$ and $\{\neg a \vee  \neg b, b\}$.

\begin{figure}[t]
\centering
\includegraphics[width=4in]{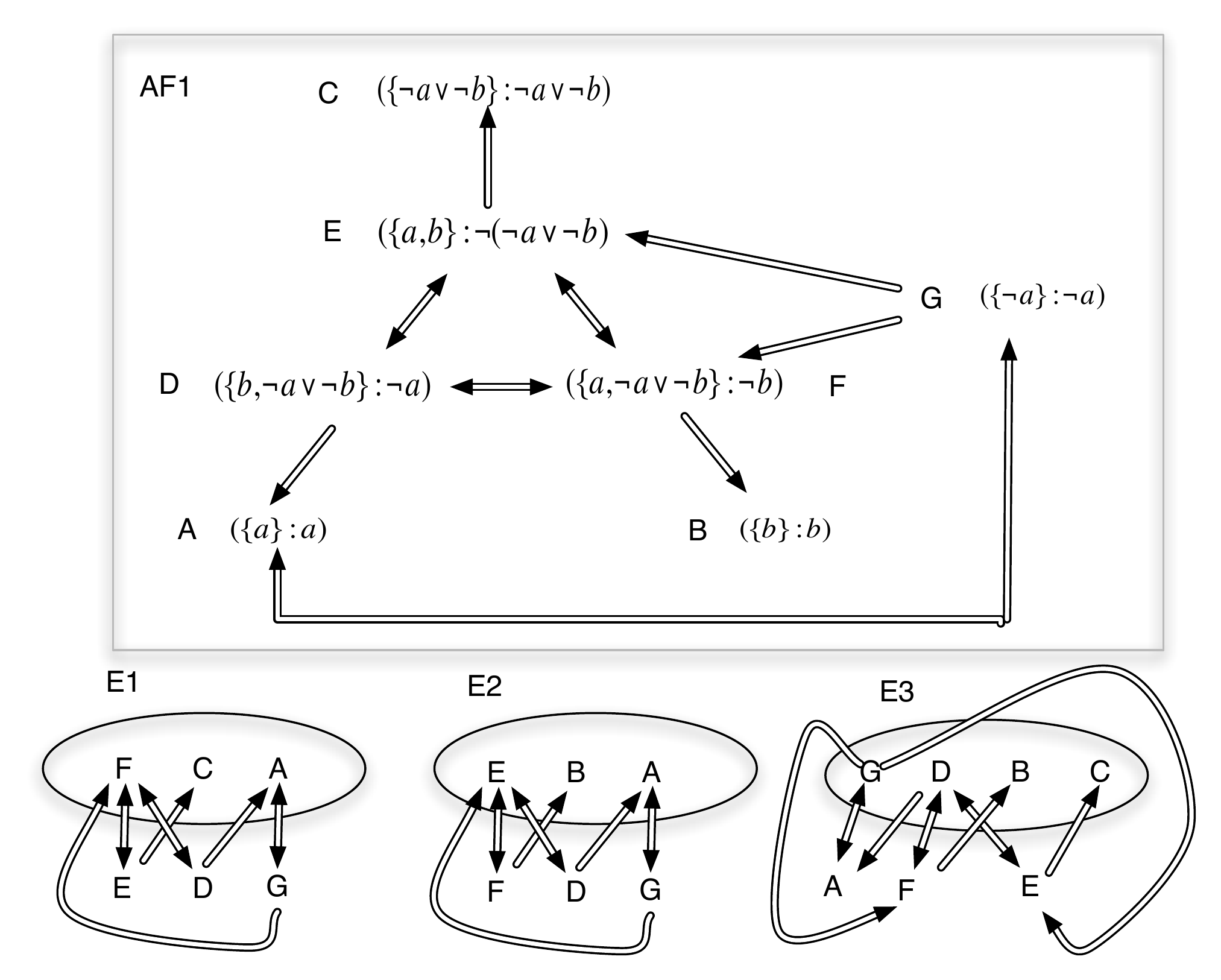}

\caption{Arguments in $AF1$ that defeat each other are connected by double headed arrows. $E1$, $E2$ and $E3$ are Dung admissible (i.e., $111$-admissible extensions). Only $E3$ is $112$-admissible.\label{FigPS}}

\end{figure}

\begin{figure}[t]
\centering
\includegraphics[width=4in]{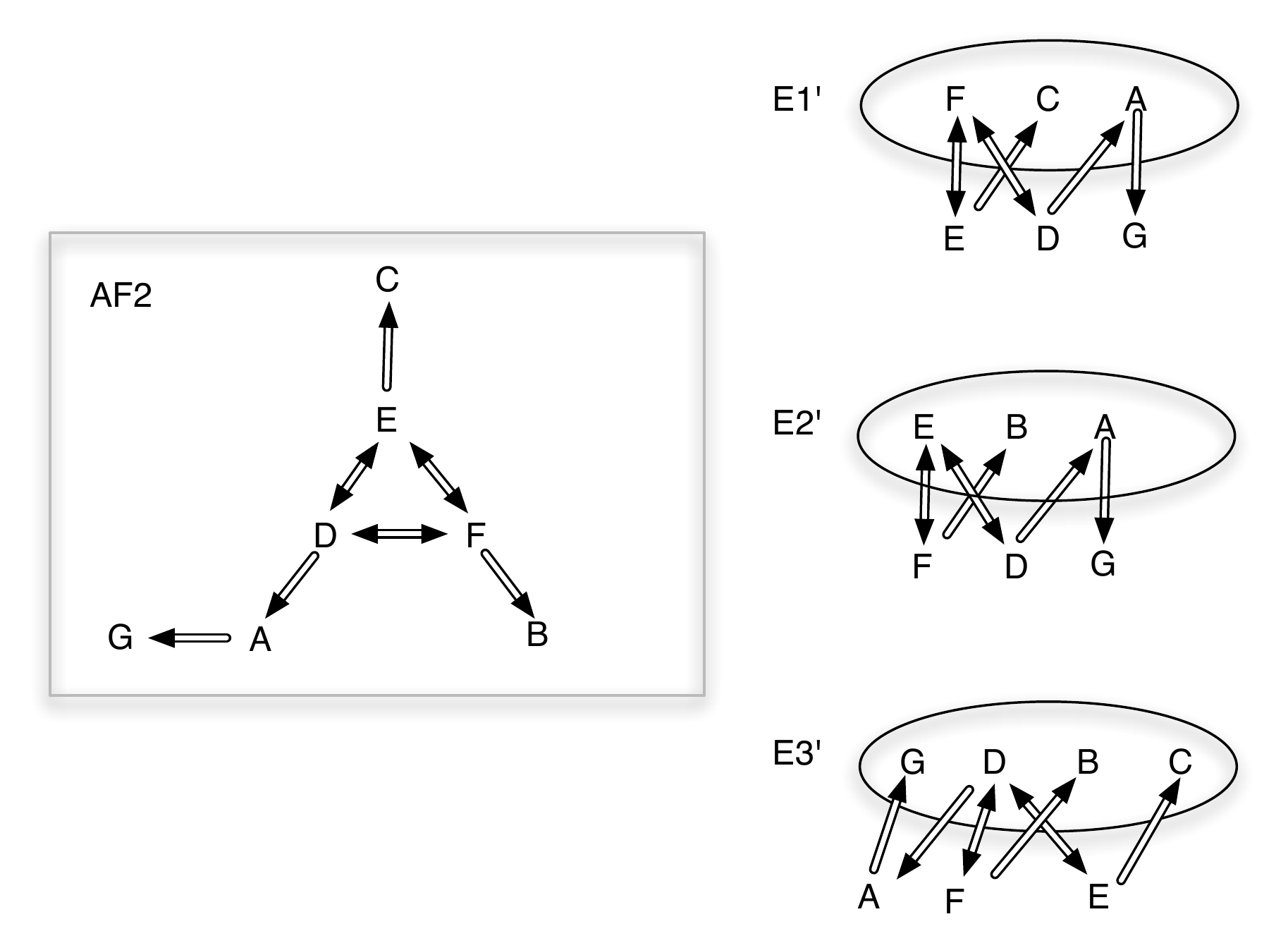}

\caption{$AF2$ defined by the base in which $\neg a$ is ordered below $a, b$ and $\neg a \vee \neg b$.  $E1'$, $E2'$ and $E3'$ are $111$-admissible extensions. None are$112$-admissible. \label{FigPS2}}

\end{figure}

Classical logic arguments in \textit{ASPIC+}
 consist of consistent subsets of premises $\Delta \subseteq  \B$ in support of a claim $\alpha$ classically entailed by $\Delta$ (and such that no proper subset of $\Delta$ entails  $\alpha$)  via a strict inference rule $\alpha_1,\ldots,\alpha_n \rightarrow \alpha$ ($\Delta = \bigcup_{i=1}^n\alpha_i$) that encodes the classical entailment, and which we represent here as $(\Delta,\alpha)$. The arguments constructed from $\B$ are then related by a binary defeat relation, obtained on the basis of the attacks and a strict preference relation over arguments defined by reference to the ordering on $\B$:

  \begin{quote}
   -- $(\Delta,\alpha) \prec (\Delta',\alpha')$  iff $\exists \gamma \in \Delta$, $\forall \beta \in \Delta'$: $\gamma < \beta$;
 \\[3pt]
  -- $(\Delta_1,\alpha_1)$ \emph{attacks} $(\Delta_2,\alpha_2)$ on $(\{\beta\},\beta)$ iff $\beta \in \Delta_2$, $\beta = \neg \alpha_1$ or\\[-16pt]

  \hspace{3mm}  $\alpha_1 = \neg \beta$; $(\Delta_1,\alpha_1)$ \emph{defeats} $(\Delta_2,\alpha_2)$ if  $(\Delta_1,\alpha_1) \nprec (\{\beta\},\beta)$.

   \end{quote}

\cite{Modgil2013361} then shows a correspondence such that each \emph{ps} of $\B$ is the set of premises in a stable extension of the $AF$ consisting of the arguments and defeats defined by the totally ordered $\B$. Then $\alpha$ is a sceptical (credulous) \emph{ps}-inference iff $\alpha$ is entailed by all (respectively at least one) \emph{ps}, iff $\alpha$ is a sceptical (credulous) argumentation defined inference (i.e., $\alpha$ is the conclusion of an argument in all, respectively at least one, stable extension). Moreover, it has subsequently been shown in \cite{modgil17StableEqualsPreferred}, that the stable extensions of any $AF$ of classical logic arguments and defeats, as defined above by a totally ordered set $\B$ of classical logic formulae, coincide with the preferred extensions.\footnote{It is well known that any stable extension of an $AF$ is preferred. Therefore \cite{modgil17StableEqualsPreferred} focuses on showing that any preferred extension is also stable.} The correspondence in \cite{Modgil2013361} can therefore also be stated for preferred subtheories and preferred extensions, and so in what follows, we focus on graded preferred semantics.


We can apply graded semantics to $AF$s that relate arguments by \emph{defeats} rather than attacks. Firstly, assume  $\B$ is the single equivalence class  $\B_1$ = $\{a,b,\neg a \vee \neg b, \neg a\}$, generating the defeat graph $AF1$ in Figure \ref{FigPS} in which all attacks succeed as defeats (represented as double arrows). Note that for every argument $(\Delta,\alpha)$ shown, the $AF$ also includes the argument $(\Delta,\beta)$ where  $\beta$ is any classical consequence of $\Delta$. However, no attacks originate from these additional arguments. It therefore suffices to consider only the arguments shown in Figure \ref{FigPS}, since these are the only arguments from which attacks originate.
Hence if an extension of $AF1$ contains a $(\Delta,\alpha)$ shown in Figure \ref{FigPS}, then it will also contain any $(\Delta,\beta)$ such that $\Delta \vdash \beta$.

Now observe that $E3$ is a subset of the single $112$-preferred extension that includes all arguments with consistent subsets of $\{b,\neg a \vee \neg b, \neg a\}$ (every defeat on $X \in E3$ is defended by two arguments in $E3$). However $E1$ and $E2$ are not $112$-admissible, and are not subsets of $112$-preferred extensions. Indeed, we obtain that $G, D, B, C \succ^{\ell mn-preferred}_{AF1} A, E, F$ according to Definition \ref{definition:degree-semantics}. Once again we witness how graded semantics can effectively account for a form of accrual. We have two distinct arguments claiming $\neg a$ ($G$ and $D$) which accrue so as to strengthen each other, and so privilege the inference $\neg a$ over $a$. Moreover these arguments accrue in their defense of $B, C$ and $D$.

Notice that explicit preferences over arguments take precedence over the implicit preference yielded by accruing arguments and the arguments they defend.
Suppose that $\B$ consists of the two equivalence classes $\B_1$ = $\{a,b,\neg a \vee \neg b\}$ and $\B_2$ = $\{\neg a\}$ (i.e., $\neg a < a,b,\neg a \vee \neg b$) generating the defeat graph $AF2$ in Figure \ref{FigPS2}. The fact that $G \prec A$ now means that $G$ and $D$ no longer accrue in support of $\neg a$ and an implicit preference for $\neg a$, and no longer accrue in defense of $B, C$ and $D$. Now
none of $E1'$, $E2'$ and $E3'$ are subsets of  $112$-preferred extensions, and $\forall X,Y \in \{G, D, B, C, A, E, F\}$,  $X \approx^{\ell mn-preferred}_{AF2} Y$. 



\section{Related Work} \label{Sec:RelatedWork}

We discuss related work and provide a detailed comparison of our semantics based on graded acceptability with some of the approaches proposed in the literature.

\subsection{Approaches to Graduality in Abstract Argumentation}

The existing literature aiming at introducing some form of 'graduality' or 'ranking' in abstract argumentation can roughly be classified in two strands: those introducing a more fine-grained notion of argument status, and which are closely based on Dung's theory; those that depart from Dung's notions of defense, acceptablity and extensions under various semantics, so as to provide a unique ranking on arguments, typically through the use of techniques for value propagation on graphs.

Within the first category, \cite{wu10labelling} pushes the envelope of Dung's theory by showing how standard notions (e.g., belonging to or being attacked by the grounded extension, belonging to some or no admissible sets, and being attacked by some or no admissible sets) suffice to isolate six different statuses of arguments, without introducing an explicit notion of graduality. In \emph{Weighted Argumentation Systems} (\emph{WAS}) \cite{dunne11weighted}, weights are associated with attacks, and Dung extensions are generalised to relax the requirement for conflict-freeness, and allow for extensions whose contained attacks' summative weight does not exceed some predefined `inconsistency budget'. The case where the weight of each attack is 1 (which we refer to as \emph{WAS}$_1$), and in which the inconsistency budget therefore equates with the number of attacks tolerated in an extension, invites comparison with our $\ell 11$ graded semantics accommodating extensions that are not necessarily Dung conflict free (i.e, when $\ell > 1$).\footnote{We will in Section \ref{Sec:Future} suggest future research in which we generalise graded neutrality, defense and semantics to account for weights on attacks.} The first thing to note is that \emph{WAS}$_1$ considers the neutrality of an extension w.r.t. the  number of contained attacks, rather than w.r.t. the number of attacks on any given contained argument. Moreover, we have argued that toleration w.r.t. the co-acceptability of attacking arguments should go hand in hand with toleration w.r.t. weaker forms of defence.
However, the implications of relaxing conflict-freeness on the existence and construction of Dung's semantics in \emph{WAS} have yet to be studied in depth (\cite{dunne11weighted} only notes that the existence of a \emph{unique} grounded extension is not guaranteed).

Amongst the graph propagation approaches, we mention the equational approaches that assign a more fine grained ranking to arguments by evaluating fixed points of functions that assign a numerical value to any given argument based on the values of its attackers. In particular, the equational approach of Gabbay and Rodrigues \cite{GabbayR16} who conjecture that their approach yields a unique solution for cyclic graphs, the compensation based semantics of Amgoud et.al. \cite{Amgoud:2016}, which assigns the same ranking to all arguments in cycles and yield a unique solution for cyclic graphs, and the social argumentation approach of Leite and Martin (\textbf{LM}) \cite{conf/ijcai/LeiteM11} (who are concerned with propagating user votes on arguments, but also yield fine grained rankings without recourse to exogenous information).

Other graph propagation approaches that do not use an equational fixed point approach,  typically account for the attack and defense paths terminating in the argument being ranked, where these paths are  sequences of, respectively even and odd, numbers of attacking arguments.
Besnard and Hunter (\textbf{BH}) \cite{BesnardHunter01} rank classical logic arguments in acyclic graphs, through a  \emph{categoriser} function -- $v: \A \rightarrow [0,1]$ defined as $v(a) = \frac{1}{1 + \sum_{b \rightarrow a} v(b)}$ -- that assigns high values to arguments with low-valued attackers (and the maximum value to un-attacked arguments) and low values to arguments with high-valued attackers.
Cayrol and Lagasquie-Schiex (\textbf{CL}) \cite{CLS04} then generalises use of this function to Dung $AF$s to develop (in their terms) a `local approach' to valuation of arguments, and then formalise a `global approach' that they argue gives more intuitive outcomes. Their approach requires a highly involved transformation of cyclic graphs to infinite acyclic graphs. More recently,
Amgoud and Ben-Naim \cite{AmgoudNaim} propose
 two ranking-based semantics, which they call discussion-based (\textbf{AB-d}) and burden-based (\textbf{AB-b}). These semantics are also based on the processing of attack paths and their general applicability relies on conjectures concerning the processing of cyclic paths.

Finally, Matt and Toni (\textbf{MT}) \cite{conf/jelia/MattT08} provide a highly original paradigm for ranking arguments, defining argument strength in terms of the value of a repeated two-person zero-sum strategic game with imperfect information.


\subsection{Graded Semantics in the Landscape of Ranking-Based Semantics}


Properties of rankings have been proposed by the above works, and a very informative comparison of these approaches---in particular  \textbf{LM}, \textbf{BH}, \textbf{CL} (for acyclic $AF$s), \textbf{AB-d},\textbf{AB-b} and \textbf{MT}---in terms of whether or not these properties are satisfied, has been provided in \cite{bonzon16comparative}. It is instructive to study these properties as they apply to our graded rankings, in part because such a study reveals fundamental distinctions between propagation based and Dung semantics based approaches to evaluation of arguments.

In what follows, we refer to absolute graded rankings $x \succeq^{S} y$  (Definition \ref{definition:degree-semantics}), which are the more naturally comparable with the other existing approaches to graduality.
Unless stated otherwise, we assume $S$ stands for any of the semantics in $\set{\mathit{grounded}, \mathit{preferred}, \mathit{stable}}$.
Moreover, we assume the constraint $n \geq m$ and $l \geq m$. Albeit not essential for the comparison, this assumption streamlines some of the proofs in this section.
We now turn to the properties studied in \cite{bonzon16comparative}, and discuss whether our approach satisfies each of them. In doing so we will recall which approaches satisfy each property (writing \textbf{All} if satisfied by all, and \textbf{None} if satisfied by none).


\subsubsection{Abstraction and Independence}

\textbf{Abstraction} (\textbf{All}) states that rankings of arguments are preserved by isomorphisms.
 It should be obvious to see that graded rankings satisfy this property.

\smallskip

\textbf{Independence} (\textbf{All}) states that the ranking of $a$ and $b$ should be independent of any argument that is neither connected to $a$ or $b$.
Formally, the connected components of an $AF$ $\Delta$ are the set of largest subgraphs of $\Delta$, denoted $cc(\Delta)$, where two arguments are in the same component iff there is some path of attacks (ignoring the direction of attack) between them. Graded rankings satisfy independence:

\begin{proposition}\label{PropInd}
$\forall \Delta' = (\A',\rightarrow') \in cc(\Delta = (\A,\rightarrow))$, $x, y \in \A'$, and positive integers $\ell$, $m$ and $n$ such that $\ell \geq m$ and $n \geq m$:
\begin{align}
x \succeq^{S}_{\Delta'} y \IMPLIES x \succeq^{S}_{\Delta}  y. \label{eq:independence}
\end{align}
\end{proposition}

\begin{proof}
To prove \eqref{eq:independence} we establish the following diagram of implications. For any integers $\ell$, $m$ and $n$ such that $\ell \geq m$ and $n \geq m$, for any extension
$S \in \{\mathit{grounded}, \mathit{preferred}$, $\mathit{stable}\}$ such that $S_{\ell mn}(\Delta) \neq \emptyset$:\footnote{This condition is obviously relevant only for $S = \mathit{stable}$.}
\begin{center}
\begin{tabular}{cccr}
$y \in \bigcap S_{\ell mn}(\Delta')$ & $\IMPLIES$ & $x \in \bigcap S_{\ell mn}(\Delta')$ & Assumption \\
$\Updownarrow$                           &                  &  $\Updownarrow$                          & \fbox{Claim} \\
$y \in \bigcap S_{\ell mn}(\Delta)$ & $\IMPLIES$ & $x \in \bigcap S_{\ell mn}(\Delta)$ & Conclusion
\end{tabular}
\end{center}
for $x, y \in \Delta'$.
We need to establish the equivalences denoted \fbox{Claim} in the above diagram, for each of the semantics. Firstly note that, clearly, since $\Delta' \in cc(\Delta)$:\footnote{Recall notation introduced in Definition \ref{definition:attack_graph}.}
\begin{equation}\label{PropIndEq}
\forall x \in \Delta' : \overline{x}_{\Delta'} = \overline{x}_\Delta \AND \overline{\overline{x}}_{\Delta'} = \overline{\overline{x}}_\Delta
\end{equation}

\fbox{$S = \mathit{grounded}$}
 It suffices to show that
 $\forall lmn$ satisfying the constraint we have that  $\grn_{\ell mn}(\Delta') = \grn_{\ell mn}(\Delta) \cap \A'$. We proceed by induction.\footnote{To illustrate the argument we use here standard induction, which to be precise works only for the case in which $\Delta$ is finitary. For the general case one needs to generalize the argument in the obvious way to transfinite induction. Cf. Remark \ref{remark:ordinal}.} By Corollary \ref{fact:grounded} $\grn_{lmn}(\Delta) = \bigcup_{0 \leq k < \omega} \cfmn{m}{n}^k(\emptyset)$. For the base case, clearly $\cfmn{m}{n}^{\Delta'}(\emptyset) = \cfmn{m}{n}^{\Delta}(\emptyset) \cap \A'$. Assume then that for $i > 1$ $(\cfmn{m}{n}^{\Delta'})^i(\emptyset) = (\cfmn{m}{n}^{\Delta})^i(\emptyset) \cap \A'$ (IH). By \eqref{PropIndEq} and IH it immediately follows that the same claim holds for $i +1$ and therefore, in the limit, $\grn_{\ell mn}(\Delta') = \grn_{\ell mn}(\Delta) \cap \A'$ as desired.

\fbox{$S = \mathit{preferred}$}
First of all, one can straightforwardly show that $\forall E$, if $E \in \prf_{\ell mn}(\Delta)$ then $E \cap \A' \in \prf_{lmn}(\Delta')$. So for any $lmn$-preferred extension in $\Delta$ there is a corresponding $lmn$-preferred extension in $\Delta'$. To establish the claim it then suffices to prove the converse, that is:
\begin{equation} \label{PropIndEq2}
  \forall E' \in \prf_{lmn}(\Delta'), \exists E \in \prf_{\ell mn}(\Delta) \ST E' = E \cap \A'.
\end{equation}
To prove \eqref{PropIndEq2} an inductive argument similar to the above one for $S = \mathit{grounded}$ can be used by exploiting  Fact \ref{fact:prf}, for each $\ell m n$-preferred extension. The claim then follows again from \eqref{PropIndEq}.

\fbox{$S = \mathit{stable}$}. Assume $\stb_{\ell mn}(\Delta) \neq \emptyset$. It follows that  $\stb_{\ell mn}(\Delta') \neq \emptyset$. Then the same argument used for the case of preferred extensions applies to prove \fbox{Claim}. If instead, $\stb_{\ell mn}(\Delta) = \emptyset$ then it trivially holds that $x \in \bigcap \stb_{\ell mn}(\Delta) \Leftrightarrow y \in \bigcap \stb_{\ell mn}(\Delta)$, no matter whether $y \in \bigcap \stb_{\ell mn}(\Delta') \Rightarrow x \in \bigcap \stb_{\ell mn}(\Delta')$. This suffices to establish \eqref{eq:independence} for $S = \mathit{stable}$.
\end{proof}

An inspection of the proof of Proposition \ref{PropInd} should show, however, that the claim can be strengthened to $\forall \Delta' \in cc(\Delta)$: $x \succ^{S}_{\Delta'} y$ $\Rightarrow$ $x \succ^{S}_{\Delta}  y$, only for
$S \in \{\mathit{grounded},$ $\mathit{preferred}\}$. This stronger formulation fails for $S = \mathit{stable}$ because of the well known non-existence of Dung stable extensions due to arguments in an odd cycle `contaminating' unrelated arguments. For example, in
\begin{quote}
\hspace{-5mm}$\Delta = \langle \{a,b,c\}, \{a \rightarrow a, b \rightarrow c \}\rangle$ and $\Delta' \in cc(\Delta) = \langle \{b,c\}, \{b \rightarrow c \}\rangle$
\end{quote}
we have that $b \succ^{stable}_{\Delta'}  c$ since $b$, but not $c$, is in the single $111$-stable---i.e., Dung stable---extension. However for $\Delta$, there are no $\ell$, $m$ and $n$ such that $b$, but not $c$, is justified under $\ell m n$-stable semantics.

\smallskip

So with respect to the properties of abstraction and independence, graded rankings behave in line with most existing approaches to graduality, with some important differences when the underlying semantics is assumed to be the stable one.


\subsubsection{Void Precedence and Self Contradiction}

The above example also illustrates why \textbf{Void Precedence} (\textbf{All})---any non-attacked argument is ranked strictly higher than any attacked argument---is also not satisfied by graded rankings under stable semantics, as $b \approx^{\mathit{stable}}_{\Delta}  c$. However, the situation is more positive under grounded and preferred semantics, for which an---arguably---natural weakening of the property holds.

We need some auxiliary notation. Define the variant of the graded defense function (Definition \ref{definition:sensitive}) that requires an infinity of defenders for any attack: $\cfmn{1}{\infty} = $ $\{x \in \A \mid$ $\forall y \ST y \ar x, \exists^\infty z \in X \ST z \ar y\}$. So this function outputs, for any set, the set of arguments whose attackers are counter-attacked by an infinity of arguments in that set. Consider now the smallest fixpoint $\lfp. \cfmn{1}{\infty}$ of such a function. We can refer to this as the $11\infty$-grounded extension. Observe right away that if the underlying framework is finitary, then $\lfp. \cfmn{1}{\infty}$ is equal to the set of unattacked arguments in that framework.  We have the following result:

\begin{proposition}
Let $\Delta = \langle \A, \rightarrow \rangle$ be an $AF$ and let $x \in \A$ be s.t. $x \in \lfp. \cfmn{1}{\infty}^\Delta$. Then $\forall y \in \A \backslash \lfp. \cfmn{1}{\infty}^\Delta$ s.t. $\overline{y} \neq \emptyset$:  $x \succ^{S}_{\Delta}  y$, for $S \in \set{\mathit{grounded},\mathit{preferred}}$.
\end{proposition}

\begin{proof}
Let $y \in \A \backslash \lfp. \cfmn{1}{\infty}^\Delta$. Then $\overline{y} \neq \emptyset$ and $\overline{\overline{y}}$ is finite. Then let $k = \max \set{| \overline{z} | \mid z \in \overline{y}}$. It follows that there exists no $11k+1$-admissible (and hence grounded or preferred) extension which contains $y$. Hence, if $x \in \lfp. \cfmn{1}{\infty}^\Delta$ then $x \succ^{S}_{\Delta}  y$.
\end{proof}
A direct consequence of the proposition is that, if $\Delta$ is finitary, the claim simplifies to the standard formulation of void preference, that is: $\forall x \in \A$ s.t. $\overline{x} = \emptyset$, then $\forall y \in \A$ s.t. $\overline{y} \neq \emptyset$:  $x \succ^{S}_{\Delta}  y$, $S \in \set{\mathit{grounded},\mathit{preferred}}$.

\smallskip

Finally, and related to the above discussion, the property of  \textbf{Self Contradiction} (\textbf{MT}), which states that a self-attacking argument is ranked strictly lower than any non self-attacking argument,  is also not satisfied by our approach. Consider $\Delta$ =  $\langle \{a,a1,a2,b\}$, $\{a \rightarrow a , a1 \rightarrow b, a2 \rightarrow b\} \rangle$, where $a$ is in all $222$-admissible (and hence $222$-preferred) extensions, whereas  $b$ is not in any $222$-admissible extension, and so $b \nsucceq^{\mathit{preferred}}_{\Delta}  a$.

\medskip

So with respect to the properties of void precedence and self contradiction, graded rankings behave essentially in line with most existing approaches to graduality. Some differences arise with respect to void precedence, when the underlying semantics is assumed to be the stable one and when frameworks are not finitary, in which case graded rankings cannot distinguish between unattacked arguments and arguments whose attackers are recursively counter-attacked by an infinity of defenders ($\lfp. \cfmn{1}{\infty}$).


\subsubsection{Further Postulates}

\textbf{Cardinality Precedence}  (\textbf{AB-d},\textbf{AB-b}) states that if the number of attacks on $b$ is strictly greater than the number of attacks on $a$, then $a$ is ranked strictly higher than $b$. Like \textbf{LM}, \textbf{BH}, \textbf{CL} and \textbf{MT}, our approach does not satisfy this postulate. This is because our ranking also accounts  for the number of defenders, as witnessed by the incomparability of $a3$ and $a4$ in Figure \ref{Motivating1}, and the fact that $a3 \nsucceq^{grounded}_{AF}  a4$ and $a4 \nsucceq^{grounded}_{AF}  a3$ (recall Example \ref{Ex:DEJ}). It is worth remarking however that, as discussed in Section \ref{Sec:ExtendingPartialOrder}, our ranking could be further refined to a lexicographic ordering giving precedence to the minimization of attackers, that would then satisfy the cardinality precedence postulate.

\smallskip

\textbf{Quality Precedence} (\textbf{None}) states that if there is an attacker of $b$ that is ranked strictly higher than all attackers of $a$, then $a$ is ranked strictly higher than $b$. The principle is not satisfied by our approach. We leave it to the reader to verify that given $e \rightarrow c \rightarrow b$, $e1 \rightarrow d1 \rightarrow a$, $e2 \rightarrow d1$, $e3 \rightarrow d2 \rightarrow a$, $e4 \rightarrow d2$, then although
each of $d1$ and $d2$ are ranked lower than $c$, both $a$  and $b$ are incomparable.

\begin{figure}[t]
\begin{center}
\includegraphics[scale=0.6]{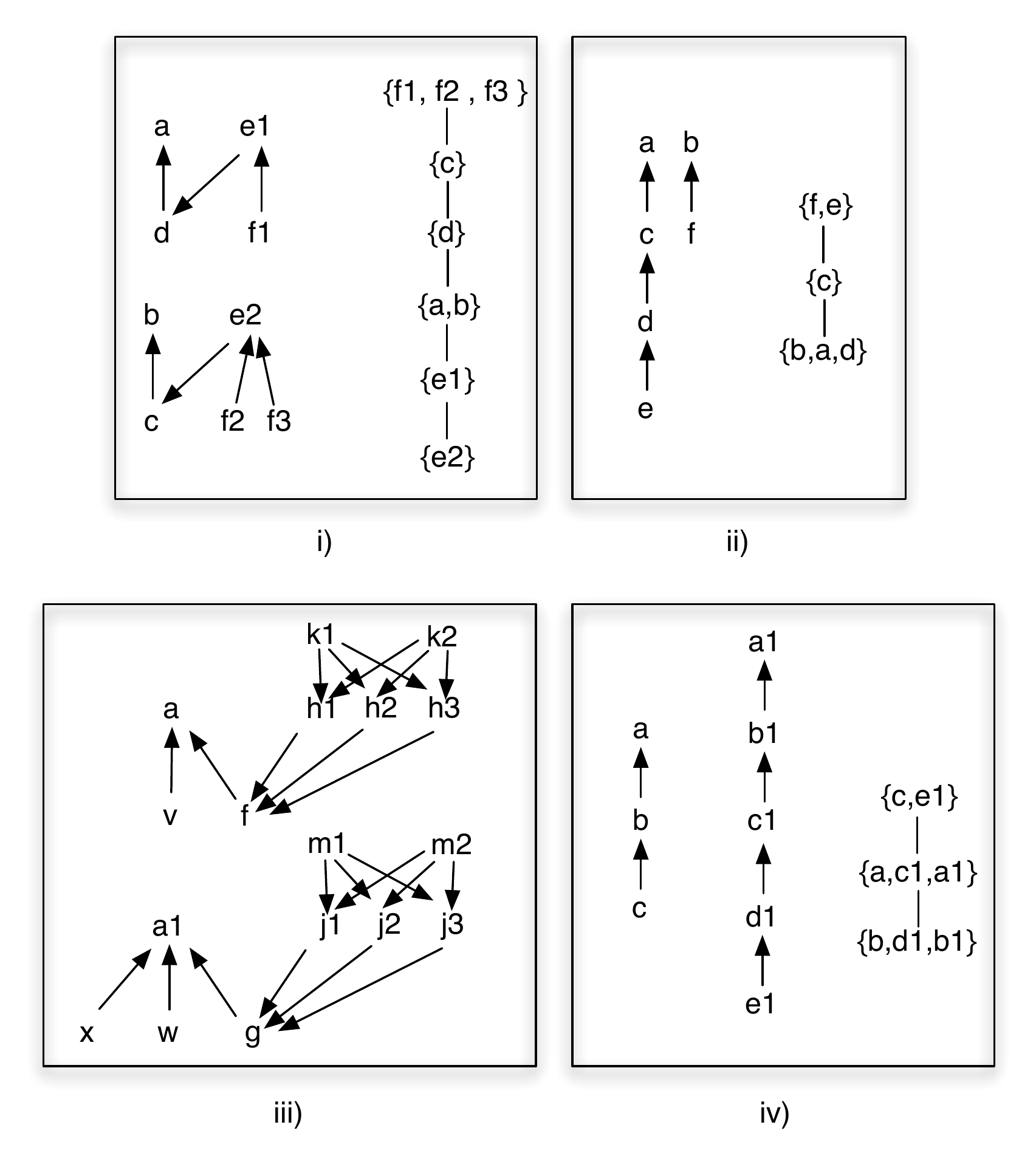}
\end{center}
\caption{Frameworks and graded rankings referenced in the discussion of postulates for ranking based semantics reviewed in \cite{bonzon16comparative}.}
\label{SCTCounterExample}
\end{figure}

\smallskip

\textbf{Strict Counter-transitivity} (\textbf{LM}, \textbf{BH}, \textbf{AB-d},\textbf{AB-b}) states that if the attackers of $b$ are strictly more than those of $a$ (i.e., cardinality precedence), or
the number of attackers of each are the same, but at least one attacker of $b$
is ranked strictly higher than at least one attacker of $a$, and not vice versa,
then $a$ is ranked strictly higher than $b$. Like \textbf{CL} and \textbf{MT}, graded rankings  do not satisfy this property. Consider the  $AF$ and ranking of arguments in Figure \ref{SCTCounterExample}i), where $a$ and $b$ both have one attacker (respectively $d$ and $c$). Then $c \succ^{S}_{AF} d$, but $a \nsucc^{S}_{AF}  b$. In fact,  $a \approx^{S}_{AF}  b$. With respect to this, it is worth observing the following. By virtue of our generalisation of Dung semantics, analysis of any individual argument under graded semantics is inherently bound up in the analysis of sets of (co-acceptable) arguments (recall the discussion in Section \ref{Sec:GradedAcceptabilityIntuitions}). Now, observe that the strongest defense needed to accept the defender $e1$ of $a$, and that then suffices
to defend $a$, is when $m$ = $2$ (since $e1$ is not defended). But then any such standard of defense also accepts $b$ (since the attack by $c$ need not then be defended).



\smallskip

\textbf{Defense Precedence} (\textbf{LM}, \textbf{BH}, \textbf{AB-d},\textbf{AB-b})  states that for two arguments with the same number of attackers, a defended argument is ranked strictly higher than a non-defended argument. Again, like \textbf{CL} and \textbf{MT}, this property is not satisfied by graded rankings, as illustrated in Figure  \ref{SCTCounterExample}ii), where the defended $a$ and undefended $b$  have the same ranking.
We again see that
 the analysis of any individual argument is inherently bound up in the analysis of sets of (co-acceptable) arguments. In this example, the strongest defense needed to accept $d$ and that suffices to accept $a$, is when $m = 2$ (since $d$ is undefended). But then this suffices to accept $b$.


\smallskip

Cayrol and Lagasquie-Schiex \cite{CLS04} propose a number of properties that relate to the addition and extension of attack/defense paths.\footnote{These properties, referencing addition of and extension (increase in) paths, assume that the arguments in the additional (extension to the) path are disjoint from the existing arguments.} Improving the ranking of an argument by \textbf{Strict Addition of a Defense Path}\footnote{We use the term `path' rather than the term `branch' used in  \cite{bonzon16comparative}.} (\textbf{None}) is clearly not satisfied by graded rankings. Indeed, the reverse is the case: $a$ is ranked higher when un-attacked than when attacked and then defended (as discussed in Example \ref{ExReinstatement}). Restriction of this property to the case where a defense path is added only to an argument that is already attacked, is satisfied only by \textbf{CL}, and is not satisfied by graded rankings.


Satisfaction of \textbf{Addition of an Attack Path} (\textbf{All}) means that the ranking of an argument $a$ is degraded when adding an attack path terminating in $a$, and
 equates with $a$ being ranked strictly higher than $a1$ in the $AF$
 $\langle c \rightarrow b \rightarrow a,  c1 \rightarrow b1 \rightarrow a1, x_n \rightarrow \ldots \rightarrow x1 \rangle$, where $x_n \rightarrow \ldots \rightarrow x_1$ is a path of attacks s.t. $x_1 = a_1$ and $n \in 2 \mathbb{N}$. This property is in general not satisfied by graded rankings. Consider the example $\Delta$ in Figure \ref{SCTCounterExample}iii). The strongest iterated defense from $\emptyset$ (and hence the `strongest' $lmn$-grounded extension) that includes $a$ is when $l = m = n = 3$. This defense also includes $a1$. Observe that: $\cfmn{3}{3}(\emptyset)$ = $E_0 = \{k1,k2,h1,h2,h3,a,v,m1,m2,j1,j2,j3,x,w\}$; $\cfmn{3}{3}(E_0)$  = $E_1 = E_0 \cup \{a1\}$; $\cfmn{3}{3}(E_1)$ = $E_1$ is the $333$-grounded extension. Indeed, there is no $lmn$-grounded extension that contains $a$ and not $a1$. Intuitively, if $m = 2$ ($n \geq 2$), then the defense of $a$ requires at least two of $h1,h2$ and $h3$ (since $v$ is unattacked and so $f$ must be defended against by $n \geq 2$ arguments), but then a defense of any $h_i$ requires that $m = 3$ (since each $h_i$ is attacked by  two unattacked arguments). Hence inclusion of $a$ in an $lmn$ grounded extension requires that $m \geq 3$. But then any such defense allows an additional attack path terminating in $a1$, while accommodating $a1$ under any such standard of defense.
 Put briefly, and recalling our discussion of strict counter-transitivity and defense precedence,
 it is the standard of defense met by $h1,h2$ or $h3$ that determines the strongest defense accommodating $a$, and this standard of defense allows for an additional attack path on $a1$.


\smallskip

 The two properties, \textbf{Increase of an Attack Path} (\textbf{All} except \textbf{MT}) and \textbf{Increase of a Defense Path} (\textbf{All} except \textbf{MT}) are distinctive of the propagation based approaches whereby the ranking of arguments is propagated down chains of attacks.  Neither is satisfied by graded rankings. The former equates with  $b1$ being ranked strictly higher than $b$, and  the latter equates with $a1$ being ranked strictly lower than $a$, in the $AF$ in Figure \ref{SCTCounterExample}iv).
We have that   $b \approx^{S}_{AF} b1$ and $a \approx^{S}_{AF} a1$. Once again,
the strongest standard of defense needed to accept $a1$ is that needed to defend $c1$, which in turn has the same ranking as $a$.


 \smallskip

 To the above discussed properties, \cite{bonzon16comparative} also proposes that all arguments can be compared according to their ranking, which is clearly not satisfied by graded rankings (cf. discussion in Section \ref{Sec:ExtendingPartialOrder}), and that all non-attacked arguments have the same ranking, which is obviously satisfied by our approach.


\subsubsection{Discussion}

We briefly comment on the above analysis and its positioning of our proposal within the growing field of ranking-based semantics.
Like all existing approaches, our rankings satisfy key properties such as abstraction and independence, as well as void precedence, albeit with some interesting caveats concerning graded stable extensions and the finitariness condition. With respect to other postulates, rankings based on graded acceptability tend to behave differently from approaches based on the propagation idea (and appear to behave more closely to \textbf{MT} which is also underpinned by intuitions different from propagation), for reasons inherent to the way graded semantics generalise standard Dung acceptability. To recap, we recall
the discussion in Section \ref{Sec:GradedAcceptabilityIntuitions} where we emphasise that our graded theory of acceptability aims at capturing a notion of graduality while at  the same time retaining the Dungian focus on  {\em sets of} arguments, rather than individual arguments, as the units of analysis. Hence, as our above analysis repeatedly shows, it is the defense of co-acceptable defenders of an  argument $x$ that determines the strongest defense needed to accept (and hence determine the ranking of) $x$. We further illustrate this point with a variation of Section \ref{SecGradedSchCQ} 's examples of graded semantics applied to frameworks instantiated by schemes and critical questions.

\begin{figure}[t]
\centering
\includegraphics[width=3.7in]{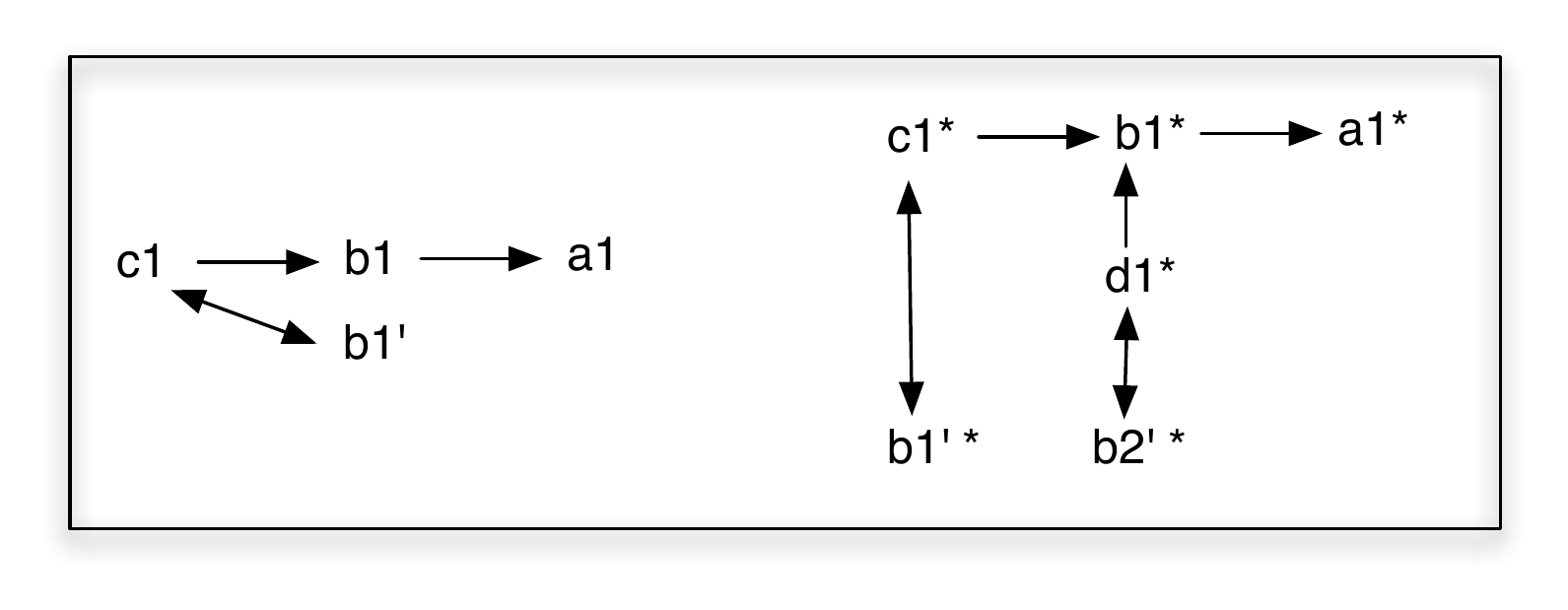}

\caption{$AF$  in which $a1^*  \approx^{\emptyset} a1$.  \label{SchCQAFs2}}

\end{figure}

\begin{example}\label{ExReinstatement2Diff}
Suppose the $AF$ $\Delta$ that includes the arguments $A_1, B_1$ and $C_1$  in Figure \ref{Arguments}, and the $AF$ $\Delta'$  that includes the additional argument $D_1$ that concludes $\neg h$ (i.e., an argument addressing \emph{AEO}'s critical question AEO2, by claiming that  Chilcott is not an expert in the domain of Blair's conduct in the Iraq war\footnote{For example, because Chilcott was not given access to the relevant documents reporting on Blair's deliberations.}). Notice that unlike Example \ref{ExReinstatement2Same}, $D_1$ undermine attack $B_1$ on a different sub-argument $B_2' = [h]$ to that (i.e., $B_1'$) undermine attacked by $C_1$. We now also have  $D_1 \leftrightarrow B_2'$. Now, $A_1$ is not then strengthened upon inclusion of $D_1$.  Abstractly, this equates with the ranking $a1^*  \approx^S a1$ in the AF in Figure \ref{SchCQAFs2}. Intuitively, the strongest defense needed to accept a defender (either  $c1^*$ or $d1^*$) of $a1^*$, and that then suffices
to defend $a1^*$, is when $m$ = $1$ and $n$ = $1$. But then this is also the strongest defense needed to accept $a1$. Contrast this with Example  \ref{ExReinstatement2Same} in which $C_1$ and $D_1$ accrue in support of the same conclusion $\neg k$ so that both  undermine attack $B_1$ on $B_1'$. Recall that this equates with $a1^*  \succ^S a1$ in the AF in Figure \ref{SchCQAFs}iii), in which the strongest defense needed to accept a defender (either  $c1^*$ or $d1^*$) of $a1^*$ that then suffices
to defend $a1^*$, is when $m$ = $1$ and $n$ = $2$. This standard of defense does not suffice to accept $a1$ in this example. Finally, notice that in the contextual approach to rankings, we do have that  $a1^*  \succ^{\set{c_1, c_1^*, d_1^*}} a1$ in the AF in Figure \ref{SchCQAFs2}. That is, if we commit to $\set{c_1, c_1^*, d_1^*}$ (the context), $a_1^*$ is then ranked higher than $a1$.
\end{example}

All in all, our analysis shows that rankings based on graded acceptability yield an original combination of postulates with respect to existing approaches, and show how a graded generalisation of Dung semantics can contribute a rich perspective to the ranking-based semantics programme.

\smallskip

Finally, we observe that in contrast to the above described graph propagation approaches, we have (in Section \ref{Sec:Applications} and in the above example) studied our notions of graduality and rankings as they apply to instantiated argumentation.
In particular, we have studied well-known instantiations of Dung's theory yielded by the use of schemes and critical questions, and classical logic instantiations that yield dialectical characterisations of non-monotonic inference. We believe that such a study is an important complement to the use of postulates in evaluating the intuitions captured by proposals for graduality and rankings. It is with this aim in mind that Section \ref{Sec:Applications} seeks to substantiate the intuitions captured by our graded generalisation of Dung's theory.

%


\section{Conclusions and Future Work}\label{Sec:Future}

\paragraph{Summary}

This paper has presented three contributions. First, we have generalised Dung's standard semantics by providing graded variants of the classic extensions studied in abstract argumentation, and studied their fixpoint-theoretic behavior. In doing so, the paper also provided a comprehensive exposition of the fixpoint theory of standard Dung's semantics (Section \ref{Sec:Background}) which, to the best of our knowledge, was not yet available in the literature.

The resulting graded semantics for abstract argumentation has then been shown to enable a simple way of ranking arguments according to how strongly they are justified under different graded semantics. We showed how this enables an arguably natural way of arbitrating, in abstract argumentation, among justified arguments that are credulously justified under the standard Dung semantics, without recourse to exogenous information. Such rankings have then also been applied to instantiated Dung frameworks, specifically via \emph{ASPIC+} formalisations of `human orientated' argumentation encoded through the use of argument schemes and critical questions, and dialectical characterisations of non-monotonic inference defined by Brewka's Preferred Subtheories. In so doing, we have also shown how the graded semantics accounts for the mutual strengthening of arguments with the same conclusion (i.e., the accrual of arguments) within the Dung paradigm.

Finally, the novel rankings have been thoroughly compared with existing approaches to ranking-based semantics, allowing us to highlight the similarities and core differences between rankings based on graded semantics and other influential approaches.


\paragraph{Future work}
The paper has aimed at establishing foundations for a graded generalisation of Dung's argumentation theory. Based on our results, many natural avenues for future research present themselves. We confine ourselves to mentioning those that in our opinion are the most promising.


\smallskip

We observed in Section \ref{Sec:ExtendingPartialOrder} that the rankings enabled by our graded semantics are, without further assumptions, partial. The ensuing theory of graded rankings is therefore less committal than existing proposals in the ranking semantics literature, and could therefore be extended to accommodate further pre-theoretical intuitions. For example, some application domains may warrant distinguishing amongst arguments justified under the standard Dung semantics, by assigning a higher ranking to those that have a higher number of attacks. For example, consider that scientific theories establish their credibility to the extent that they successfully defend themselves against arguments attempting to refute (attack) them. Exploring directions in which graded rankings could be extended to capture this alternative intuition
in a principled manner, is clearly a natural direction of research.


\smallskip

Dung's original semantics were born out of an attempt to systematize patterns underpinning the semantics of all the main non-monotonic logics. Section \ref{Sec:GradedPrefSubtheories}'s application of graded semantics to dialectical characterisations of  Preferred Subtheories strongly suggests that graded semantics could offer a similar tool for systematizing paraconsistent non-monotonic logics. By relaxing the conflict-freeness and self-defence requirements, paraconsistent non-monotonic inference relations could be defined akin to the way in which non-monotonic inferences are defined by Dung's argumentation theory.
As graded semantics enable a 'controlled' relaxation of conflict-freeness, these inference relations should be able to keep the explosivity of the ensuing logic at bay. The development of these ideas can lead, we argue, to novel interesting insights at the interface of paraconsistent logic and argumentation theory.
This in turn may motivate the study of rationality postulates for argumentation \cite{Caminada2007286} that are reformulated to accommodate paraconsistent inference relations.

Notice also, that in Section \ref{Sec:GradedPrefSubtheories} we applied graded semantics to defeat graphs that employ exogenous preferences to decide which attacks succeed as defeats, and highlighted how these explicit preferences take precedence over the implicit preferences yielded by our rankings. The interaction between exogenous and endogenously defined preferences is also an obvious direction for future research.


\smallskip

The theory of graded argumentation as developed in this paper has focused on extracting graduality from information endogenous to standard argumentation frameworks, essentially by counting the attackers of arguments. This intuition can however be extended along lines explored by weighted argumentation frameworks \cite{dunne11weighted}. We briefly sketch how this could be done, assuming now each attack is assigned a value in $\mathbb{R}^+$ by a weighting function $w$.  Recall Definition \ref{definition:sensitive} and consider the following alternative definition of graded defense:
\begin{align}
\cfmn{m}{n}(X) & = \set{x \in \A \mid \rho( \set{(y,x) \in \overline{x} \mid \rho(y^- \cap X) < n}) < m} \label{eq:weighted}
\end{align}
where, for $R \subseteq \ar$, $\rho(R) =  \sum_{(x,y) \in R} w(x,y)$. Clearly the definition of graded defense given in Definition \ref{definition:sensitive} is the special case of \eqref{eq:weighted} where $w$ assigns the weight of $1$ to each attack. This generalization suggests that the logic behind graded semantics can be leveraged beyond the standard case explored here, and therefore  also applied to existing proposals based on weighted argumentation (e.g., \cite{dunne11weighted}).

\smallskip

In Section \ref{Sec:RelatedWork} we discussed differences between rankings defined by graded semantics and propagation-based approaches. Although, as we have shown, the two approaches build on different underpinning intuitions, it should be stressed that they remain compatible. In fact an interesting direction for future research would be to explore how features of the latter might be integrated with our graded approach. For example, recalling the example illustrating that strict counter-transitivity is not satisfied, one might note that since $b$'s attacker $c$ is ranked above $a$'s attacker $d$ by the graded
ranking in Figure \ref{SCTCounterExample}i, then the ranking could be refined to favour $a$ over $b$. Similarly, defense precedence could be enforced by noting that since $b$'s attacker $f$ is ranked strictly higher
than $a$'s attacker $c$ (according to the graded ranking in Figure \ref{SCTCounterExample}ii, then the ranking can be refined to rank $b$  strictly below $a$. By the same reasoning, the ranking in Figure \ref{SCTCounterExample}iii could be refined to rank $b1$ above $b$ and so $a$ above $a1$ (so enforcing satisfaction of `increase of an attack path'  and `increase of a defense path'), and the ranking $a1^*  \approx^S a1$ in Example \ref{ExReinstatement2Diff} could be refined to obtain
$a1^*  \succ^S_\Delta a1$ by noting that
  $b1 \succ^S_{\Delta} b1^*$.


\smallskip

Finally,  Dung's standard semantics lend themselves to natural dialectical interpretations via argument games (cf. \cite{modgil09proof} for an overview). Given that graded semantics retain many of the key fixpoint-theoretic properties of standard Dung's semantics, we believe that standard argument games---specifically the game for the grounded extension and for the credulous preferred semantics---could be appropriately modified to obtain game-theoretic characterizations of (at least some of) our semantics.\footnote{Games for all semantics could be obtained via a detour through suitable logic games. Cf. \cite{grossi13abstract}.} These types of games would shed novel light on the under-investigated link between dialectical approaches to argumentation and the younger literature on ranking-based semantics.




\appendix

\section*{Appendix. Proofs of Section \ref{Sec:Background}} \label{appendix:proofs}

\begin{proof}[Proof of Fact \ref{fact:continuous}]
\fbox{\eqref{eq:upward}}
We prove the two directions.
\rightleft The claim follows from the monotonicity of $\cf{\Frame}$. For any $X$, $\cf{\Frame}(X) \subseteq \cf{\Frame}(\bigcup_{X \in D} X)$, hence $\bigcup_{X \in D} \cf{\Frame}(X) \subseteq \cf{\Frame}(\bigcup_{X \in D} X)$.
\leftright Assume $x \in \cf{\Frame}(\bigcup_{X \in D} X)$. By finitariness, and since $D$ is upward directed, $\exists X \in D$ s.t. if $y \rightarrow x$, $\exists z \in X$ s.t. $z \rightarrow y$ Hence $x \in \bigcup_{X \in D} \cf{\Frame}(X)$.
\fbox{\eqref{eq:downward}}
We prove the two directions.
\leftright The claim follows again from the monotonicity of $\cf{\Frame}$. For any $X$, $\cf{\Frame}(X) \supseteq \cf{\Frame}(\bigcap_{X \in D} X)$, hence $\bigcap_{X \in D} \cf{\Frame}(X) \supseteq \cf{\Frame}(\bigcap_{X \in D} X)$.
\rightleft Assume $x \in \bigcap_{X \in D} \cf{\Frame}(X)$. By finitariness there exists a finite set $X \in D$ defending $x$. By the assumption $X$ is contained in $\bigcap_{X \in D} X$, for otherwise there would exist a $Y \in D$ such that $X \supset Y$ and $x \not\in \cf{\Frame}(Y)$, against the assumption. We conclude that $x \in \cf{\Frame}\left(\bigcap_{X \in D} X\right)$.\\
\end{proof}

\begin{proof}[Proof of Lemma \ref{lemma:preserve_cf}]
\fbox{$X \subseteq \cf{\Frame}^n(X) $} holds by the assumption that $X$ is admissible and by the monotonicity of $\cf{}$ (Fact \ref{fact:simple}).
\fbox{$\cf{\Frame}^n(X) \subseteq \cff{\Frame}(\cf{\Frame}^n(X))$} is proven by induction over $n$. The base case holds by assumption as $\cf{\Frame}^0(X) = X$ is admissible. For the induction step assume (IH) that $\cf{\Frame}^n(X) \subseteq \cff{\Frame}(\cf{\Frame}^n(X))$. We show that $\cf{\Frame}^{n+1}(X) \subseteq \cff{\Frame}(\cf{\Frame}^{n+1}(X))$. Suppose towards a contradiction that is not the case. Then there exists $x, y \in \cf{\Frame}^{n+1}(X) = \cf{\Frame}(\cf{\Frame}^{n}(X))$ such that $x \al y$. By the definition of $\cf{}$, there exists $z \in \cf{\Frame}^{n}(X)$ such that $y \al z$. But as  $y \in \cf{\Frame}^{n+1}(X)$ there exists also $w \in \cf{\Frame}^{n}(X)$ such that $z \al w$. From this we conclude that $\cf{\Frame}^{n}(X)$ is not conflict free, against IH.
\fbox{$\cff{\Frame}(\cf{\Frame}^n(X)) \subseteq \cff{\Frame}(X)$} follows from the first claim by the antitonicity of $\cff{}$ (Fact \ref{fact:simple}).
\end{proof}

\begin{proof}[Proof of Lemma \ref{theorem:approxim8}]

\fbox{\eqref{eq:below}}
\fbox{First}, we prove that $\bigcup_{0 \leq n < \omega} \cf{\Frame}^n(X)$ is a fixpoint by the following series of equations:
\begin{align*}
\cf{\Frame} \left( \bigcup_{0 \leq n < \omega} \cf{\Frame}^n(X) \right) & =  \bigcup_{0 \leq n < \omega}\cf{\Frame}(\cf{\Frame}^n(X)) \\
                                                                                                                                                & =  \bigcup_{0 \leq n < \omega} \cf{\Frame}^n(X)
\end{align*}
where the first equation holds by the continuity of $\cf{\Frame}$ (Fact \ref{fact:continuous}) and the second by the fact that, since $X$ is admissible, $X = \cf{\Frame}^0(X) \subseteq \cf{\Frame}^n(X)$ for any $0 \leq n < \omega$.
\fbox{Second}, we prove that $\bigcup_{0 \leq n < \omega} \cf{\Frame}^n(X)$ is indeed the least fixpoint containing $X$. Suppose, towards a contradiction that there exists $Y$ s.t.: $X \subset Y = \cf{\Frame}(Y) \subset \bigcup_{0 \leq n < \omega} \cf{\Frame}^n(X)$. It follows that $X \subset Y =  \cf{\Frame}(Y) \subset \cf{\Frame}^n(X)$ for some $0 \leq n < \omega$. But, by monotonicity, we have that $\cf{\Frame}^n(X) \subseteq \cf{\Frame}^n(Y)$. Contradiction.

\fbox{\eqref{eq:above}} The proof is similar to the previous case, but involves a few extra subtleties. \fbox{First}, we prove that $\bigcap_{0 \leq n < \omega} \cf{\Frame}^n(\cff{\Frame}(X))$ is a fixpoint, through the series of equations
\begin{align*}
\cf{\Frame} \left( \bigcap_{0 \leq n < \omega} \cf{\Frame}^n(\cff{\Frame}(X)) \right) & =  \bigcap_{0 \leq n < \omega}\cf{\Frame}(\cf{\Frame}^n(\cff{\Frame}(X))) \\
                                                                                                                                                & =  \bigcap_{0 \leq n < \omega} \cf{\Frame}^n(\cff{\Frame}(X))
\end{align*}
which hold by Fact \ref{fact:continuous} and by the fact that $\cff{\Frame}(X) = \cf{\Frame}^0(\cff{\Frame}(X)) \supseteq \cf{\Frame}^n(\cff{\Frame}(X))$ for any $0 \leq n < \omega$. The latter property holds because $X$ is assumed to be admissible, and hence $X \subseteq \cf{\Frame(X)}$.
By the antitonicity of $\cff{\Frame}$ (Fact \ref{fact:simple}) we therefore have that $\cff{\Frame}(\cf{\Frame}(X)) \subseteq \cff{\Frame}(X)$ and, since $\cf{\Frame} = \cff{\Frame} \circ \cff{\Frame}$ (Fact \ref{fact:simple} again), it follows that $\cf{\Frame}(\cff{\Frame}(X)) \subseteq \cff{\Frame}(X)$ as desired.
\fbox{Second}, it remains to be proven that $\bigcap_{0 \leq n < \omega} \cf{\Frame}^n(\cff{\Frame}(X))$ is indeed the largest fixpoint of $\cf{\Frame}$ contained in $\cff{\Frame}(X)$. Like in the previous case we proceed towards a contradiction. Suppose there exists $Y$ s.t.: $\bigcap_{0 \leq n < \omega} \cf{\Frame}^n(\cff{\Frame}(X)) \subset Y = \cf{\Frame}(Y) \subseteq \cff{\Frame}(X)$. There must therefore exist an integer $k$ such that, as a consequence of Lemma \ref{lemma:preserve_cf}, $X \subseteq \cf{\Frame}^k(X) \subseteq  \cff{\Frame}(\cf{\Frame}^k(X)) \subset Y$. By the antitonicity of $\cff{\Frame}$ and the fact that $\cf{} = \cff{} \circ \cff{}$ (Fact \ref{fact:properties_dn}), and since $Y$ is taken to be a fixpoint of $\cf{\Frame}$, it follows that $\cff{\Frame} (Y) = \cff{\Frame}(\cf{\Frame}(Y)) \subset \cf{\Frame}(\cf{\Frame}^k(X))$. So $\cff{\Frame} (Y)$ is also a fixpoint of $\cf{\Frame}$, it contains $X$ and it is included in $\bigcup_{0 \leq n < \omega} \cf{\Frame}^n(X)$, which, by the previous claim, is the smallest fixpoint of $\cf{\Frame}$ containing $X$. Contradiction.
\end{proof}

\begin{proof}[Proof of Theorem \ref{lemma:smallest}]
By Lemma \ref{theorem:approxim8}, $\lfp_X.\cf{\Frame} = \bigcup_{0 \leq n < \omega} \cf{\Frame}^n(X)$. So $\bigcup_{0 \leq n < \omega} \cf{\Frame}^n(X)$ denotes a fixpoint of $\cf{\Frame}$, and more specifically the smallest such fixpoint that contains the admissible set $X$. By Lemma \ref{lemma:preserve_cf} we also know that such set is conflict-free. We therefore conclude that $\bigcup_{0 \leq n < \omega} \cf{\Frame}^n(X)$ is a conflict-free fixpoint of $\cf{\Frame}$, that is, a complete extension, and that this complete extension is the smallest such set containing $X$, as claimed.
\end{proof}


\section*{References}

\end{document}